\documentclass[lettersize,journal]{IEEEtran}
\usepackage{amsmath,amsfonts}
\usepackage{algorithmic}
\usepackage{algorithm}
\usepackage{array}
\usepackage[caption=false,font=normalsize,labelfont=sf,textfont=sf]{subfig}
\usepackage{textcomp}
\usepackage{stfloats}
\usepackage{url}
\usepackage{verbatim}
\usepackage{graphicx}
\usepackage{cite}
\usepackage{amssymb}
\usepackage{booktabs}
\usepackage{multirow}
\usepackage{mathrsfs}

\usepackage{amsthm}
\theoremstyle{plain}
\newtheorem{theorem}{Theorem}
\newtheorem{proposition}[theorem]{Proposition}

\newtheorem{corollary}[theorem]{Corollary}
\theoremstyle{definition}

\theoremstyle{remark}
\newtheorem{remark}{Remark}

\DeclareMathOperator{\Tr}{Tr}

\begin{document}

\title{ProML: Progressive Multi-label Learning}
\vspace{-5mm}

\author{Yi Kou,~\IEEEmembership{Student Member,~IEEE,}
	Jieting Wang,~\IEEEmembership{Member,~IEEE,}
    and 
    Yuhua Qian,~\IEEEmembership{Member,~IEEE}
\thanks{Yi Kou, Jieting Wang and  Yuhua Qian are with the Institute of Big Data Science and Industry, the School of Artificial Intelligence, and the Key Laboratory of Evolutionary Science Intelligence of Shanxi Province, Shanxi University, Taiyuan 030006, Shanxi, China. E-mail: ky20eku@163.com; jtwang@sxu.edu.cn; jinchengqyh@126.com;.}
\thanks{Corresponding authors: Yuhua Qian.}
}

\markboth{Journal of \LaTeX\ Class Files,~Vol.~14, No.~8, August~2021}%
{Shell \MakeLowercase{\textit{et al.}}: A Sample Article Using IEEEtran.cls for IEEE Journals}


\maketitle
\begin{abstract}
Multi-label learning typically pursues model construction through the joint optimization of label fitting, correlation modeling, and structural regularization. While empirically effective, such a coupled paradigm imposes multiple learning objectives on a single predictive parameter matrix, forcing it to simultaneously fit observed labels and encode structural priors. This coupling may induce ill-conditioned optimization geometry, increase hyper-parameter sensitivity, and complicate capacity control. 
To address the these issues, this paper proposes a \textit{Progressive Multi-Label Learning} (ProML) framework, which decouples direct label fitting from predictor learning via the introduction of an auxiliary predictive variable. This auxiliary variable admits a closed-form solution, yielding an equivalent profiled objective in which the supervised fitting term is scaled by a trade-off parameter. This result provides a unified view of ProML and conventional coupled learning, with the latter recovered as a limiting case. 
We theoretically analyze the optimization geometry of ProML and show that, under quadratic positive semidefinite structural regularization, ProML attenuates the dominant curvature of the fitting term, enlarges the stable learning rate range, and provides a controllable fidelity-regularization trade-off. 
Further, we instantiate the ProML for multi-label feature selection task by integrating global-local label priors and row-sparse regularization. 
Extensive experiments and controlled optimization analyses on 12 benchmark datasets empirically corroborate the theoretical analysis and demonstrate the robust and competitive performance of the proposed method.
\end{abstract}


\section{Introduction}
\IEEEPARstart{M}ULTI-LABEL Learning (MLL) serves as the pivotal paradigm for solving complex real-world data, aiming to assign multiple class labels to each sample simultaneously \cite{tsoumakas2007multi,sorower2010literature,zhang2013review}. In such problems, effective model construct necessitates not only the full exploitation of feature information but also the explicit encoding of intrinsic dependencies among labels \cite{huang_learning_2016, zhao_learning_2022, li_learning_2022, zhu2017multi, geng2016label, jia2024discriminative}. To this end, existing approaches generally follow a label correlation regularization \cite{zhao_learning_2022,bao2022asymmetry} or graph regularization \cite{li_learning_2022,jia2024discriminative} learning paradigm. Specifically, they integrate three distinct tasks---label prediction, dependency modeling, and sparsity control---into a single compact optimization problem by imposing a supervised fitting term, a label structure constraint, and a model complexity penalty concurrently onto a single parameter matrix. For example, Huang et al. \cite{huang_learning_2016} proposed the LLSF model, which learns label-specific representations while simultaneously accounting for both feature sharing among correlated labels and sparsity of label-specific features on a single parameter matrix. Building upon this, Zhao et al. \cite{zhao_learning_2022} further extended the framework by jointly modeling label correlations from both global and local perspectives, thereby capturing label dependencies in a more comprehensive manner. 
Due to its compact and concise form, this coupled formulation has been widely deployed across diverse field, \textit{e.g.}, image annotation \cite{wei2015hcp}, text categorization \cite{hao2021rcnns}, and emotions recognition \cite{zhang2021multi}. 

Despite its popularity, the general multi-label learning paradigm has a structure limitation that has received relatively less attention. Specifically, in general multi-label learning (GMLL\footnote{The basic form of GMLL is $\underset{\boldsymbol{W}}{\min} \|\mathbf{X}\boldsymbol{W}-\mathbf{Y}\|_F^2 + \lambda\Omega(\boldsymbol{W};\boldsymbol{C})$. For more details, please refer to Eq.(\ref{eq:gmll}) and (\ref{eq:gmll_s}).}) modeling, the parameter matrix serves two distinct functions. On the one hand, it directly fits the observes labels. Simultaneously, on the other hand, it encodes the label correlation structure and complexity constrains. These coupled tasks underscore the direct competition between the supervised fitting term and structural regularization term within same parameter space, resulting in significant issues across three aspects: $1)$ \textit{Optimization geometry} aspect. The curvature of the fitting term is excessively large, which makes optimization trajectory oscillation. $2)$ \textit{Hyperparameter sensitivity} aspect. The fitting term and structural term compete directly, making the model is sensitive to regularization coefficients. $3)$ \textit{Generalization ability} aspect. Multiple objectives coupled within the same parameter space, capacity control is difficult, affecting generalization performance. 

To be specific, first, the local curvature of the loss is governed by its Hessian matrix \cite{ghorbani2019investigation}. For the least-squares (LS) loss, the Hessian is governed by ${\mathbf{X}\!^{\top}\! \mathbf{X}}$, and hence a large leading eigenvalue of ${\mathbf{X}\!^{\top}\!\mathbf{X}}$ induces high curvature along the dominant eigendirection. As a consequence, this steep curvature imposes a stringent upper bound on the allowable stable learning rate, making the gradient-based optimization process more prone to oscillation when the step-size is not sufficiently small \cite{lecun1992automatic}. 
Second, the fitting term steers the predictions toward the observed labels, while the structural term enforces the classifier to satisfy the label similarity structure, and the regularization term further restricts the model complexity to control capacity---all of which are imposed upon the same parameter matrix \cite{zhu2017multi, jia2020joint}. These components may exert different and even competing influences on optimization, necessitating a careful balance among label consistency, structural constraints, and capacity control through carefully selected trade-off parameters. This intrinsic competition inevitably renders the overall framework highly sensitive to the choice of hyperparameters \cite{sener2018multi}. 
Finally, the classifier is often required to simultaneously fit the observed labels, encode the label dependency structures, and enforce regularization constrains within the same parameter space. Such coupled objectives increase the difficulty of capacity control, since the shared parameter matrix may not only capture predictive label patterns but also absorb noise from observed labels and bias from imposed structural assumptions \cite{wang2017regularizing,zhao2021evaluating}. 
Therefore, the fundamental challenge of coupled modeling paradigm doesn't originate from the regularization, but rather from the inherent tendency of the learning mechanism to systematically amplify the dominant influence of the supervised fitting term. 
This critical deficiency urgently calls for the exploration of a novel paradigm that can appropriately decouple label fitting from structural learning.

Based on the observations mentioned above, in this paper, we propose a \textit{Progressive Multi-label Learning} paradigm, termed ProML. 
Unlike conventional coupled approaches that force the predictor to directly fit the observed labels, ProML introduces a learnable auxiliary target variable as an intermediary bridge between supervision and prediction. This design explicitly decouples label fitting and correlation modeling.
More precisely, the auxiliary target is designed to maintain consistency with the labels, while the predictor is merely encouraged to approach this intermediate target. Such a progressive fitting mechanism relaxes the direct competition between the supervised fitting term and the structural regularizer, enabling label dependencies and sparsity priors to be encoded more flexibly. 
More importantly, the auxiliary variable can be reduced through closed-form solution, yielding a profiled objective function that depends solely on the predictor. This reformulation reveals that ProML is not merely an auxiliary variable extension, but a principled rebalancing of label fidelity and structural regularization. 
As the trade-off parameter tend to infinity, the profiled ProML objective exactly degenerates to the GMLL formulation, indicating that GMLL is essentially a boundary case of ProML. This profiled perspective directly reveal the underlying mechanism that endows ProML with significantly more stable optimization behaviors. 
Fig. \ref{fig:landscapes_3d} shows the loss landscapes and optimization trajectories of two paradigm. It is clearly observed that, under the same settings, ProML exhibits more stable optimization behavior.
\begin{figure}[t]
    \centering
    \includegraphics[width=0.9\linewidth]{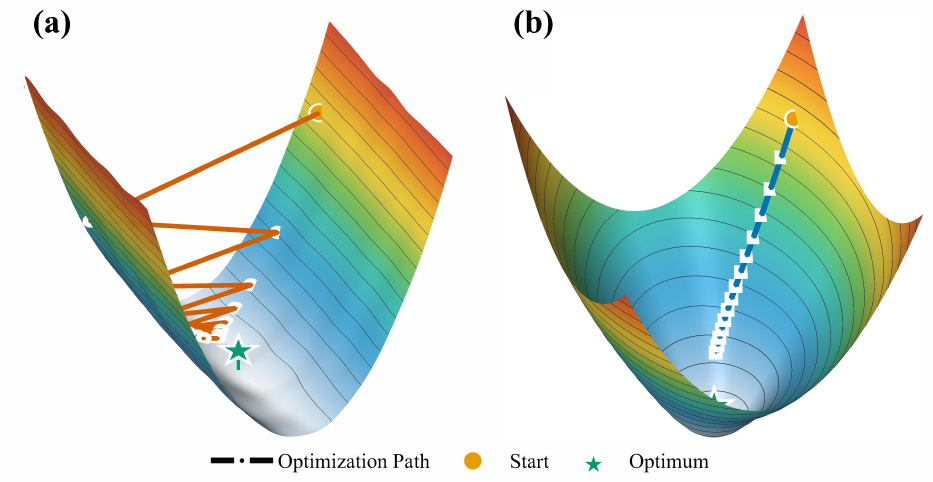}
    \caption{The loss landscapes and optimization trajectories of (a) GMLL and (b) ProML. For ProML, the additional profiling parameter is set to $\mu=0.046$. Under these otherwise identical settings, the figure provides a visual comparison of the optimization landscapes and trajectories induced by the two formulations. The curves represent the optimization paths, and the yellow dots and green stars denote the starting points and the obtained optima, respectively.}
    \label{fig:landscapes_3d}
    \vspace{-2.5mm}
\end{figure}

To further examine the practical applicability of the proposed framework, we instantiate ProML for the multi-label feature selection (MLFS) task. MLFS provides a representative testbed for ProML, as it simultaneously involves label prediction, parameter sparsity control, label dependency modeling, and optimization stability. 
Moreover, real-world data further compound this difficulty, as label dependencies are rarely symmetric: the presence of one label may imply another, while the reverse implication may not hold. To capture this asymmetry, we construct an asymmetric label correlation priors based on global-local dependencies and incorporate it into the ProML-based feature selection model. 
The resulting formulation jointly harmonizes the profiled prediction, dependency regularization, and row-wise sparse feature selection within a unified optimization objective.

The main contributions of this paper are summarized as follows:
\begin{itemize}
    \item We propose a new learning paradigm, ProML, which establishes the unification with GMLL via an explicit profiled mechanism. And, we demonstrate that GMLL is essentially a degenerate case of ProML.
    \item From the perspective of optimization geometry, we prove that ProML can reshape the spectral distribution of the loss landscape, weaken the dominant curvature of the fitting term. Consequently, it expands the stable optimization region, and provides a controllable trade-off between label consistency and structural regularization.
    \item We instantiate ProML to the multi-label feature selection, where label dependency priors and row sparsity constrains are integrated. This instantiation yields a discriminative and interpretable solution, demonstrating the framework's flexibility and adaptability in handling intricately coupled learning scenarios.
    \item We validate the proposed framework through mechanism analysis at the framework level and benchmark evaluations at the task level, including spectrum effect analysis, optimization robustness assessment, feature selection performance, and ablation studies.
\end{itemize}

The rest of this paper is organized as follows. Section \ref{sec:works} reviews related work on multi-label learning. Section \ref{sec:proml} introduces the ProML paradigm and provides some related optimization properties. Section \ref{sec:ProML_fs} 
presents specific models and optimization algorithms that implement ProML for the feature selection task. Section \ref{sec:experi} reports comprehensive experimental results. Finally, Section \ref{sec:conclu} concludes this paper.

\section{Related work}\label{sec:works}
\subsection{Coupled Multi-label Learning}
Currently, leveraging label correlations to guide feature mapping has emerged as a predominant approach in multi-label learning. Numerous studies have incorporated various structural priors, such as including label dependency modeling, graph regularization, manifold constraints, and low-rank structures, into the objective function.

Previous studies primarily focused on directly constructing predictive mappings from feature space to label space. For example, Huang et al, \cite{huang2017joint} proposed JFSC, which introduces a label dissimilarity penalty based on cosine similarity among labels and learns the classifier by jointly modeling shared and label-specific features. Weng et al, \cite{weng2023learning} developed LFGLC, which captures label correlations from both global co-occurrence statistics and local label similarities, while employing $\ell_{1}$-norm regularization to learn label-specific features for predictive modeling. 
Si et al. \cite{si2023multi} proposed HOMI, which posits that multi-label matrix typically exhibit full-rank or near-full-rank characteristics. Accordingly, they construct a neighborhood graph based on the local geometric structure of samples and incorporate graph Laplacian regularization to unify high-order label correlation learning with the multi-label prediction process within a optimization framework.
Furthermore, beyond the standard setting, partial multi-label learning presents additional challenges, as the candidate label set contains both ground-truth labels and noisy annotations. To address this scenario, Xie et al. \cite{xie2021partial} introduced PML-NI, which employs trace norm regularization to impose a low-rank constraint for capturing label correlations, while simultaneously designing a noise identifier that leverages sparse regularization to explicitly model feature-induced noisy labels. From the perspective of feature selection, Li et al. \cite{li2023multi} proposed RFSFS, which integrates label fitting term, second-order label correlation constraint, sparse feature selection term, and feature redundancy penalty into a unified optimization framework, aiming to identify highly sparse label-specific features and low-redundancy shared features. 
A common thread across these methods is that they all adopt a direct fitting strategy to the observed label matrix, applying both prediction loss and structural constraints to a single parameter matrix, thereby integrating label fitting, correlation modeling, and sparsity control into a single optimization problem. However, most existing work focuses on designing stronger label correlations, little attention has been paid to the coupling mechanism between direct label fitting and structural constraints.

\subsection{Multi-label Feature Selection}
As a hot topic in the field of multi-label learning, multi-label feature selection, aims to effectively screen out a feature
subsets that contribute to the prediction of multiple labels from high-dimensional datasets. 

Generally, multi-label feature selection can be roughly grouped into three categories, \textit{i.e.}, filter, wrapper and embedded approaches \cite{QIAN2023101948,kashef2018multilabel}. Filter methods operate independently of any subsequent classifier, constructing specific evaluation criteria, such as mutual information \cite{lin2016multi}, to rank features prior to the classification. Wrapper methods search for a feature subset that maximizes the predictive performance of a given learner to identify one or multiple superior subsets, e.g., genetic algorithm \cite{lee2015memetic}. While embedded methods directly incorporate feature selection as an inherent component of the classifier training process, such as sparse regularization-based methods \cite{fan2024learning}, manifold-based methods \cite{zhong2021multi}. 

Here, we focus on reviewing several representative embedded approaches, which are closely related to the coupled learning paradigm in this paper. Sparse regularization methods identify discriminative features by jointly minimizing the fitting error and sparsity-inducing penalties, such as the $\ell_{1}$-norm or $\ell_{2,1}$-norm. For instance, Huang et al. \cite {huang2017joint} leverage the sparsity-inducing property of the $\ell_{1}$-norm to learn label-specific features for multi-label learning. Zhang et al. \cite{zhang2019manifold} perform low-dimensional embedding of the original feature space to capture label correlations by manifold regularization, and adopt the $\ell_{2,1}$ -norm regularization to select discriminative features shared across all class labels. Furthermore, Jia et al. \cite{jia2024discriminative} point out that the $\ell_{1}$-norm and $\ell_{2,1}$-norm regularizations may introduce redundant constraints under certain circumstances, and consequently propose a more robust feature selection framework. Despite their empirical effectiveness, the above methods still adhere to the conventional coupled learning paradigm. 

\subsection{Label Correlation}
Label correlation is an important component in multi-label learning. Most existing methods \cite{weng2023learning,zhang2019manifold,fan2024learning,kashef2018multilabel,sorower2010literature} are based on the idealized assumption that label relationship are symmetric, that is, the mutual constraints strength between any two labels are completely equivalent in both direction and magnitude.  In real-world multi-label data, however, label dependencies are often asymmetric: the presence of one label may imply another label, while the reverse implication may not hold. MAHR \cite{huang2012multi} first revealed the asymmetric nature of label correlations. ACML \cite{bao2022asymmetry} derived an asymmetric adjacency matrix by combining distance relevance and cosine similarity. However, existing asymmetric methods still lack modeling completeness in their design. Specifically, MAHR is limited to the global scale, overlooking the local heterogeneity of local label distributions. In contrast, ML2G \cite{zhong2021multi} and MFSRD \cite{qian2025multi} essentially concentrate solely on local aggregations, failing to form a dual-track perspective of macro and micro level. 
Therefore, how to develop a label associations with asymmetric dependency relationships, remains a research topic worthy of further investigation.

\section{The Unified View of Multi-label Learning}\label{sec:proml}
In this section, we present the conventional General Multi-Label Learning (GMLL) paradigm and the proposed Progressive Multi-Label Learning (ProML) paradigm from a unified perspective. We first discuss the general formulation and potential limitations of the conventional coupled paradigm, and then introduce the fundamental idea of ProML and reveal the intrinsic relationship between the two paradigms. On this basis, we investigate the theoretical properties of ProML in terms of optimization behavior, learning stability, and generalization capability, thereby demonstrating its potential advantages over conventional coupled learning and providing a theoretical foundation for the subsequent model development.

\subsection{Notations}
Formally, let $\mathbf{X}$ be the feature space and $\mathbf{Y}$ the label space. The multi-label dataset is characterized by a joint distribution $\mathcal{P}$ on $\mathbf{X}\!\times\!\mathbf{Y}$. For any dataset $\mathcal{S}\!=\!\{{\mathbf(x_i,y_i)}\}_{i=1}^n\!\sim\!\mathcal{P}$, we denote the feature matrix as $\mathbf{X}\!=\![\mathbf{x}_1,\cdots,\mathbf{x}_n]\!^{\top}\!\in \!\mathbb{R}^{n\times d}$, where $d$ is the feature dimensionality, and the label matrix as $\mathbf{Y}\!=\![\mathbf{y}_1,\cdots,\mathbf{y}_n]\!^{\top}\!\in \!{\{0,1\}}^{n\times c}$, where $c$ is the label cardinality. The learning objective is to obtain a predictive mapping $f\!:\! \mathbf{X}\to \mathbf{Y}$ that minimum the generalization error on unseen test samples. 

\vspace{-0.6mm}
\subsection{General Multi-label Learning Paradigm}
Most existing multi-label learning methods adhere to a coupled modeling paradigm, whose generic optimization objective can be uniformly formulated as:
\begin{equation}
    \min_{\boldsymbol{W}} \mathscr{L}(\mathbf{X}\boldsymbol{W},\mathbf{Y})+\Omega(\boldsymbol{W};\boldsymbol{C}), \label{eq:gmll}
\end{equation}
where $\mathscr{L}(\mathbf{X}\boldsymbol{W},\mathbf{Y})$  denotes the fitting loss between predictions and observed labels, and $\Omega(\boldsymbol{W};\boldsymbol{C})$ denotes the structural regularization term induced by label correlations, sparse priors and model complexity constrains, with $\boldsymbol{C}$ is label correlation matrix. The specific form of $\Omega(\boldsymbol{W};\boldsymbol{C})$ varies across different model designs. For example, it may encode label smoothness \cite{zhao_learning_2022}, instance correlation \cite{li_learning_2022}, label dependency \cite{zhu2017multi}, feature sparsity \cite{jia2024discriminative}, or a composite of these priors.

Specifically, by instantiating the fitting loss as the least squares loss and incorporating a label correlation regularizer together with a parameter complexity penalty, the problem can be expressed as:
	\begin{equation}
	\min_{\boldsymbol{W}} \underset{\textit{Fitting term}}{\underbrace{\frac{1}{2}\|\mathbf{X}\boldsymbol{W}-\mathbf{Y}\|_F^2}} + \underset{\textit{Structure regularization term}}{\underbrace{\frac{\lambda}{2} \Tr(\boldsymbol{WCW}^T) + \alpha\mathcal{R}(\boldsymbol{W})}},
    \label{eq:gmll_s}
	\end{equation}
where the fitting term enforces the predictions $\mathbf{X}\boldsymbol{W}$ closely match the observed labels $\mathbf{Y}$. The structural regularizer consists of two components: the trace term $\Tr(\boldsymbol{WCW}^{\top})$, induced by the label correlation matrix $\boldsymbol{C}$, encodes the structural dependencies among the the parameters of the classifiers for each label. While $\mathcal{R}(\boldsymbol{W})$ serves to control model complexity or induce sparsity, thereby mitigating overfitting. Therefore, GMLL can be summarized as the joint optimization of label fitting, structural consistency, and complexity control.

Despite the above formulation is concise and widely used, this formulation consolidates three distinct learning tasks on the same parameter matrix.  Since the gradients of these terms may not align, the shared parameters forces the model to compromise among label consistency, structural priors, and capacity control, making it difficult to satisfy simultaneously. This inherent optimization conflict undermines the model's stability and generalization performance.

\vspace{-0.7mm}
\subsection{Progressive Multi-label Learning Paradigm}
To alleviate the above issue, we introduce a progressive multi-label learning paradigm, termed ProML. The core idea is to replace direct label fitting with progressive target refinement strategy. Specifically, instead of forcing $\mathbf{X}\boldsymbol{W}$ to directly approximate $\mathbf{Y}$, ProML introduces  learnable intermediary prediction target $\mathbf{P}$. The predictor $\boldsymbol{W}$ is steered toward $\mathbf{P}$, while $\mathbf{P}$ is progressively encouraged to approach the observed label matrix $\mathbf{Y}$. The general ProML objective is formulated as
\begin{equation}
\min_{\boldsymbol{W},\mathbf{P}}\frac{1}{2}\|\mathbf{P-Y}\|_F^2 + \frac{\mu}{2}\|\mathbf{X}\boldsymbol{W}-\mathbf{P}\|_F^2 + \Omega(\boldsymbol{W};\boldsymbol{C}),
\label{eq:Proml}
\end{equation}
where $\mu>0$  is a tradeoff parameter. The first term enforces label fidelity of the auxiliary variable $\mathbf{P}$, ensuring that it doesn't deviate too much from the observed labels, the second term steers the model prediction $\mathbf{X}\boldsymbol{W}$ toward the intermediary $\mathbf{P}$, and the third term plays the same role as the structural regularizer in GMLL. 

In contrast to GMLL, the defining distinction of ProML lies in that the prediction $\mathbf{X}\boldsymbol{W}$ is no longer required to directly fit the observed labels $\mathbf{Y}$. Instead, the fitting pathway is decomposed into two sequential stages: $\mathbf{X}\boldsymbol{W}\to \mathbf{P}\to \mathbf{Y}$. This allows the model to maintain label fidelity while effectively weakening the rigid coupling between the model predictions and the observed supervision.

The term ``progressive'' in ProML refers to the iterative refinement process between the weight matrix $\boldsymbol{W}$ and the intermediate target $\mathbf{P}$. Specifically, ProML is optimized via an alternating update scheme. At the $t$-iteration, given the target $P^{(t)}$, the predictor is updated by solving the following:
\begin{equation}
    \boldsymbol{W}^{(t+1)} = \arg\min_{\boldsymbol{W}} \frac{\mu}{2}\|\mathbf{X}\boldsymbol{W}-P^{(t)}\|_F^2 + \Omega(\boldsymbol{W};\boldsymbol{C}).
\end{equation}

This step indicates that $\boldsymbol{W}$ is guided to fit the current target $P^{(t)}$. The model learns under a softened, structurally compatible target guidance. Subsequently, given the updated predictor $\boldsymbol{W}^{(t+1)}$, the target $\mathbf{P}$ is updated by solving the following:
\begin{equation}
    \mathbf{P}^{(t+1)} = \arg\min_{\mathbf{P}} \frac{1}{2}\|\mathbf{P}-\mathbf{Y}\|_F^2 + \frac{\mu}{2}\|\mathbf{X}\boldsymbol{W}^{(t+1)}-\mathbf{P}\|_F^2.
\end{equation}

Thus, $\mathbf{P}^{(t+1)}$ lies between the observed supervision $\mathbf{Y}$ and the current model prediction $\mathbf{X}\boldsymbol{W}^{(t+1)}$. As an intermediate prediction, it gradually passes the label supervision information to the predictor.

The above problem has the following closed-form solution:
\begin{equation}
    \mathbf{P}^{(t+1)} \!=\!\frac{\mathbf{Y}+\mu \mathbf{X}\boldsymbol{W}^{(t+1)}}{1+\mu}\!=\!\frac{1}{1+\mu}\mathbf{Y} \!+\! \frac{\mu}{1+\mu}\mathbf{X}\boldsymbol{W}^{(t+1)}.
    \label{eq:p6}
\end{equation}

Eq.\eqref{eq:p6} shows that $\mathbf{P}^{(t+1)}$ is a convex combination of the observed label matrix and the current model prediction. Therefore, the update of $\mathbf{P}$ gradually moves the intermediate target toward the observed labels while still preserving consistency with the current prediction model. The overall alternating process can be summarized as:
\begin{equation*}
    \mathbf{P}^{(t)}  \longrightarrow  \boldsymbol{W}^{(t+1)}  \longrightarrow  \mathbf{P}^{(t+1)}  \longrightarrow  \boldsymbol{W}^{(t+2)}  \longrightarrow  \cdots. 
\end{equation*}

During this process, $\boldsymbol{W}$ is continually steered toward the target $\mathbf{P}$, while $\mathbf{P}$ is progressively refined toward the observed labels $\mathbf{Y}$. This mechanism effectively establishes a buffer between the predictive model and the observed labels, shielding the model from absorbing the full impact of noisy labels or structural inconsistencies at each iteration.

\subsection{Optimization Geometry and Stability Tradeoff}
Based on above discussion, this section proceeds to analyze the core properties of ProML from a theoretical perspective, providing an optimization-geometric understanding that underpins the effectiveness of the proposed framework.

To focus on the core optimization-geometric mechanism of ProML and to clarify that the model's advantages stem from the framework design itself rather than from any specific regularizer, we temporarily omit the concrete form of the structural regularizer $\mathcal{R}(\boldsymbol{W})$ in the theoretical analysis, \textit{i.e.}, we set $\Omega(\boldsymbol{W};\boldsymbol{C})=\frac{\lambda}{2} \Tr(\boldsymbol{WCW}^T)$. Due to space limitations, the detailed proofs in this paper are provided in the Supplementary Material.

\begin{proposition}[\textbf{Profiled Equivalence of ProML}]\label{pro_profile}
    For any fixed parameter matrix $\boldsymbol{W}$, the auxiliary prediction variable $\mathbf{P}$ admits a unique closed-form solution
\begin{equation}
    \mathbf{P}^{\ast}(\boldsymbol{W}) \!=\!\frac{\mathbf{Y}+\mu \mathbf{X}\boldsymbol{W}}{1+\mu}, 
\end{equation}
where $\mu>0$ is the tradeoff parameter. Consequently, profiling out $\mathbf{P}$ from the original objective in Eq.(\ref{eq:Proml}) yields an equivalent reduced optimization problem over $\boldsymbol{W}$:
\begin{equation}
    \min_{\boldsymbol{W}} \frac{s}{2}\|\mathbf{X}\boldsymbol{W}-\mathbf{Y}\|_F^2+\Omega(\boldsymbol{W};\boldsymbol{C}),
    \label{eq:profil_pro}
\end{equation}
where $\|\cdot\|_F$ denotes the Frobenius norm, and the supervised fitting term is scaled by $s = \mu / (1 + \mu)\in (0,1)$.
\end{proposition}

\begin{proof}
    The detailed proof is provided in the Supplementary Material II.  
\end{proof}

\begin{remark}
    The profiled equivalence established in Proposition 1 reveals the key mechanism of ProML: rather than merely introducing an auxiliary variable, it effectively rescales the supervised fitting term via profiled elimination, thereby re-calibrating the relative dominance between the label fitting term and the structural regularizer. Furthermore, this equivalent formulation also reveals the intrinsic connection between ProML and GMLL, leading to the following corollary.
\end{remark}

\begin{corollary}[\textbf{Limiting case of ProML}]\label{pro_limit}
    As the trade-off parameter $\mu\to \infty$, the profiled ProML objective in Eq.\eqref{eq:profil_pro} degenerates to the standard GMLL formulation in Eq.\eqref{eq:gmll_s}.
\end{corollary}
\begin{proof}
    The detailed proof is provided in the Supplementary Material III.
\end{proof}

\begin{remark}
Corollary 2 clarifies that ProML and GMLL should not be regarded as two distinct forms. Instead, the family of objective functions defined by ProML naturally incorporates GMLL as a boundary case. Specifically, ProML introduces the scaling factor $s$ via profiled elimination, which directly governs the relative weight of the supervised fitting term within the overall optimization objective.
The parameter $\mu$ plays a dual role in this framework. On the one hand, $\mu$ controls the strength of consistency between the model prediction and the intermediate target. On the other hand, as revealed by Eq.(\ref{eq:profil_pro}), $\mu$ uniquely determines the effective scaling coefficient $s$ of the fitting term in the profiled objective. 
When $\mu$ is large, the ProML objective function degenerates to GMLL. When $\mu$ is small, the structural regularization term gains relatively greater influence, but insufficient consideration of observational information can lead to under-fitting. 
\end{remark}

\begin{proposition}[\textbf{Smoothness Compression}]\label{pro_compre}
    Let $\mathbf{H}_{P}$ and $\mathbf{H}_{G}$ denote the Hessian matrices of the ProML and GMLL objective, respectively. For any $\mu>0$, and the quadratic structural regularization, \textit{i.e.}, $\Omega(\boldsymbol{W};\boldsymbol{C})=\Tr(\boldsymbol{WCW}^{\top})$, with $\boldsymbol{C}\succeq 0$. Then, we have:
\begin{equation}
    \mathbf{H}_{P}\preceq \mathbf{H}_{G},
\end{equation}
and consequently, the dominant eigenvalues satisfy:
\begin{equation}
    \lambda_{\max}(\mathbf{H}_{P}) \le\lambda_{\max}(\mathbf{H}_{G}).
\end{equation} 
\end{proposition}

\begin{proof}
    The detailed proof is provided in the Supplementary Material IV.
\end{proof}

\begin{remark}
    Proposition 3 indicates that ProML can attenuate the contribution of the Hessian matrix. For any parameter direction, the local curvature  of ProML doesn't exceed the corresponding curvature of GMLL. This implies that ProML globally smooths the loss landscape. Based on this, we can obtain the corollary regarding the condition number and the step size.
\end{remark}

\begin{figure}
    \centering
    \includegraphics[width=0.9\linewidth]{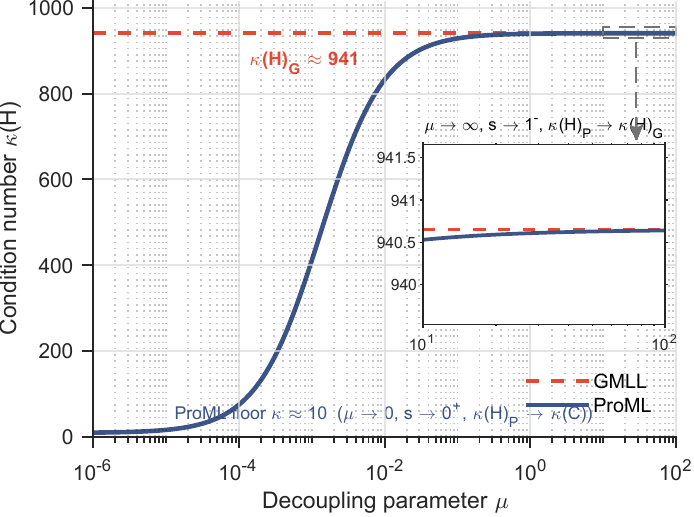}
    \caption{Condition-number evolution of ProML on the \textit{Emotions} dataset. The blue curve shows $\kappa(\mathbf H_P)$ as a function of the decoupling parameter $\!\mu\!$, and the red dashed line denotes the GMLL condition number $\kappa(\mathbf H_G)$. As $\!\mu\!\to\! 0, s\!\to\!0^+$, ProML approaches the structural regularization dominated regime with $\!\kappa(\mathbf H_P)\!\to\!\kappa(\boldsymbol{C})$. As $\!\mu\!\to\!\infty, s\!\to\! 1^-$, ProML recovers GMLL and $\kappa(\mathbf H_P)\!\to\!\kappa(\mathbf H_G)$. }
    \label{fig:kappa}
    \vspace{-2.5mm}
\end{figure}

\begin{corollary}[\textbf{Condition Number Compression}]\label{pro_condi}
Let $\kappa(M)=\frac{\lambda_{\max}(M)}{\lambda_{\min}(M)}$ denote the condition number of a symmetric positive definite matrix. Under the same quadratic structural regularization setting as in Proposition \ref{pro_compre} with $\boldsymbol{C}\succ 0$, assume that the feature space is strictly more ill-conditioned than the label correlation space in the following spectral sense:
\begin{equation}
    \frac{\lambda_{\max}(\mathbf{X}^{\top}\mathbf{X})}{\lambda_{\min}(\mathbf{X}^{\top}\mathbf{X})} > \frac{\lambda_{\max}(\boldsymbol{C})}{\lambda_{\min}(\boldsymbol{C})}.
\end{equation}
Then,  for any $\mu>0$, the condition number of the Hessian matrix $\mathbf{H}_{P}$ is strictly smaller than $\mathbf{H}_{G}$:
\begin{equation}
    \kappa(\mathbf{H}_{P})<\kappa(\mathbf{H}_{G}).
\end{equation}	
\end{corollary}

\begin{proof}
    The detailed proof is provided in the Supplementary Material V.
\end{proof}

\begin{remark}
The condition number compression effect established in Corollary 4 provides a theoretical foundation for the stability advantage of ProML from the perspective of the spectral structure of the optimization landscape. 
The condition number reflects the degree of ill-posedness of the optimization landscape of the objective function, while ProML compresses the spectral distribution of the fitting term via the scaling factor $s$, effectively suppressing the maximum curvature, improving the condition number of the problem, and alleviating the optimization landscape of the objective function. 
Fig. \ref{fig:kappa} presents the evolution curve of as a function of the scaling factor $s$ on the \textit{Emotions} dataset. As $s\to 0^+$, the Hessian is dominated by the structural regularization term and $\kappa(\mathbf{H}_{P})\to\kappa(\boldsymbol{C})$. As $s\to 1^-$, ProML gradually recovers GMLL and $\kappa(\mathbf{H}_{P})\to \kappa(\mathbf{H}_{G})$.
This property directly explains the empirical observation that ProML permits a larger stable learning rate and exhibits stronger convergence robustness during optimization, as demonstrated by the following corollary.
\end{remark}

\begin{figure}[t]
    \centering
    \includegraphics[width=0.9\linewidth]{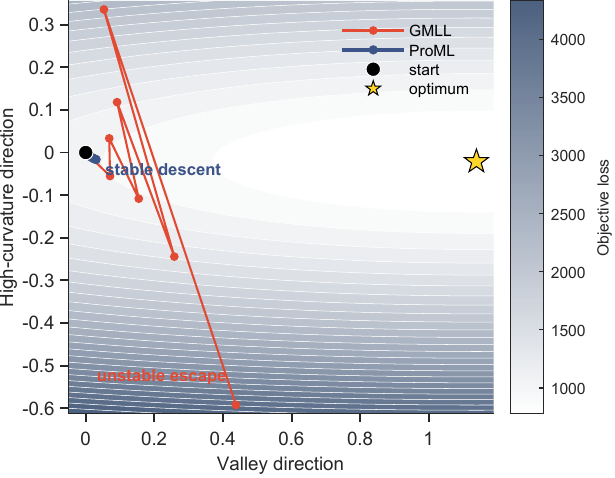}
    \caption{Projected optimization trajectories with a super-critical learning rate on the \textit{Emotions} dataset. The learning rate is set above the critical threshold of GMLL but remains within the stable range of ProML, \textit{i.e.}, $\!\eta_{\mathrm{crit}}^{G}\!<\!\eta\!:=\!1.3\! \times \!\eta_{crit}^G\!<\!\eta_{\mathrm{crit}}^{P}\!$. In this setting, GMLL oscillates along the high-curvature direction and escapes to a high-loss region, while ProML maintains stable descent toward the optimum. Since $\eta$ has not reached the optimal learning rate for ProML, the rate of decline is slow.}
    \label{fig:landscape}
    \vspace{-2.5mm}
\end{figure}

\begin{corollary}[\textbf{Step-size Broadening}]\label{pro_step}
Let $\eta_{crit}=\frac{2}{\lambda_{\max}(\mathbf{H})}$ denote the maximum critical learning rate. For any $\mu> 0$, and under the same quadratic structural regularization setting as in Proposition \ref{pro_compre} with $\boldsymbol{C}\succeq 0$, ProML achieves a  broadening of the secure learning rate interval for GMLL. Then, we have 
\begin{equation}
    \eta_{\text{crit}}^{G} \le \eta_{\text{crit}}^{P} \le \left(\frac{1+\mu}{\mu}\right) \eta_{\text{crit}}^{G}.
\end{equation}
\end{corollary}

\begin{proof}
    The detailed proof is provided in the Supplementary Material VI.
\end{proof}

\begin{remark}
Corollary 5 explains, from a numerical optimization standpoint, why ProML exhibits stronger robustness than GMLL under gradient-based optimization. For any given learning rate $\eta$, when $\eta$ exceeds ${2}/{L}_{G}$ but remains below ${2}/{L}_{P}$, the GMLL update step has already surpassed its stability boundary, causing the optimization trajectory to oscillate violently along the dominant curvature direction or even diverge. In contrast, ProML maintains steady convergence due to its broader stable interval.
Fig. \ref{fig:landscape} visualizes the optimization trajectories of GMLL and ProML with super-critical learning rate on the \textit{Emotions} dataset. Observation shows that the optimization trajectory of GMLL exhibits severe oscillation, while the trajectory of ProML maintains a smooth downward trend.
Fig. \ref{fig:loss_grad} reports the training trajectories of the loss values and gradient norms for both GMLL and ProML under the same super-critical step-size setting. As shown in Fig. \ref{fig:loss_grad}, GMLL exhibits exponential growth in both loss and gradient norm, while ProML converges smoothly with monotonically decreasing loss and rapidly decaying gradient norms.
\end{remark}

\begin{figure}
    \centering
    \includegraphics[width=0.9\linewidth]{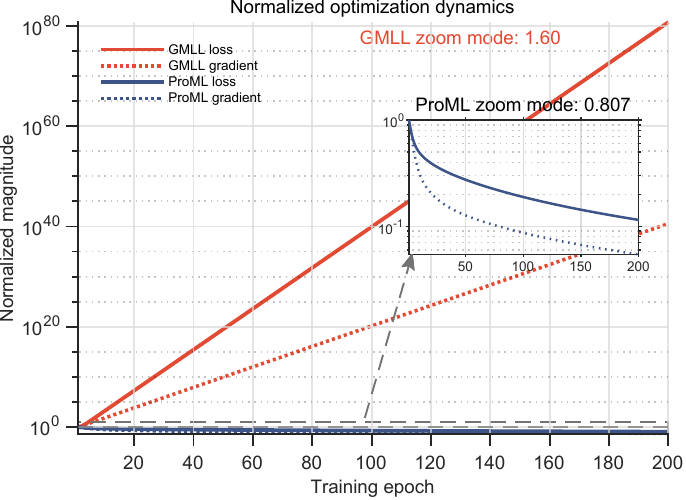}
    \caption{Normalized loss and gradient-norm trajectories with a super-critical learning rate on the \textit{Emotions} dataset. When learning rate is set to $\!\eta\!:=\!1.3\! \times \!\eta_{crit}^G$, both the loss and gradient-norm of GMLL exhibit explosive growth, while ProML maintains a stable decay throughout the entire training process. The zoom-mode factors explain this contrast: the largest-curvature component is amplified by approximately $1.60$ per iteration for GMLL, but contracted by approximately $0.807$ per iteration for ProML.}
    \label{fig:loss_grad}
    \vspace{-3mm}
\end{figure}

The above properties are established under the setting where the regularization term $\mathcal{R}(\boldsymbol{W})$ is omitted, aiming to highlight the regulatory effect of the core decoupling mechanism of ProML on the Hessian. After introducing $\mathcal{R}(\boldsymbol{W})$, the above conclusions still remain valid through appropriate modifications.  
Specifically, when $\mathcal{R}(\boldsymbol{W})$ is a smooth convex function, adding the semi-positive component $\nabla^2 \mathcal{R}(\boldsymbol{W})\succeq 0$ to the Hessian matrix of the objective function,  the semi-positive relationship of spectral compression remains unaffected, and ProML continues to maintain its condition number advantage over GMLL. 
When $\mathcal{R}(\boldsymbol{W})$ is a non-smooth function, the non-smooth terms disrupt the explicit representation of the global Hessian matrix at singular points. However, by introducing a proximal operator at the algorithmic level, ProML retains its compression mechanism for the effective smoothing constant of the supervised fitting term.

Based on the above theoretical results concerning Hessian spectrum analysis and optimization stability, this section further examines the effect of ProML on controlling the complexity of the hypothesis space from the perspective of generalization theory.

\begin{proposition}[\textbf{Effective Hypothesis Class Contraction}]\label{pro_eff_hy}
    Let $\mathcal{F}_{G}\!=\!\{\boldsymbol{W}\!:\|\boldsymbol{W}\|_{F}\le B_{1}\}$ and $\mathcal{F}_{P}\!=\!\{\boldsymbol{W}\!:\|\boldsymbol{W}\|_{F}\le B_{s}\}$ denote the empirical hypothesis space of GMLL and ProML, respectively, and let $S$ be a sample of size $n$. Then, for any $\mu>0$ and $\boldsymbol{C}\succ 0$, the empirical Rademacher complexity on $S$ satisfies
\begin{equation}
    \hat{\mathfrak{R}}_{S}(\mathcal{F}_{P}) \le \frac{\mu}{1+\mu} \hat{\mathfrak{R}}_{S}(\mathcal{F}_{G}).
\end{equation}
\end{proposition}

\begin{proof}
    The detailed proof is provided in the Supplementary Material VII.
\end{proof}

\begin{remark}
    As shown in Fig. \ref{fig:rade_com}, the Rademacher complexity contraction established in Proposition 6 should be interpreted as an effective capacity contraction, rather than a topological compactification of the original parameter space. ProML doesn't alter the ambient space of parameter matrix, rather, the scaling factor $s$ induced by profiled elimination leads to a smaller norm bound on the solution set. The empirical Rademacher complexity of a linear function class is proportional to the radius of the norm ball, ProML effectively controls the hypothesis space complexity through this radius contraction.
\end{remark}

\begin{figure}
    \centering
    \includegraphics[width=0.9\linewidth]{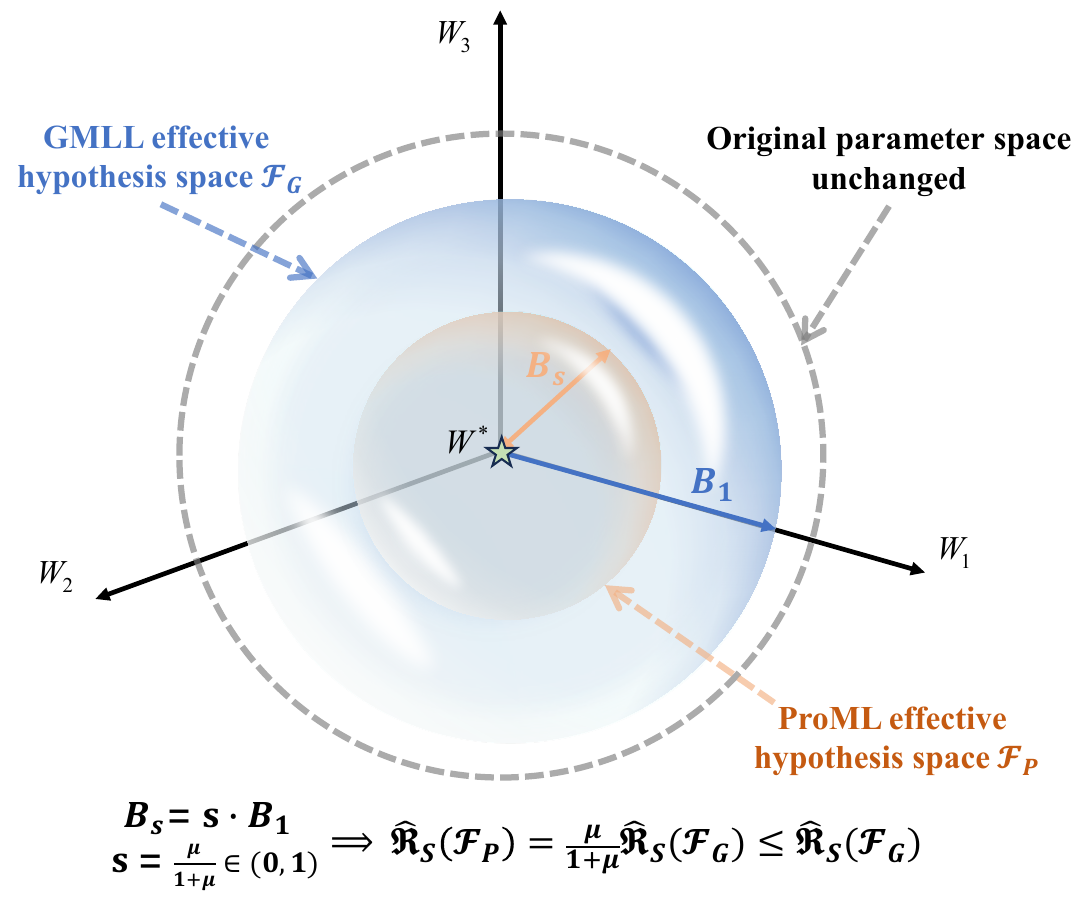}
    \caption{Visualization of the effective hypothesis class contraction. The gray dotted sphere indicates that the original parameter space remains unchanged. The blue and orange spheres denote the GMLL effective hypothesis class $\mathcal{F}_G$ with radius $B_1$ and the ProML effective hypothesis class $\mathcal{F}_P$ with radius $B_s\!=\!s\cdot B_1$, respectively, where $s\!=\!\mu/(1\!+\!\mu)$. This radius contraction leads to the empirical Rademacher complexity bound $\hat{\mathfrak{R}}_{S}(\mathcal{F}_P)\!\le\! \hat{\mathfrak{R}}_{S}(\mathcal{F}_G)$.}
    \label{fig:rade_com}
    \vspace{-2.5mm}
\end{figure}

\vspace{-0.2mm}
\section{Instantiating ProML for Feature Selection}\label{sec:ProML_fs}
\subsection{Task Formulation}
Multi-label feature selection, as the effective dimensionality reduction method within multi-label learning community to address the high-dimensional feature curse, aims to select the most discriminative features from the original feature space to improve learning performance and reduce the model complexity \cite{zhang2009feature}.
As we know, each sample is associated with multiple labels, and there are correlations among these labels. Thus, utilizing label information is a critical step to improve model performance. In addition, employing sparse learning to learn the feature weight matrix is a plausible and effective feature selection approach. ProML exactly meets these design requirements, as it simultaneously involves label prediction, parameter sparsity control and label dependency modeling. Further, to avoid work similar to \cite{huang2017joint,zhang2025multi}, the number of selected features is specified  manually after the weight matrix has been learned. We adopt the more sparsity-inducing $\ell_{2,1/2}$-matrix norm as the regularization constraint here. To achieve this goal, we can formulate the ProML problem into the following optimization framework, denoted as ProFS:
\begin{equation}
\begin{aligned}
    \min_{\boldsymbol{W},\mathbf{P}} \frac{1}{2}\|\mathbf{P}\!-\!\mathbf{Y}\|_{F}^2 & + \frac{\mu}{2} \|\mathbf{X}\boldsymbol{W}\!-\!\mathbf{P}\|_{F}^2\\
    & + \frac{\alpha}{2} \Tr(\boldsymbol{W\!RW}\!^{\top}\!) +\frac{\lambda}{2} \|\boldsymbol{W}\|_{2,1/2}^{1/2},
\end{aligned}
\label{equ:objective}
\end{equation}
where $\mu, \alpha, \lambda$ are trade-off parameters. $\Tr(\cdot)$ is the label correlation regularization that inspired by \cite{huang2017joint}, where $r_{ij}=1-c_{ij}$, and $c_{ij}$ indicates the correlation between $y_{i}$ and $y_{j}$.

\subsection{Asymmetrical Label Dependency Construction}
Previous studies have indicated that the relationships between labels are inherently asymmetric. Therefore, how to characterize these asymmetric relationships between labels is a key factor in enhancing the performance of models. Specifically, we consider this problem from both global and local perspectives.

\vspace{2mm}
\textit{\textbf{Global Dependence Correlations}}. In the multi-label learning, all labels intrinsically depend on a shared feature space, and their mappings onto this space effectively characterize their value information. Since features have different degrees of influence on different labels, we expect to extract the most discriminative features for each label. Therefore, based on the idea of feature-label dependence, we explored label correlations from a global perspective.

First, considering the difference in data representation types between features and labels, we perform data augmentation on the label space to obtain label distribution information \cite{geng2016label}, which reflects not only which labels a sample carries but also the degree to which each label describes the instance. Formally, let $\tilde{y}_j(x)$ denote the significance degree of label $y_j$ for sample $x$, and let $[\delta(x)]_{Y}$ and $[\delta(x)]_{y_j}$ denote, respectively, the neighborhood sets of sample $x$ over the entire label set $\mathbf{X}$ and restricted to label $y_j$. Then, the label significance of each label for sample $x$ is defined as \cite{kou2023novel}: $\tilde{y}_j(x) = {\mid[\delta_C(x)]_{y_j}\mid} \,/\, {\mid[\delta_C(x)]_{Y}\mid}$. 

Subsequently, \textit{MIC}, a typical association measure, is employed to characterize the dependence between features and labels, so as to clarify the dependency degree of each feature on each label. Let $r(f_i, \tilde{y}_j)$ denote the dependence between feature $f_i$ and label $y_j$, and let $\mathbf{r}_j \!=\! [r(f_1, \tilde{y}_j),\! \dots\!,\! r(f_d, \tilde{y}_j)]^\top\! \in\! [0,1]^d$ be the dependence vector of the feature space \textit{w.r.t.} label $y_j$. After that, the feature-label dependencies are fused into label correlations via the \textit{Sugeno integral} \cite{marichal2000sugeno}, whose procedure relies on ordering the input values. For a given label $y_i$, the components of its dependence vector $\mathbf{r}_i$ are sorted in ascending order and normalized to sum to one, yielding $r(f_1^*, \tilde{y}_i)\! \le\! \cdots\! \le r(f_d^*, \tilde{y}_i)$, after which they are combined with the corresponding increasing fuzzy measure values through a min-max operation.

Here, the dependence values $\mathbf{r}_i$ of features on label $y_i$ serve as the input to the \textit{Sugeno integral}, while the cumulative dependence $r(f_{\{k,\dots,d\}}^*,\!\tilde{y}_j)\!=\! \sum_{t=k}^d\! r({f_t^*},\!\tilde{y}_j)$ of the other label $y_j$ acts as the fuzzy measure. This cross-design ensures that the integrated result simultaneously depends on the dependence distributions of both labels. 

Thus, for any label variable pair $(y_i, y_j)$, the correlations can be expressed as:
\begin{equation}
    \boldsymbol{r}({y}_i, {y}_j)\!=\!
\begin{cases}
\underset{1\leq k\leq d}{\max} \big(\!r({f_k^*},\tilde{y}_i) \!\wedge\! r(\!{f_{\{k,\cdots,d\}}^*}, \tilde{y}_j)\!\big),\! & i\neq j\\ 
1, \!& i=j,
\end{cases}
\end{equation}
where $\land$ denotes the minimum operator. This yields the global label correlation matrix $\mathfrak{G}(\mathbf{Y},\mathbf{Y})\!=\!(\boldsymbol{r}({y}_i, {y}_j))_{c \times c}$. 

It is worth emphasizing that correlation matrix $\mathfrak{G}$ inherently possesses an asymmetric property, \textit{i.e.}, $\mathfrak{G}(y_i, y_j) \!\neq\! \mathfrak{G}(y_j, y_i)$. The fundamental reason lies in the directional nature built into the ordering mechanism of the \textit{Sugeno integral}. When computing $\mathfrak{G}(y_i, y_j)$, features are sorted in ascending order of their dependence on $y_i$ while the fuzzy measure is drawn from the cumulative dependence of $y_j$. When computing $\mathfrak{G}(y_j, y_i)$, the ordering basis and the measure role are reversed. Since the value patterns of the dependence vectors $\mathbf{r}_i$ and $\mathbf{r}_j$ of $y_i$ and $y_j$ are generally different, equality holds unless the two labels share exactly the same dependence distribution over all features. 
This structural property precisely encodes the intrinsic difference between the inferential strength of label $y_i$ \textit{w.r.t.} $y_j$ and that of $y_j$ \textit{w.r.t.} $y_i$.

\begin{figure}[!t]
\vspace{-2mm}
\centering
\begin{algorithm}[H]
\caption{Algorithm in obtaining label correlations.}
\label{alg:CALC}
\begin{algorithmic}[1]
\setlength{\parskip}{0pt}
\setlength{\itemsep}{0pt}
    \renewcommand{\algorithmicrequire}{\textbf{Input:}}
    \renewcommand{\algorithmicensure}{\textbf{Output:}}
    \REQUIRE Training data $\mathbf{X}\in \mathbb{R}^{n\times d}$, and label matrix $\mathbf{\mathbf{Y}}\in \mathbb{R}^{n\times c}$.
    \ENSURE Asymmetric label correlation matrix $\boldsymbol{C}^{c\times c}$.
    \STATE Initialize $\mathfrak{G} \leftarrow \mathbf{0}_{c\times c}$, $\mathfrak{C} \leftarrow \mathbf{0}_{c\times c}$;
    \STATE Compute label distribution $\tilde{\mathbf{Y}}$;
    \FOR{each label $y_j$ and each feature $f_l$}
    \STATE Compute the dependency vector ${\bf r}_j$ by \textit{MIC};
    \ENDFOR
    \FOR{each label pair $(i,j)$ with $i\ne j$}
    \STATE Sort the elements of $r_i$ in non‑decreasing order: $r(f_1^*, \tilde{y}_i)\! \le\! \cdots \le r(f_d^*, \tilde{y}_i)$;
    \STATE Define the cumulative sum $r(f_{\{k,\dots,d\}}^*, \tilde{y}_j)$;
    \STATE Compute the label correlation with \textit{Sugeno integral} according to $\boldsymbol{r}({y}_i, {y}_j)$;
    \ENDFOR
    \STATE Obtain the global label correlation according to $\mathfrak{G}(\mathbf{Y},\mathbf{Y})$;
    \STATE Obtain sample groups $\{G_l\}_{l=1}^k$ by $k$-means and $k$ is determined by \textit{Median‑Drop Rule};
    \FOR{each group $G_m$}
    \STATE Compute local correlation matrix $S_l$ by \textit{dCor};
    \ENDFOR
    \STATE Obtain the label correlation according to $\mathfrak{C}(\mathbf{Y},\mathbf{Y})$;
    \STATE Fuse the label correlation matrix $\boldsymbol{C}$;
    \RETURN $\boldsymbol{C}$.
\end{algorithmic}
\end{algorithm}
\vspace{-4mm}
\end{figure}

\vspace{2mm}
\textit{\textbf{Local Sample Label Correlations}}. Similar samples typically have similar labels \cite{liu2021emerging}. Therefore, classification models that account for local label correlations are better suited to real-world problems. Motivated by \cite{huang2012local, zhang2025multi}, we partition the training samples into $m$ groups $\{G_1,\cdots,G_m\}$ via clustering, with each group sharing a subset of label correlations. Unlike \cite{huang2012local, zhang2025multi} and related works that cluster in the label space to identify sample groups with similar label combinations, this paper focuses on clustering in the feature space, aggregating samples by the similarity of their feature distributions so that samples within the same group share similar label patterns. Specifically, we adopt \(k\)-means clustering combined with the \textit{\textbf{Median-Drop Rule}} to adaptively determine the number of clusters and obtain the local sample groups.

\vspace{2mm}
\textit{\textbf{Median‑Drop Rule}.} For candidate cluster numbers $k \!=\! 2, 3,\! \dots,\! K_{\max}$, we successively perform $k$-means and compute the adjacent sum of squared errors (\textit{SSE}) drops $\Delta_k = \textit{SSE}(k) - \textit{SSE}(k+1)$, selecting the first $k$ whose $\Delta_k$ reaches the median of all drops and taking $k+1$ as the final cluster number. This indicates that when the SSE reduction from adding a new cluster reaches the median level of the overall drops, the marginal benefit of further increasing the cluster number has diminished to an ordinary level.

Based on the clustering results, we compute the label correlations within each group using the distance correlation (\textit{dCor}) measure. Let $Y_l$ be the label subspace corresponding to group $G_l$ and $S_l$ denote the label correlation subset shared by the internal instances of $G_l$. Then, the label correlation matrix on $G_l$ is defined as $S_{l}(Y,Y) = (\textit{dCor}_{l}(\tilde{y}_{i},\tilde{y}_{j}))_{c\times c}$, and weighting each group by its sample proportion $\theta_l$ yields the label correlation at the local sample scale: $\mathfrak{C}(\mathbf{Y},\mathbf{Y}) = \sum_{l=1}^{m} \theta_l S_{l}(Y,Y),$ where $\theta_l = |G_l| / \sum_{i=1}^m |G_i|$.

\vspace{2mm}
\textit{\textbf{Fusion}}. The final correlation matrix is obtained by element-wise addition of $\boldsymbol{C} = \mathfrak{G}+\mathfrak{C}$, with all entries constrained to $[0,1]$ and the diagonal fixed to $1$, so as to satisfy numerical boundedness and semantic consistency. The resulting fused matrix simultaneously retains the asymmetry imparted by the feature-label dependence and the correlation structure contributed by the local sample groups. Based on this, the complete algorithm process is shown in Algorithm \ref{alg:CALC}.

\subsection{Optimization}
Based on the objective function and regularization described above, the objective function in Eq.(\ref{equ:objective}) can be solved by the gradient descent algorithm. However, the $\ell_{2,1/2}$-matrix norm is non-convex, which increases the difficulty of the solution. 

According to the definition of $\ell_{2,p}$-norm \cite{wang_unified_2014}, we have
\begin{equation*}
	 \| \boldsymbol{W} \|_{2,p}^{p} = {2}/{p}\ \Tr(\boldsymbol{W^{\top}\!DW)},
\end{equation*}
where $\boldsymbol{D}$ is diagonal matrix, \textit{i.e.,} $\boldsymbol{D}_{ii} = \frac{p}{2}{\left \| w^i \right \|_{2}^{p-2}}$, and $w^i$ denotes the $i$-row of $\boldsymbol{W}$.

The Frobenius norm $\!\|\mathbf{X}\!\boldsymbol{W}\!-\!\mathbf{P}\|_{F}^2$ is equivalent to $\Tr((\mathbf{X}\boldsymbol{W}\!-\!\mathbf{P})\!^{\top}\allowbreak (\mathbf{X}\boldsymbol{W}\!-\!\mathbf{P}))$. 
Therefore, the objective function in Eq.(\ref{equ:objective}) is equivalent to
\begin{align}
	\notag
	\min_{\boldsymbol{W}, \mathbf{P}} \frac{1}{2} &\Tr(\mathbf{(P\!-\!Y)}\!^{\top}\!\mathbf{(P\!-\!Y)}) \!+\! \frac{\mu}{2}\Tr((\mathbf{X}\!\boldsymbol{W}\!-\!\mathbf{P})\!^{\top}\!(\mathbf{X}\!\boldsymbol{W}\!-\!\mathbf{P}))\\
	& +\! \frac{\alpha}{2}\Tr(\boldsymbol{WRW\!^{\top}}) \!+\! 2{\lambda}\Tr(\boldsymbol{W\!^{\top}\!DW}).
	\label{equ:obje_f}
\end{align}

Taking the derivative of Eq.(\ref{equ:obje_f}) \textit{w.r.t.} $\mathbf{P}$ and setting it equal to zero, we have
\begin{equation*}
	\mathbf{(P-Y)} - \mu (\mathbf{X}\boldsymbol{W}-\mathbf{P}) = 0.
\end{equation*}
Then the prediction labels $\mathbf{P}$ is obtained by
\begin{equation}
	\mathbf{P} = (\mathbf{I+\mu I})\!^{-1} (\mu \mathbf{X}\boldsymbol{W} + \mathbf{Y}),
	\label{equ:PQ}
\end{equation}
where $\mathbf{I}$ is identity matrix.

Let $\mathbf{M=(1 + \mu )}\!^{-1}\mathbf{I}$ and $\mathbf{N}=(\mu\mathbf{X}\boldsymbol{W} + \mathbf{Y})$, then Eq.(\ref{equ:PQ}) can be rewritten as $\mathbf{P=MN}$.

Substituting into Eq.(\ref{equ:obje_f}), the objective function becomes
\begin{equation}
    \begin{aligned}
	\min_{\boldsymbol{W}} \frac{1}{2} \Tr(\mu\boldsymbol{W}\!^{\top}\! \mathbf{X}\!^{\top}\!& \mathbf{X} \boldsymbol{W} \!-\! \mathbf{N}\!^{\top} \mathbf{M}\!^{\top} \mathbf{N}) \!+\! \frac{\alpha}{2}\!\Tr(\boldsymbol{WRW\!^{\top}}) \\ 
	&\qquad + 2{\lambda}\Tr(\boldsymbol{W\!^{\top}\!DW}).
	\label{equ:pq}
\end{aligned}
\end{equation}
By defining $\mathbf{N}\!=\!(\mu\mathbf{X}\boldsymbol{W} \!+\! \mathbf{Y})$ and substituting it into objective function Eq.(\ref{equ:pq}), then we get
\begin{equation}
\begin{aligned}
	\min_{\boldsymbol{W}} \frac{1}{2}\Tr(\boldsymbol{W}\!^{\top}\!\mathbf{A}\!\boldsymbol{W}& - \mathbf{B}) + \frac{\alpha}{2}\Tr(\boldsymbol{WRW\!^{\top}}) \\
    &\quad + 2{\lambda}\Tr(\boldsymbol{W\!^{\top}\!DW}),
\end{aligned}
\end{equation}
where 
\begin{gather}
    \mathbf{A}=\frac{\mu}{1+\mu}\mathbf{X}\!^{\top}\mathbf{X}, \\
	\mathbf{B}= \frac{\mu}{1+\mu} (\boldsymbol{W}\!^{\top}\mathbf{X}\!^{\top} \mathbf{Y} + \mathbf{Y\!^{\top}}\mathbf{X} \boldsymbol{W}). \label{equ:bb}
\end{gather}

For the sake of simplicity, according to the above function and the definition of $s=\frac{\mu}{1+\mu}$ in Proposition \ref{pro_profile}, we assume 
\begin{equation}
\begin{aligned}
    f_{\boldsymbol{D}}\!(\boldsymbol{W}) \!= & \frac{1}{2}\Tr(\boldsymbol{W}\!^{\top}\! \mathbf{A}\! \boldsymbol{W}) - s\!\Tr(\boldsymbol{W}\!^{\top}\mathbf{X}\!^{\top} \mathbf{Y})  \\
    &\qquad \!+ \frac{\alpha}{2}\Tr(\boldsymbol{W R W\!^{\top}})  + 2{\lambda}\Tr(\boldsymbol{W\!^{\top}\! D W}).
\end{aligned}
\label{obj}
\end{equation}

For fixed $\!\boldsymbol{D}\!\succeq\! 0$, Eq.(\ref{obj}) defines a convex quadratic surrogate provided that $\boldsymbol{R}\succeq 0$ is positive semi-definite. The surrogate objective function $f_{\boldsymbol{D}}(\boldsymbol{W})$ is convex and smooth, we seek to solve it by gradient descent method. 
Let $\nabla f_{\boldsymbol{D}}(\boldsymbol{W})$ denotes the derivative of $f_{\boldsymbol{D}}(\boldsymbol{W})$ \textit{w.r.t.} $\boldsymbol{W}$, which can be represented as:
\begin{equation}
\nabla\!f_{\boldsymbol{D}}(\boldsymbol{W})=\mathbf{A}\boldsymbol{W}-s\mathbf{X\!^{\top}Y} + \frac{\alpha}{2}\boldsymbol{W}\tilde{\boldsymbol{R}}+4\lambda\boldsymbol{D\!W},
\label{equ:deri}
\end{equation}
where $\tilde{\boldsymbol{R}}=\boldsymbol{(R\!^{\top}\!+\!R)}$. 

With the above formula, we use the following updating rule for the optimization solution: 
\begin{equation}
	\boldsymbol{W_{t+1}} = \boldsymbol{W_{t}} - \eta_t  \nabla f_{\boldsymbol{D}}(\boldsymbol{W_{t}}),
	\label{equ:upda}
\end{equation}
where $\eta_t >0$ is the step size for each update iteration.

The pseudo-code of the optimization algorithm is summarized as shown in Algorithm \ref{alg:ProML}. 
Before presenting the empirical evaluation, we establish theoretical analysis for the proposed model \textit{w.r.t.} optimization and generalization. Specifically, we first characterize the optimization behavior by establishing the monotonic descent and convergence of the objective sequence. We then derive generalization bounds based on Rademacher complexity and uniform stability, respectively, thereby quantifying the model's generalization capability on unseen data.

\begin{figure}[t]
\vspace{-3mm}
\centering
\begin{algorithm}[H]
\caption{The ProFS Algorithm.}
\label{alg:ProML}
\begin{algorithmic}[1]
    \renewcommand{\algorithmicrequire}{\textbf{Input:}}
    \renewcommand{\algorithmicensure}{\textbf{Output:}}
    \REQUIRE Training data $\mathbf{X}\in \mathbb{R}^{n\times d}$, label matrix $\mathbf{Y}\in \mathbb{R}^{n\times c}$, label correlation matrix $\boldsymbol{C}^{c\times c}$, parameters $\mu, \alpha, \lambda$ and $\eta$.
    \ENSURE Sparse weight matrix $\boldsymbol{W}^{d\times c}$.
    \STATE Randomly initialize: $\boldsymbol{W}^{(0)}$;
    \STATE Compute $\mathbf{M = (1 + \mu )\!^{-1}I}$ and $\mathbf{A} = \frac{\mu}{1+\mu}\mathbf{X}\!^\top\mathbf{X}$;
    \REPEAT
    \STATE Compute predicted label $\mathbf{P}_{t}$ via Eq.(\ref{equ:PQ});	
    \STATE Compute matrix $\!\boldsymbol{D}_{t}\!$ with $\small{[\boldsymbol{D}_{t}]_{ii} = \frac{1}{4(\|\mathbf{w}_t^i\|_2^2 + \epsilon)^{3/4}}}$;
    \STATE Derivation \textit{w.r.t.} $\boldsymbol{W}_{t}$ according to Eq.(\ref{equ:deri});
    \STATE Update $\boldsymbol{W}_{t+1}$ by Eq.(\ref{equ:upda});
    \STATE $t \leftarrow t+1$;
    \UNTIL {stop criterion reached};
    \RETURN $\boldsymbol{W}$.
\end{algorithmic} 
\end{algorithm}
\vspace{-6mm}
\end{figure}

\subsection{Convergence Property}
The convergence describes the whether the algorithm gradually approaches the optimal solution step by step during the iteration process, and also reflects whether the problem itself is solvable. This is one of the metrics for evaluating the performance of the algorithm. The convergence of the objective function in ProFS is summarized as follows.

\begin{proposition}\label{cor:conv}
Let $\{\boldsymbol{W}_{t}\}_{t \ge 0}$ be the sequence of optimization processes generated by the gradient update rule in Eq.\eqref{equ:upda}. Suppose that for each row $i$ and each iteration $t$, $\boldsymbol{w}_t^i \neq \mathbf{0}$, and that the step size $\eta_t > 0$ is determined by the Lipschitz constant of $\nabla f_{\boldsymbol{D}}$ with respect to the Frobenius norm. Then, the objective function $f_{\boldsymbol{D}}(\boldsymbol{W})$ defined in Eq.\eqref{obj} is monotonically non-increasing along the iterations, \textit{i.e.},
\begin{equation}
    f_{\boldsymbol{D}}(\boldsymbol{W}_{t+1}) \le f_{\boldsymbol{D}}(\boldsymbol{W}_{t}), \quad \forall t \ge 0.
    \label{eq:cor_mono}
\end{equation}
Consequently, the sequence $\{f_{\boldsymbol{D}}(\boldsymbol{W}_{t})\}$ converges to a finite value, and any limit point of $\{\boldsymbol{W}_{t}\}$ is a stationary point of Eq.\eqref{obj} satisfying the first-order optimality condition.
\end{proposition}

\begin{proof}[Proof Sketch]
First, Lemma VIII.1 provides the scalar inequality for optimizing the nonconvex $\ell_{2,1/2}$-regularizer. Applying this inequality row-wise yields Lemma VIII.2. Combining the resulting bound with the descent property of the quadratic sub-problem yields the monotonic decreasing property of $f_{\boldsymbol{D}}$. The detailed proof is provided in Supplementary Material VIII.
\end{proof}

\begin{remark}
Proposition 7 shows that the optimization procedure of ProFS follows the monotone decreasing principle. At each iteration, the nonconvex $\ell_{2,1/2}$-regularizer is locally represented by a reweighted quadratic surrogate, while the Lipschitz-based step size ensures a stable descent of the smooth surrogate objective. Consequently, the interaction between the quadratic fitting–structure terms and the row-sparsity constraint does not cause an increase in the overall objective during optimization. 
\end{remark}

Next, we present the generalization bounds \textit{w.r.t.} the Rademacher complexity of the hypothesis space and the uniform stability of the algorithm, thereby clearly demonstrating the model’s generalization capability. 

\subsection{Generalization Bound with Rademacher Complexity}
Generalizability measures how well a model performs on unseen data, which is a critical metric for evaluating learning algorithm. Here, we derive a generalization bound for ProFS using Rademacher complexity.
It should be noted that the original ProFS objective contains the nonconvex $\ell_{2,1/2}$-matrix norm. Therefore, the following analysis is conducted on the convex quadratic surrogate obtained after fixing the reweighting matrix $\boldsymbol{D}$, rather than directly on the original nonconvex objective.

\begin{proposition}\label{the:gen_f}
Let $\mathcal{W}$ be the class of linear predictors used in Eq.\eqref{obj}, where $\|\mathbf{x}\|_2\!\le\!1$ for all instances. Suppose that the loss function $\ell$ is $s\sqrt{c }$-Lipschitz and bounded in $[0,M]$. Then for any $\delta \in (0,1)$ with probability at least $1-\delta$ over a sample $S$ of size $n$, the  following inequality holds for all $\boldsymbol{W}\in \mathcal{W}$,
\begin{equation}
	R(\boldsymbol{W}) \le \hat{R}_S(\boldsymbol{W})+\frac{\mu2\sqrt{2}Qc}{(1+\mu)\sqrt{n}}+ 3M\sqrt{\frac{\ln(2/\delta)}{2n}},
	\label{euq:ld}
\end{equation}
where $R(\boldsymbol{W}) = \mathbb{E}_{(\mathbf{x},\mathbf{y})}[\ell(\boldsymbol{W};(\mathbf{x},\mathbf{y}))]$ is the expected risk,  $\hat{R}_S(\boldsymbol{W}) = \frac{1}{n}\sum_{i=1}^n \ell(\boldsymbol{W};(\mathbf{x}_i,\mathbf{y}_i))$ is the empirical risk, and $s=\frac{\mu}{1+\mu}$, $Q\!=\!\max\{\|\boldsymbol{W}\|_F:\boldsymbol{W}\in \mathcal{W}\}$.
\end{proposition}

\begin{proof}[Proof Sketch]
    First, we use the vector contraction inequality (Lemma I.4) to transform the empirical Rademacher complexity of loss function into linear predictor. Then, we upper bound the Rademacher complexity, and by combing with Lemma I.5, the desired bound can be derived. The detailed proof is provided in Supplementary Material IX.
\end{proof}

\begin{remark}
    Proposition \ref{the:gen_f} shows that the generalization error of ProFS converges at a rate of $O(1/\sqrt{n})$. The bound is controlled by the feature radius $\|x\|$, predictor norm $Q$, label dimension $c$, and decoupling parameter $\mu$, indicating that a smaller model norm as well as the scaling factor $s\!=\!\frac{\mu}{1+\mu}$ and a larger training set lead to better generalization.
\end{remark}

However, Rademacher complexity only provides a distribution dependent measure of the capacity of the hypothesis space, it doesn't capture the effect of the optimization trajectory (\textit{e.g.}, number of iterations, step sizes) on generalization ability. 

\subsection{Generalization Bound with Uniform Stability}
We further analyze the generalization behavior of the gradient descent algorithm used to optimize the objective function. Unlike the stochastic gradient descent (SGD) setting studied by Hardt et al. \cite{hardt2016train}, our update rule at each iteration accesses the entire training set, which fundamentally alters the stability analysis. 

A key property for analyzing gradient-based methods is smoothness, \textit{i.e.}, the Lipschitz continuity of the gradient, which governs the convergence rate of gradient descent and plays a crucial role in stability analysis. The following Proposition illustrates that function $f$ is smooth.

\begin{proposition}[Smoothness of the Surrogate Objective]\label{prop:grad_lipschitz}
Let $S=\{z_i=(x_i,y_i)\}_{i=1}^{n}$ be a training sample, where $x_i\in\mathbb{R}^{d}$ and $y_i\in\mathbb{R}^{c}$. For a fixed reweighting matrix $\boldsymbol{D}\succeq 0$, define the regularized empirical objective of Eq.\eqref{obj}
\begin{equation}
    \begin{aligned}
        F_S(\boldsymbol{W}) = \sum_{i=1}^{n}\ell(\boldsymbol{W};z_i) +& \frac{\alpha}{2}{\Tr}(\boldsymbol{WRW}\!^\top)
    \\ & + 2\lambda{\Tr}(\boldsymbol{W}\!^\top \boldsymbol{DW}),
    \end{aligned}
\end{equation}
where $\ell(\boldsymbol{W};z_i) = \frac{s}{2}\|x_i^\top \boldsymbol{W}-y_i\|_2^2$, $s=\mu/(1+\mu)$. 
If the structural regularization is convex, then $F_S$ is convex and $\beta$-smooth \textit{w.r.t.} the Frobenius norm. Consequently, for any $\boldsymbol{W}_1,\boldsymbol{W}_2 \in \mathbb{R}^{d\times c}$:
\begin{equation}
    \|\nabla F_S(\boldsymbol{W}_1)-\nabla F_S(\boldsymbol{W}_2)\|_F \le \beta\|\boldsymbol{W}_1-\boldsymbol{W}_2\|_F,
\end{equation}
where
\begin{equation}
    \beta = {s}\|\mathbf{X}^\top \mathbf{X}\|_2 + \frac{\alpha}{2}\|\widetilde{\boldsymbol{R}}\|_2 + 4\lambda\|\boldsymbol{D}\|_2, \ \widetilde{\boldsymbol{R}}\!=\!(\boldsymbol{R}\!+\!\boldsymbol{R}\!^{\top}).
\end{equation}
\end{proposition}

\begin{remark}
    Based on the smoothness property established in Proposition \ref{prop:grad_lipschitz}, the gradient descent iterates exhibit controlled changes, which allows us to bound the model. This leads to the following uniform stability bound for full‑batch gradient descent under convexity and Lipschitz conditions. The detailed proof is provided in Supplementary Material X.
\end{remark}

With the smoothness inequality above, we can derive the following generalization bound for ProFS algorithm.

\begin{proposition}\label{the:genstab}
Let the loss function $\ell$ is convex, $s\sqrt{c}$-Lipschitz, and $\beta$-smooth, $0\le \ell(\boldsymbol{W}\!_{S};z)\le M$. Suppose that $\mathcal{W}$ be the class of linear predictors used in Eq.\eqref{obj}, where $\|\mathbf{x}\|_2\!\le\!1$ for all instances, and  $\boldsymbol{W}\!_{S}$ be the output obtained after $T$ iterations of the update rule in Eq.(\ref{equ:upda}) on a training sample $S$ of size $n$, where the learning rate at the t-th iteration $\eta_t$ satisfies $0<\eta_t\!\le\! 2/\beta$. Then, for any $\delta\! \in\! (0, 1)$, the following bound holds with probability at least $1\!-\!\delta$ over the random draw of $S$,
\begin{equation}
    \begin{aligned}
        R(\boldsymbol{W}_{S})\le \hat{R}&(\boldsymbol{W}_{S}) + \frac{2\mu^2 c}{(1+\mu)^2n} \sum_{t=1}^T \eta_t \\
         &+ \left(4\frac{\mu^2 c}{(1+\mu)^2}\sum_{t=1}^T \eta_t + M\right)\sqrt{\frac{\log\frac{1}{\delta}}{2n}}.
    \end{aligned}
\end{equation}
\end{proposition}

\begin{proof}[Proof Sketch]
    First, we prove that sample perturbations satisfy the McDiarmid's inequality (Lemma I.9). Then, combined with Lemma XI.1, the desired bound can be derived. The detailed proof is provided in Supplementary Material XI.
\end{proof}

\begin{remark}
    Proposition \ref{the:genstab} indicates that when the algorithm is sufficiently stable (\textit{i.e.}, when $\epsilon_{\text{stab}}$ is sufficiently small), its generalization error is primarily determined by the empirical risk, while the confidence term tends to zero as the sample size $n$ increases. This also indicates that whether an algorithm can generalize depends not only on the complexity of the hypothesis space, but also on the algorithm's sensitivity to perturbations in the training data. A stable algorithm, even one with an extremely large number of model parameters, still has theoretical guarantees for its generalization performance as long as it is stable.
\end{remark}

\section{Experiments}\label{sec:experi}
In this section, we systematically validate the effectiveness and generalization of the proposed method through empirical experiments. The experimental analysis focuses on two main objectives, first, the overall performance superiority of ProFS is demonstrated by comparing its performance gap with state-of-the-art conventional methods. Second, the significant advantage of the asymmetric topological structure is verified by comparing it with other asymmetric-related methods. Below, we first introduce the experimental settings, including datasets, evaluation metrics and implementation details. Subsequently, the in-depth mechanism analysis of the experimental results is conducted.

\vspace{-1.2mm}
\subsection{Datasets and General Settings}
A total of 12 benchmark multi-label datasets are used for experimental studies. These datasets can be found on the MULAN\,\footnote{https://mulan.sourceforge.net/datasets.html}, KDIS\,\footnote{https://www.uco.es/kdis/mllresources/} and Yahoo\,\footnote{http://www.lamda.nju.edu.cn/files/MDDM-expdata.rar} repositories, respectively. The details of these datasets  are listed in Table \ref{tab:datasets}.

\begin{table}[]
	\centering
    \renewcommand\arraystretch{0.9}
	\caption{{Description of Datasets}}
	\label{tab:datasets}
	\footnotesize
    \setlength{\tabcolsep}{6pt}
	\begin{tabular}{cccccc}
		\toprule
		No.&Dataset & Samples & Features & Labels & Type \\ \toprule
		1&Cal500  & 502  & 68		& 174  & Music \\
		2&Emotions& 593  & 72     & 6  & Music \\
		3&Birds	& 645  &  260   & 20  & Audio  \\ 
		4&Image	& 2000   &294   & 5   & Image \\
		5&Scene	& 2407 &  294   &   6 & Image  \\
		6&Yeast    	&  2417      & 103    & 14   & Biology\\
		7&Business& 5000    &438    &30     &  Text\\
		8&Entertainment &5000  &640     &21&    Text	 \\
		9&Education&5000    &550    &33     & Text\\
		10&Computers&5000   &681     &33    &  Text\\
		11&Health&  5000    & 612    &32     & Text\\
		12&Social   &5000   &1047    &39    &   Text\\ 
		\bottomrule
	\end{tabular}
    \vspace{-3mm}
\end{table}

For all benchmark datasets, we performed five-fold cross-validation, ran the experiment ten times independently, and recorded the average results as well as standard deviations. Five widely-used multi-label classification evaluation metrics \cite{zhang2013review}, including \textit{Hamming loss (HL), Ranking loss (RL), One-error (OE), Coverage (CV)} and \textit{Average precision (AP)} are employed for performance evaluation. Among these, the first four metrics yield better performance with smaller values ($\downarrow$), while the last metric performs better with larger values ($\uparrow $). It should be noted that some comparison algorithms have not made their code publicly available, this study manually reproduced them based on their literature descriptions. We hereby note that experimental results may exhibit slight deviations due to environmental differences. 

\subsection{Comparison with Conventional Methods}
To comprehensively validate the effectiveness of our proposed method, we selected five state-of-the-art MLFS approaches as well as a widely used benchmark MLL baseline approach for comparison. These methods encompass mainstream feature selection paradigms, such as specific feature learning, manifold regularization, global redundancy optimization, and heuristic search. Their brief introductions are as follows:
\begin{itemize}
	\item JFSC \cite{huang2017joint} considers pairwise label correlations to learn shared features and label-specific features, and constructs multi-label classifier simultaneously on the learned low-dimensional data representations.
	\item MDFS \cite{zhang2019manifold} exploits global and local label correlations to learn sparse representation of feature via projecting feature information into a low-dimensional space and a manifold regularizer.
	\item GLFS \cite{zhang2023group} simultaneously considers the features special to the labels in the same group and specific features owned by each label to execute feature selection. 
	\item GRROOR \cite{xu2024embedded} employs a feature-weighted regression model that takes into account global feature redundancy and global label correlation information.
	\item MLACO \cite{paniri2020mlaco} is a novel method that introduces the idea of ant colony algorithm into multi-label feature selection for the first time. 
	\item ML-KNN \cite{zhang2007ml} is a lazy learning method, which is derived from the $k$-nearest neighbor algorithm. Due to its simplicity, it is currently one of the most commonly used baseline method in MLL.
\end{itemize}

\begin{table*}[!th]
	\centering
	\renewcommand\arraystretch{0.84}
    \caption{{Prediction performance of each conventional comparing algorithm (mean$\pm$std) in terms of \textit{Hamming Loss} and \textit{Ranking Loss}, where the best results are highlighted in \textbf{boldface}.}}
    \footnotesize
    \setlength{\tabcolsep}{9pt}
	\label{tab:hl}
    \begin{tabular}{clccccccc}
			\toprule
			\multirow{2}{*}{Id}&\multirow{2}{*}{Datasets} & \multicolumn{7}{c}{\textit{Hamming Loss} ($\downarrow$)} \\ \cline{3-9}
			& & ML-KNN 		 & JFSC 			& MDFS 				& GLFS 				& GRROOR 		& MLACO 	& ProFS \\ \toprule
			1  & Emotions      & 0.212$\pm$0.018 & 0.212$\pm$0.015  & 0.221$\pm$0.014 	& 0.226$\pm$0.022 	&0.224$\pm$0.025 &0.244$\pm$0.006	& \textbf{0.212$\pm$0.012} \\
			2  & Cal500        & 0.139$\pm$0.007 & 0.138$\pm$0.003 	& 0.141$\pm$0.001 	& 0.142$\pm$0.001	&0.140$\pm$0.001 &0.143$\pm$0.001  	& \textbf{0.137$\pm$0.003} \\
			3  & Birds         & 0.057$\pm$0.003 & 0.055$\pm$0.006 	& 0.054$\pm$0.001 	& 0.053$\pm$0.002 	&0.056$\pm$0.003 &0.067$\pm$0.002 	& \textbf{0.053$\pm$0.007} \\
			4  & Image         & \textbf{0.179$\pm$0.008} & 0.186$\pm$0.012 	& 0.185$\pm$0.006 	& 0.207$\pm$0.012 	&0.200$\pm$0.023  &0.196$\pm$0.010  & 0.183$\pm$0.013 \\
			5  & Scene         & 0.114$\pm$0.017 & 0.111$\pm$0.004 	& 0.134$\pm$0.008 	& 0.125$\pm$0.017 	&0.128$\pm$0.030  &0.122$\pm$0.009  & \textbf{0.111$\pm$0.003} \\
			6  & Yeast         & 0.200$\pm$0.009 & 0.202$\pm$0.013 	& 0.206$\pm$0.006 	& 0.207$\pm$0.007 	&0.231$\pm$0.003  &0.208$\pm$0.004  & \textbf{0.200$\pm$0.003} \\
			7  & Business      & 0.027$\pm$0.001 & 0.025$\pm$0.001 	& 0.027$\pm$0.001 	& 0.028$\pm$0.001 	&0.028$\pm$0.000  &0.029$\pm$0.001  & \textbf{0.025$\pm$0.001} \\
			8  & Entertainment & 0.062$\pm$0.001 & 0.054$\pm$0.002 	& 0.060$\pm$0.002 	& 0.056$\pm$0.003 	&0.060$\pm$0.004  &0.062$\pm$0.001  & \textbf{0.054$\pm$0.001} \\
			9  & Education     & 0.039$\pm$0.001 & 0.039$\pm$0.001 	& 0.039$\pm$0.001 	& 0.404$\pm$0.001 	&0.042$\pm$0.001  &0.042$\pm$0.001  & \textbf{0.039$\pm$0.001} \\
			10 & Computers     & 0.040$\pm$0.001 & 0.037$\pm$0.001 	& 0.035$\pm$0.002 	& 0.037$\pm$0.001 	&0.037$\pm$0.004  &0.041$\pm$0.001  & \textbf{0.034$\pm$0.004} \\
			11 & Health        & 0.048$\pm$0.001 & 0.038$\pm$0.001 	& 0.041$\pm$0.003 	& 0.039$\pm$0.002 	&0.043$\pm$0.005  &0.044$\pm$0.002  & \textbf{0.037$\pm$0.001} \\
			12 & Social        & 0.023$\pm$0.001 & 0.023$\pm$0.001 	& 0.026$\pm$0.001 	& 0.024$\pm$0.002 	&0.025$\pm$0.002  &0.026$\pm$0.001  & \textbf{0.023$\pm$0.001} \\ 
            \hline
			& Average value & 0.095			 & 0.093			& 0.097				& 0.129				&0.101			  & 0.102			& \textbf{0.092}\\
	\end{tabular}
    \begin{tabular}{clccccccc}
			\toprule
			\multirow{2}{*}{Id}&\multirow{2}{*}{Datasets} & \multicolumn{7}{c}{\textit{Ranking Loss} ($\downarrow$)} \\ \cline{3-9}
			& & ML-KNN 		 & JFSC 			& MDFS 				& GLFS 				& GRROOR 		& MLACO 	& ProFS \\ \toprule
			1  & Emotions      & 0.171$\pm$0.027 & 0.201$\pm$0.013 	& 0.178$\pm$0.012 	& 0.183$\pm$0.024 	&0.183$\pm$0.054 &0.222$\pm$0.007 & \textbf{0.154$\pm$0.026} \\
			2  & Cal500        & 0.185$\pm$0.005 & 0.190$\pm$0.006 	& 0.190$\pm$0.001 	& 0.192$\pm$0.001 	&0.187$\pm$0.002 &0.190$\pm$0.001  & \textbf{0.183$\pm$0.007} \\
			3  & Birds         & 0.117$\pm$0.013 & 0.169$\pm$0.025 	& 0.125$\pm$0.006 	& 0.134$\pm$0.009 	&0.143$\pm$0.016 &0.162$\pm$0.005  & \textbf{0.114$\pm$0.021} \\
			4  & Image         & 0.188$\pm$0.015 & 0.188$\pm$0.018 	& 0.194$\pm$0.016 	& 0.238$\pm$0.033 	&0.238$\pm$0.071  &0.216$\pm$0.018  & \textbf{0.172$\pm$0.018} \\
			5  & Scene         & 0.114$\pm$0.013 & 0.105$\pm$0.008 	& 0.159$\pm$0.013 	& 0.139$\pm$0.046 	&0.179$\pm$0.118  &0.113$\pm$0.020  & \textbf{0.105$\pm$0.007} \\
			6  & Yeast         & 0.177$\pm$0.010 & 0.175$\pm$0.016 	& 0.181$\pm$0.007 	& 0.183$\pm$0.008 	&0.208$\pm$0.004  &0.180$\pm$0.004  & \textbf{0.175$\pm$0.005} \\
			7  & Business      & 0.040$\pm$0.005 & 0.054$\pm$0.003 	& 0.039$\pm$0.002 	& 0.044$\pm$0.002 	&0.044$\pm$0.002  &0.044$\pm$0.002  & \textbf{0.037$\pm$0.002} \\
			8 & Entertainment & 0.126$\pm$0.004 & 0.105$\pm$0.004 	& 0.119$\pm$0.004 	& 0.120$\pm$0.011 	&0.127$\pm$0.018  &0.130$\pm$0.007  & \textbf{0.098$\pm$0.004} \\
			9 & Education     & 0.080$\pm$0.003 & 0.087$\pm$0.007 	& 0.079$\pm$0.005 	& 0.086$\pm$0.005 	&0.093$\pm$0.006  &0.097$\pm$0.003  & \textbf{0.076$\pm$0.002} \\
			10 & Computers     & 0.087$\pm$0.004 & 0.100$\pm$0.005 	& 0.080$\pm$0.007 	& 0.090$\pm$0.005 	&0.086$\pm$0.009  &0.094$\pm$0.004  & \textbf{0.073$\pm$0.003} \\
			11 & Health        & 0.063$\pm$0.004 & 0.062$\pm$0.006 	& 0.067$\pm$0.008 	& 0.060$\pm$0.007 	&0.062$\pm$0.010  &0.067$\pm$0.005  & \textbf{0.045$\pm$0.003} \\
			12 & Social        & 0.062$\pm$0.004 & 0.075$\pm$0.005 	& 0.063$\pm$0.005 	& 0.069$\pm$0.004 	&0.067$\pm$0.006  &0.072$\pm$0.003  & \textbf{0.062$\pm$0.003} \\ 
            \hline
			& Average value 	& 0.118			& 0.126				& 0.123				& 0.128				&0.135			& 0.132				& \textbf{0.108}\\
	\end{tabular}
    \begin{tabular}{clccccccc}
			\toprule
			\multirow{2}{*}{Id}&\multirow{2}{*}{Datasets} & \multicolumn{7}{c}{\textit{Average Precision} ($\uparrow$)} \\ \cline{3-9}
			& & ML-KNN 		 & JFSC 			& MDFS 				& GLFS 				& GRROOR 		& MLACO 	& ProFS \\ \toprule
			1  & Emotions      & 0.787$\pm$0.023 & 0.772$\pm$0.012  & 0.780$\pm$0.011 	& 0.780$\pm$0.025 	& 0.779$\pm$0.045 &0.774$\pm$0.009 & \textbf{0.810$\pm$0.030} \\
			2  & Cal500        & 0.490$\pm$0.101 & 0.486$\pm$0.002 	& 0.481$\pm$0.003 	& 0.476$\pm$0.003 	& 0.486$\pm$0.004 &0.200$\pm$0.004 & \textbf{0.513$\pm$0.014} \\
			3  & Birds         & 0.677$\pm$0.028 & 0.646$\pm$0.033 	& 0.696$\pm$0.011 	& 0.688$\pm$0.015 	& 0.668$\pm$0.028 &0.667$\pm$0.006  & \textbf{0.715$\pm$0.018} \\
			4  & Image         & 0.773$\pm$0.006 & 0.773$\pm$0.021 	& 0.767$\pm$0.017 	& 0.723$\pm$0.035 	& 0.729$\pm$0.067 &0.756$\pm$0.021  & \textbf{0.791$\pm$0.016} \\
			5  & Scene         & 0.811$\pm$0.026 & 0.829$\pm$0.009 	& 0.760$\pm$0.020 	& 0.780$\pm$0.058 	& 0.743$\pm$0.126 &0.803$\pm$0.025  & \textbf{0.830$\pm$0.007} \\
			6  & Yeast         & 0.753$\pm$0.014 & 0.755$\pm$0.012 	& 0.746$\pm$0.009 	& 0.744$\pm$0.010 	& 0.710$\pm$0.005 &0.750$\pm$0.005  & \textbf{0.755$\pm$0.006} \\
			7  & Business      & 0.877$\pm$0.008 & 0.873$\pm$0.008 	& 0.878$\pm$0.003 	& 0.871$\pm$0.003 	& 0.871$\pm$0.005 &0.877$\pm$0.005  & \textbf{0.881$\pm$0.003} \\
			8 & Entertainment & 0.550$\pm$0.009 & 0.667$\pm$0.013 	& 0.577$\pm$0.026 	& 0.585$\pm$0.009 	& 0.571$\pm$0.059 &0.565$\pm$0.027  & \textbf{0.684$\pm$0.005} \\
			9 & Education     & 0.614$\pm$0.005 & 0.609$\pm$0.007 	& 0.609$\pm$0.014 	& 0.577$\pm$0.023 	& 0.548$\pm$0.024 &0.545$\pm$0.012  & \textbf{0.630$\pm$0.007} \\
			10 & Computers     & 0.646$\pm$0.009 & 0.663$\pm$0.008 	& 0.669$\pm$0.016 	& 0.639$\pm$0.017 	& 0.652$\pm$0.033 &0.634$\pm$0.008  & \textbf{0.695$\pm$0.010} \\
			11 & Health        & 0.693$\pm$0.013 & 0.742$\pm$0.007 	& 0.664$\pm$0.036 	& 0.718$\pm$0.024 	& 0.692$\pm$0.041 &0.680$\pm$0.022  & \textbf{0.781$\pm$0.006} \\
			12 & Social        & 0.729$\pm$0.007 & 0.746$\pm$0.011 	& 0.747$\pm$0.008 	& 0.712$\pm$0.013 	& 0.702$\pm$0.039 &0.697$\pm$0.016  & \textbf{0.752$\pm$0.016} \\ 
            \hline
            & Average value 	& 0.700			& 0.713				& 0.698				& 0.691				&0.679			& 0.662				& \textbf{0.736}\\
            \bottomrule
	\end{tabular}
\end{table*}

The hyper-parameters for all comparing algorithms strictly followed the recommended settings outlined in their respective literature. For ProFS, parameters $\mu$, $\lambda$ and $\alpha$ are selected in $\{10^{-6},10^{-5},\cdots,10^{2}\}$, respectively. The detailed information are shown in Section \ref{ssec:parana}.

Table \ref{tab:hl} and Table XII.1 (see in Supplementary Material XII) report the experimental results of seven algorithms across five metrics.  
Based on these reported results, we can draw the following observations. 
\begin{enumerate}
    \item ProFS achieves the best average results across five evaluation metrics, and it exhibits significant advantages in three ranking metrics, namely \textit{RL}, \textit{CV} and \textit{AP}, which fully demonstrates the strong generalization capability of asymmetric correlation modeling when handing the ranking of complex multi-label data. 
    \item For the label-specific approaches JFSC and GLFS, they focus on learning shared and label-specific features, and also utilize (global or local) label correlations. The results show that, apart from being on par with the \textit{HL} metric, ProFS outperforms them by a clear margin on the ranking metrics \textit{RL}, \textit{OE}, \textit{CV} and \textit{AP}. This indicates that if feature characterization fails to capture the directionality inherent in labels, it remains difficult to coincide with the underlying label ranking logic.
    \item For the manifold regularization approach MDFS, although this model also adopts the strongly sparse-constrained $\ell_{2,1}$-matrix norm and consider both global and local label correlations, ProFS still shows superior performance across all metrics. This further highlights the advancement of ProFS in constructing asymmetric matrix fusion and label fitting decoupling strategies.
    \item For the de-redundancy approaches GRROR and MLACO, while both possess excellent feature de-duplication capabilities, they still underperform ProFS in overall performance. This reveals that conducting feature de-redundancy under the symmetry assumption of label relationships is far from sufficient. 
    \item Regarding the classical baseline approach ML-KNN, while it demonstrates certain local robustness in specific scenarios (such as the \textit{OE} metric on \textit{Image}), ProFS still achieves significant overall superiority.
\end{enumerate}

Additional, Fig. \ref{fig:five-asy_lei} presents the overall rankings of the seven algorithms across five evaluation metrics, where the more peripheral algorithm indicates higher rankings. The geometric envelope results show that the proposed algorithm achieved optimal performance in 95\% of the 60 cases and attained optimal or suboptimal performance in 98.33\% of the cases, demonstrating greater stability and effectiveness.

Furthermore, we use \textit{Friedman test} as the statistical test to analyze the statistical signification of the differences in the performance of the seven approaches. Specially, \textit{Friedman test} follows the $F$-distribution with $(k-1)$ and $(k-1)(N-1)$ degree of freedom, which can be defined as:
\begin{gather*}
    F_F\! =\! \frac{(N-1)\chi _F^2}{N(k-1)\!-\!\chi _F^2}, \\[-5pt]
    \text{and}\ \chi_{F}^{2} \!=\! \frac{12N}{k(k\!+\!1)} ({\sum_{j=1}^{k} R_{j}^{2}\!-\!\frac{k(k\!+\!1)^{2}}{4}}\, ),
\end{gather*}
where $k$ and $N$ indicate the number of compared algorithms and datasets, respectively. Let $r_i^j$ denotes the rank of the $j$-th algorithm on the $i$-th data set. $R_j=\frac{1}{N}\sum_i r_i^j$ is the average rank of the $j$-th algorithm on all datasets.

\begin{table}[!t]
	\scriptsize
	\centering
	\caption{{Summary of Friedman statistics $F_F$ \textit{w.r.t.} conventional methods}}
	\label{tab:ff}
    \footnotesize
    \setlength{\tabcolsep}{3.8pt}
    \begin{tabular}{lcc}
        \toprule
        Evaluation metric & $F_F$ & \multicolumn{1}{l}{Critical value ($\alpha=0.05$)} \\ \toprule
        \textit{Hamming loss}      &  6.2124 & \multirow{5}{*}{2.239}                  \\
        \textit{Ranking loss}      &  9.3295 &                                    \\
        \textit{One-error}         &  21.1356 &                                    \\
        \textit{Coverage}          &  9.3948 &                                    \\
        \textit{Average precision} &  13.2730  &                                    \\ \bottomrule
	\end{tabular}
    \vspace{-4mm}
\end{table}

\begin{figure*}[!h] 
    \vspace{-4mm}
	\centering 
	\subfloat[Hamming loss]{\includegraphics[width=0.33\textwidth]{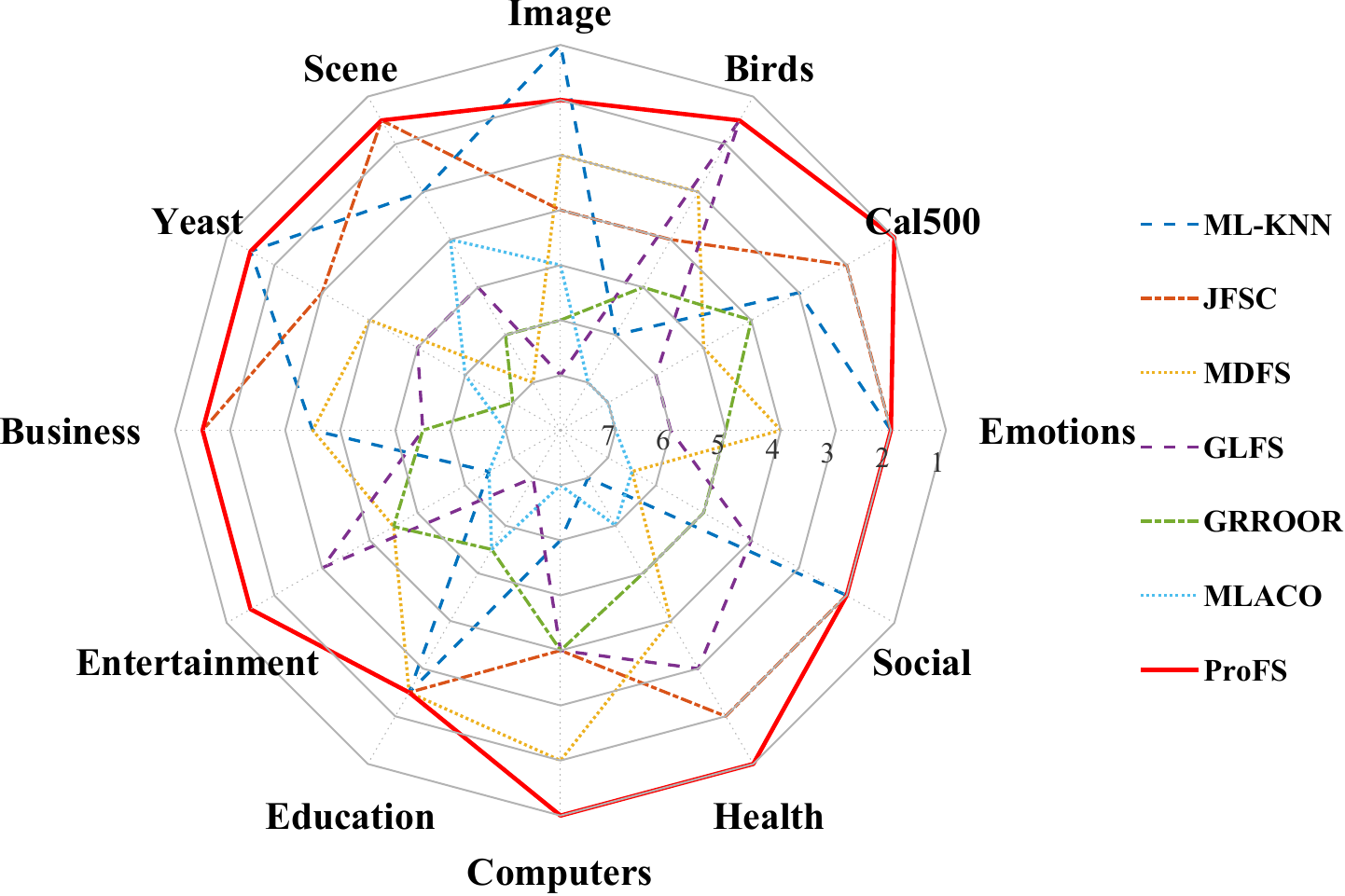}} 
	\hfill
	\subfloat[Ranking loss]{\includegraphics[width=0.33\textwidth]{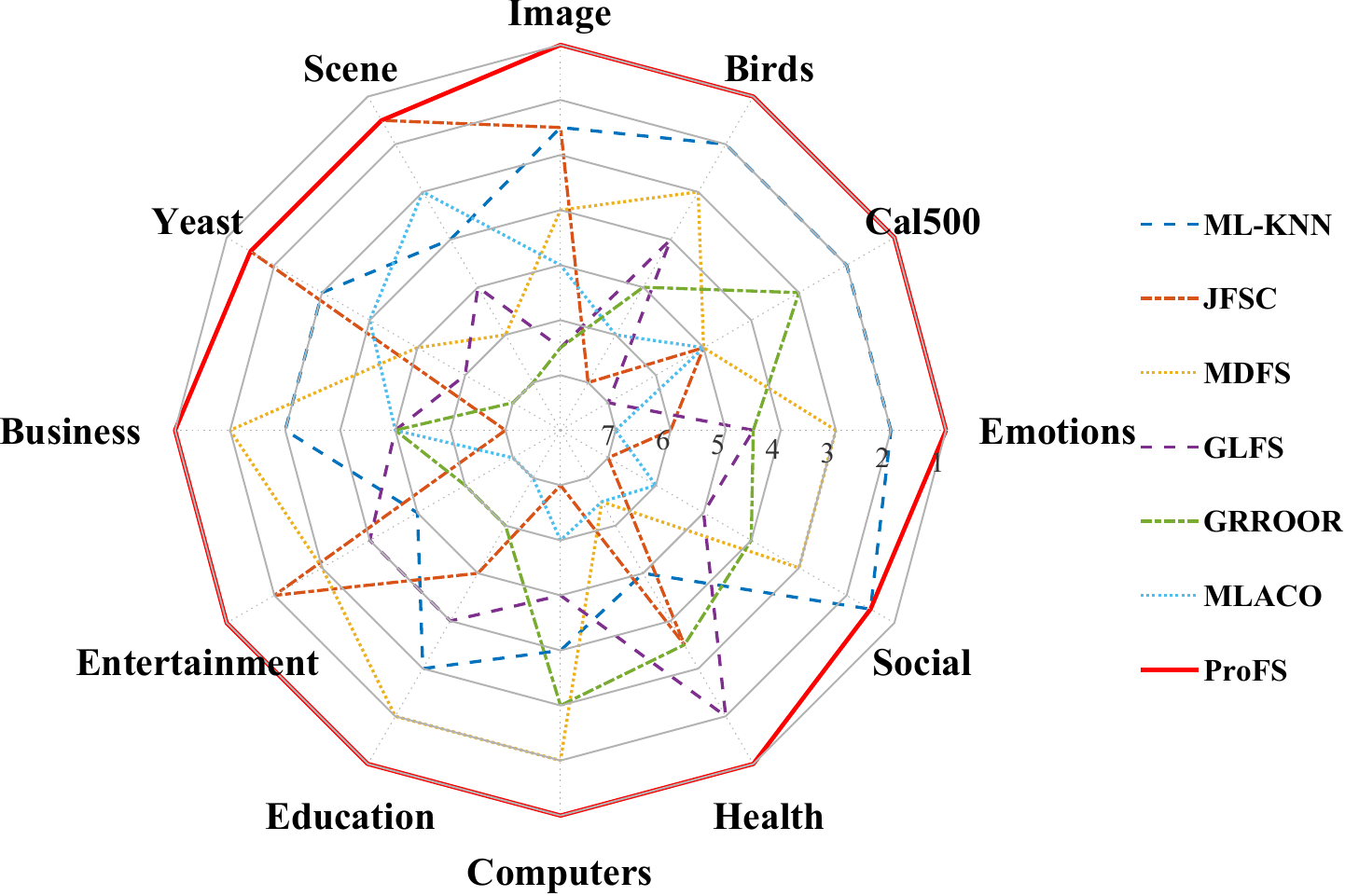}} 
    \hfill
	\subfloat[One error]{\includegraphics[width=0.33\textwidth]{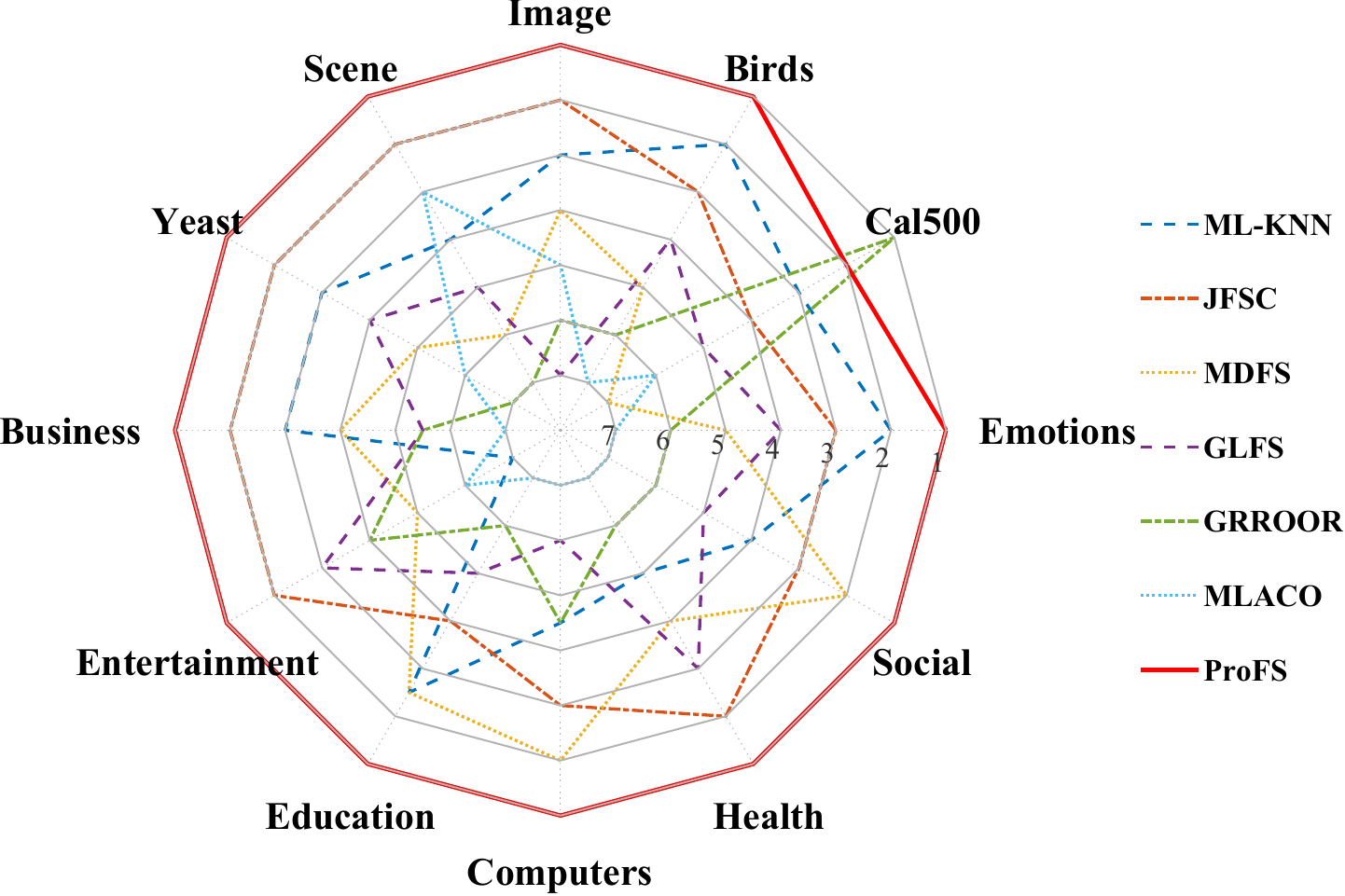}}
	
    \subfloat[Coverage]{\includegraphics[width=0.33\textwidth]{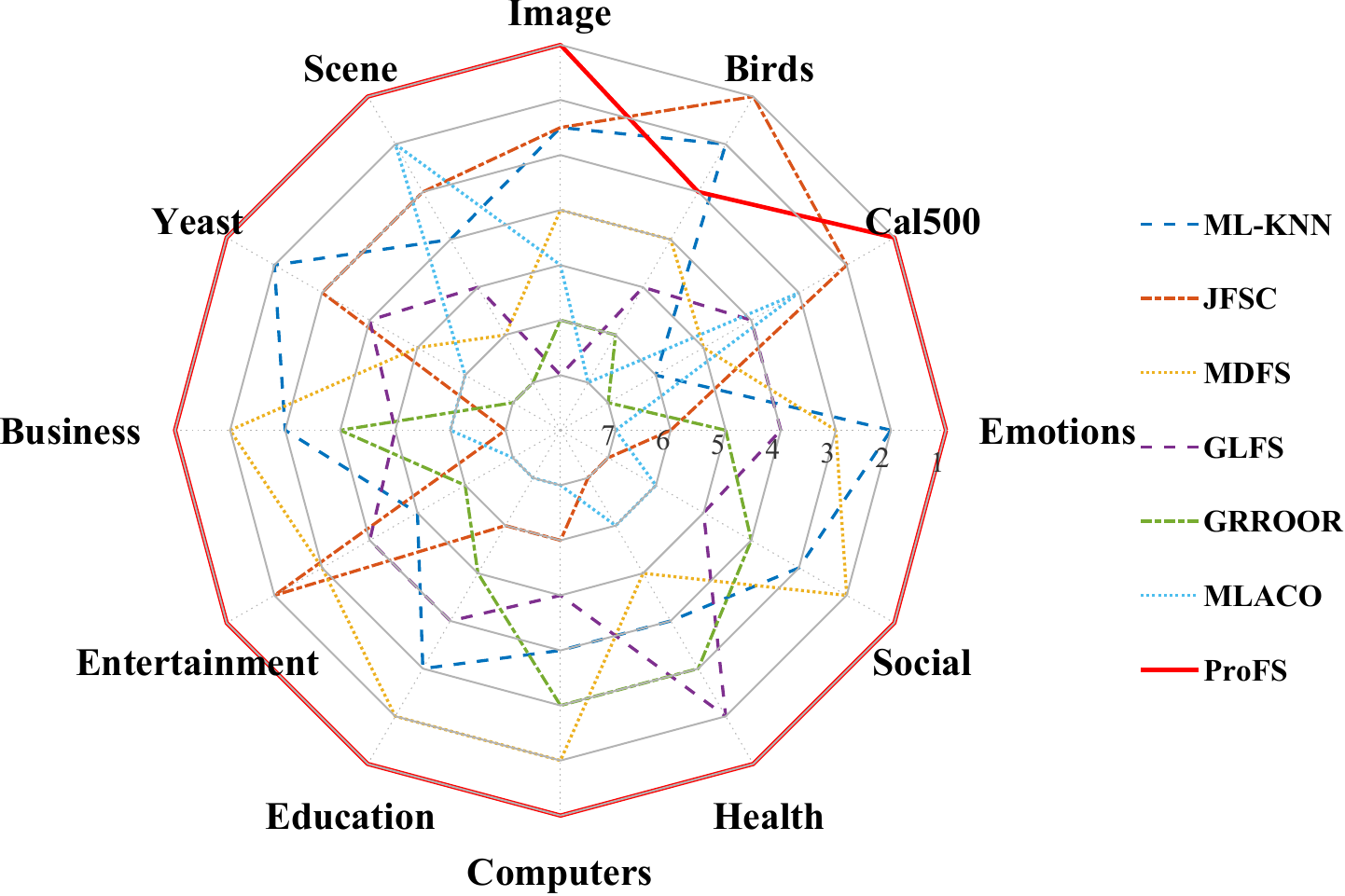}}
    \subfloat[Average precision]{\includegraphics[width=0.33\textwidth]{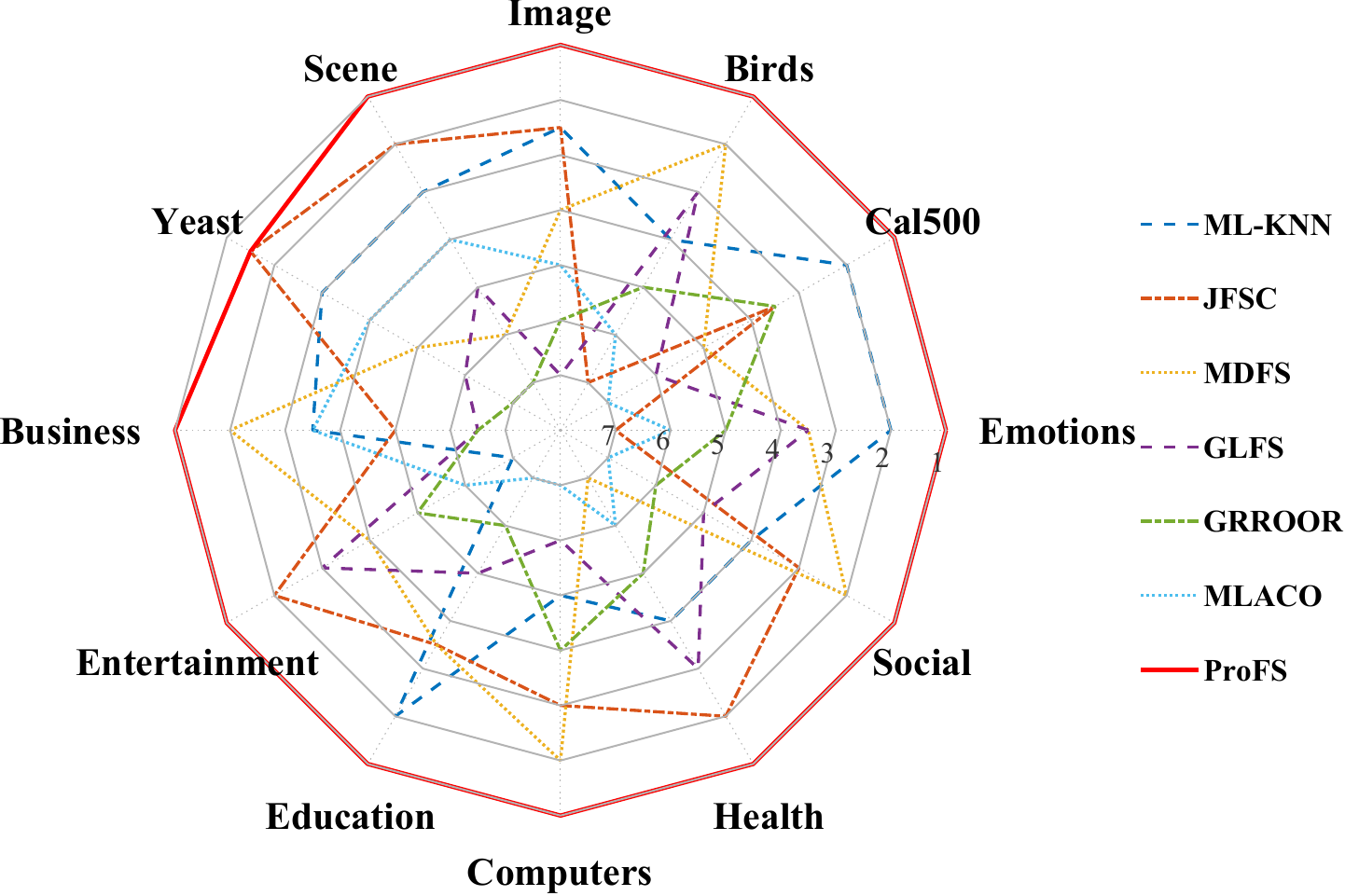}}
	\caption{Comparison of ProFS and conventional comparing algorithms on the five evaluation metrics.}  
	\label{fig:five-asy_lei}
    \vspace{-2mm}
\end{figure*}

\begin{figure*}[!th]
	\centering 
	\subfloat[Hamming loss]{\includegraphics[width=0.19\textwidth]{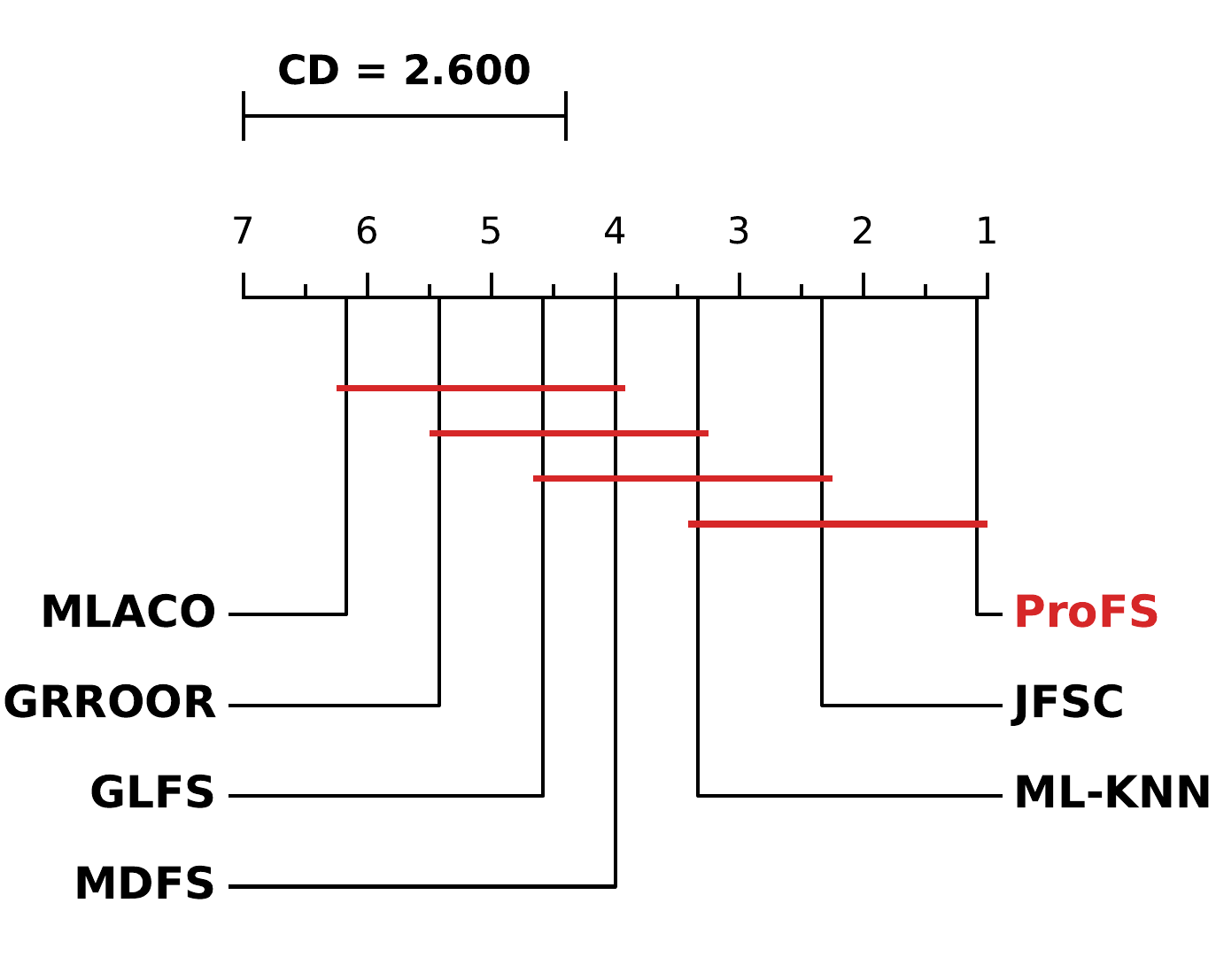}} 
	\label{fig:first}  
	\subfloat[Ranking loss]{\includegraphics[width=0.19\textwidth]{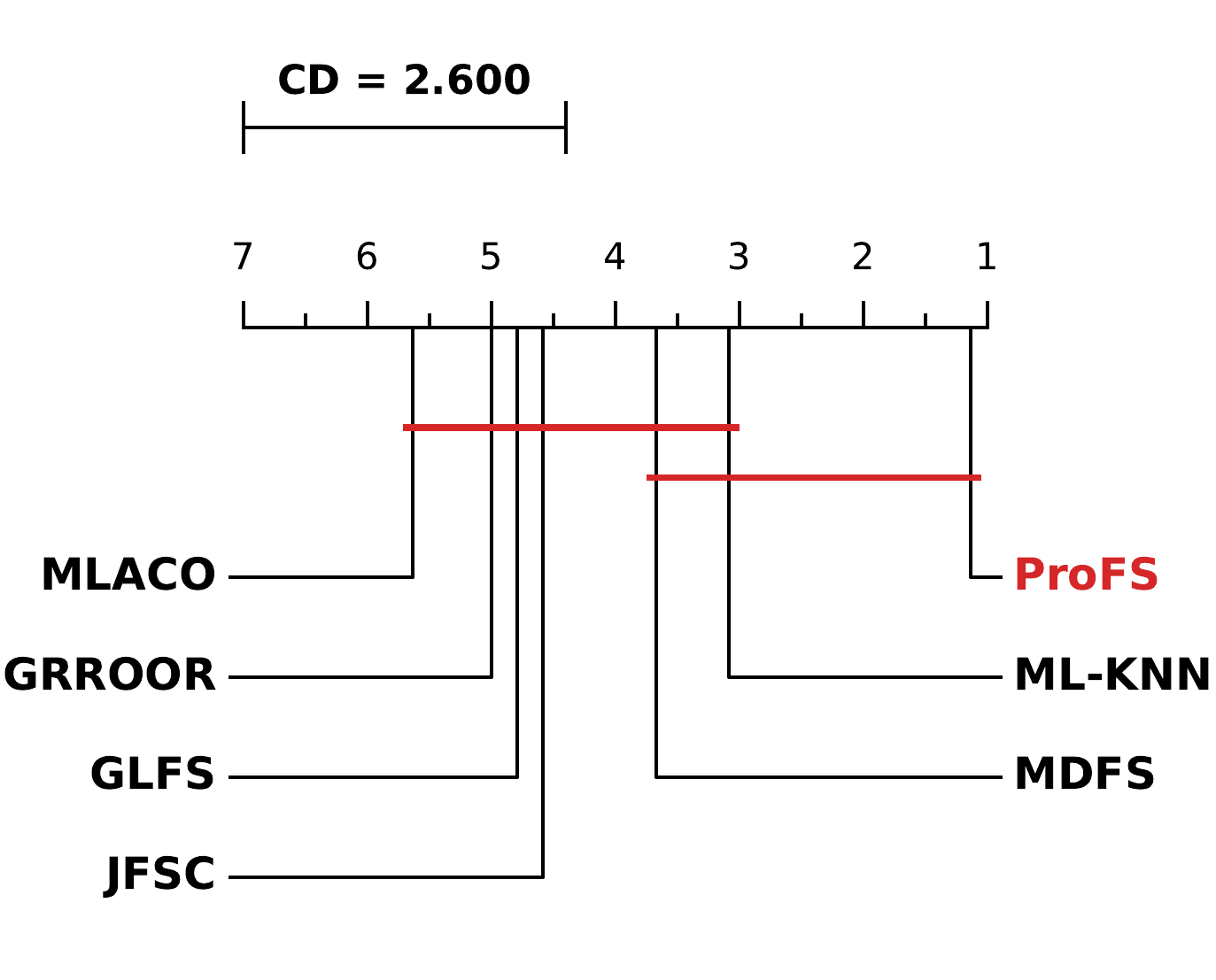}}
	\label{fig:second}  
	\subfloat[One error]{\includegraphics[width=0.19\textwidth]{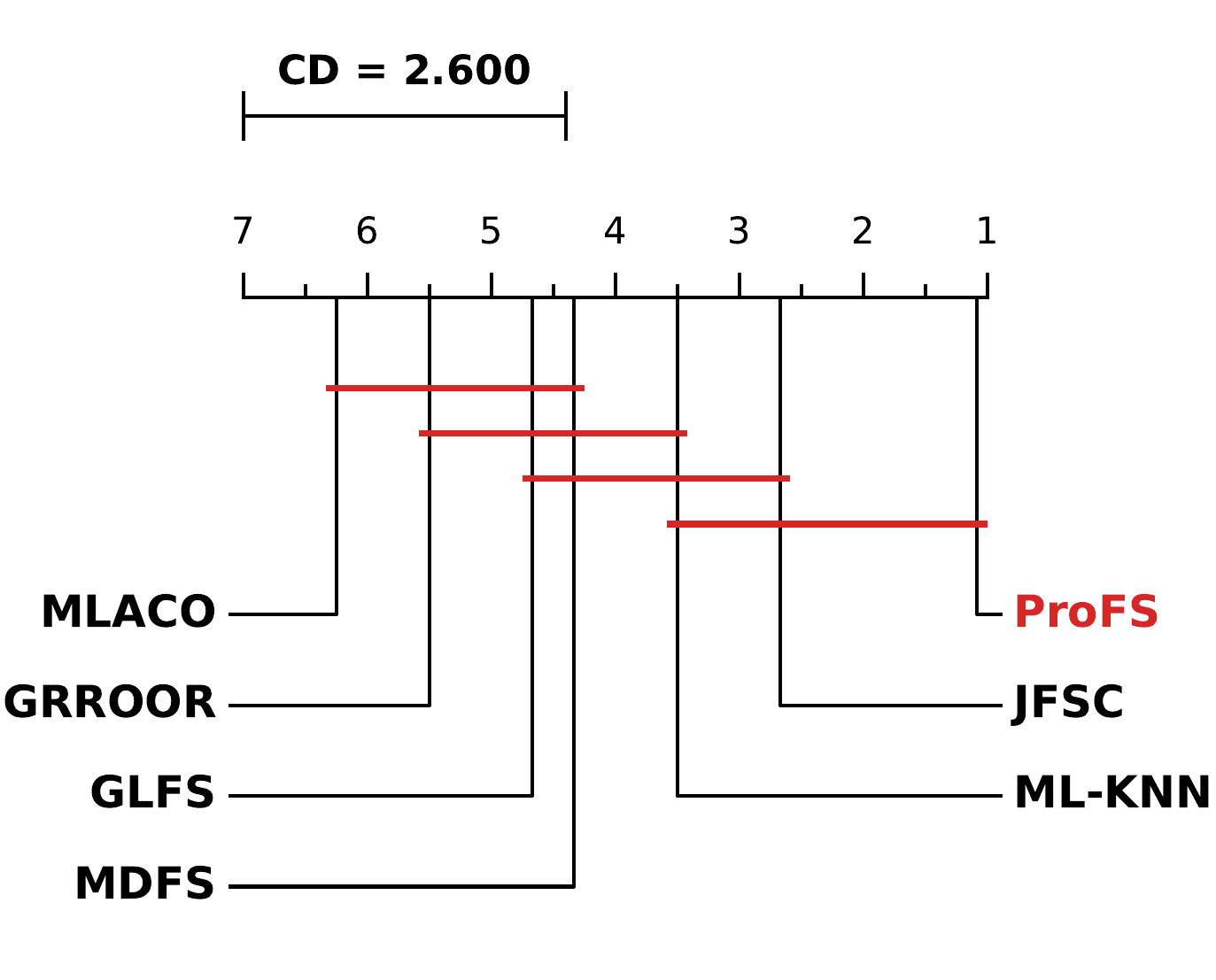} }
	\label{fig:third}  
	\subfloat[Coverage]{\includegraphics[width=0.19\textwidth]{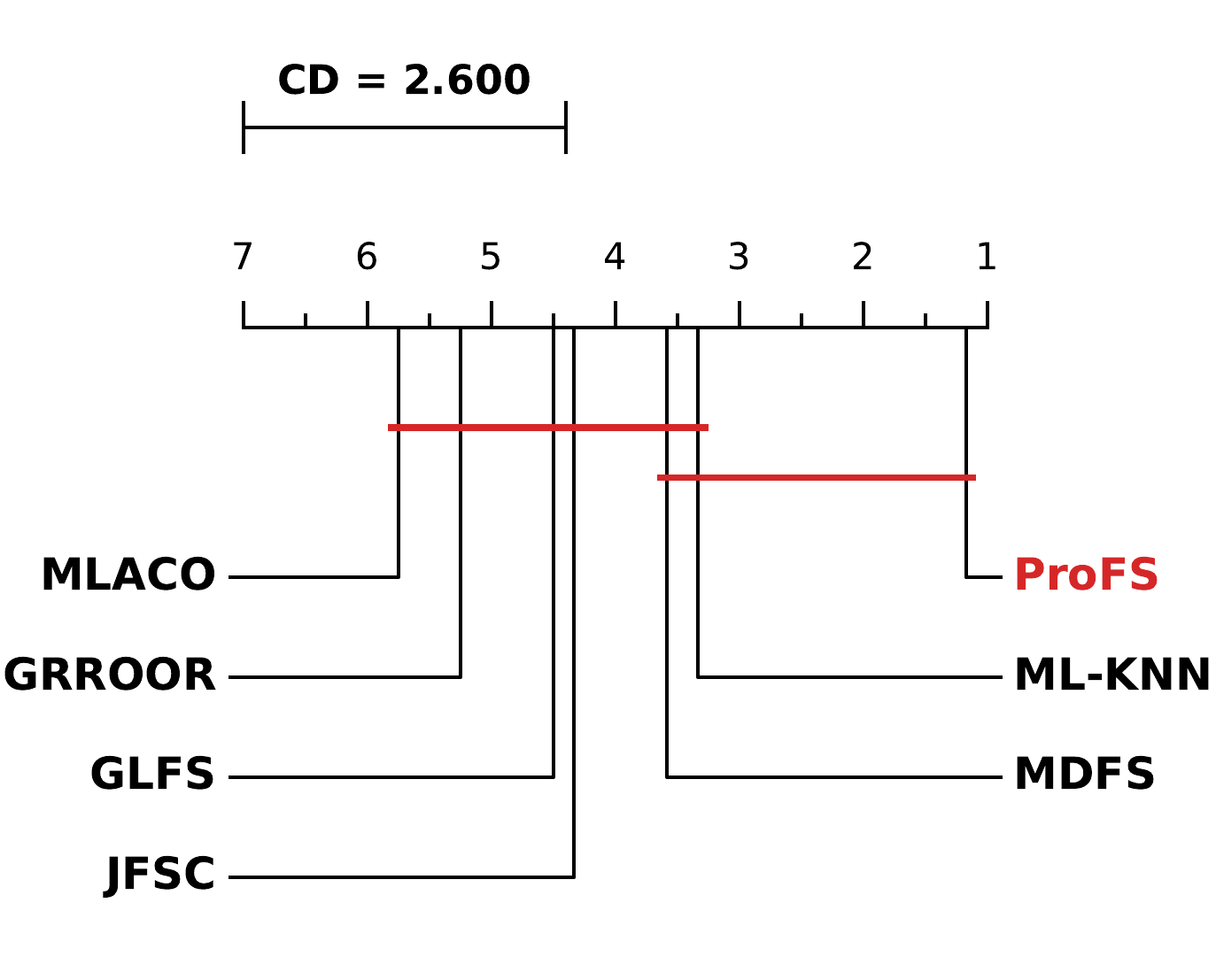} }
	\label{fig:fourth}  
	\subfloat[Average precision]{\includegraphics[width=0.19\textwidth]{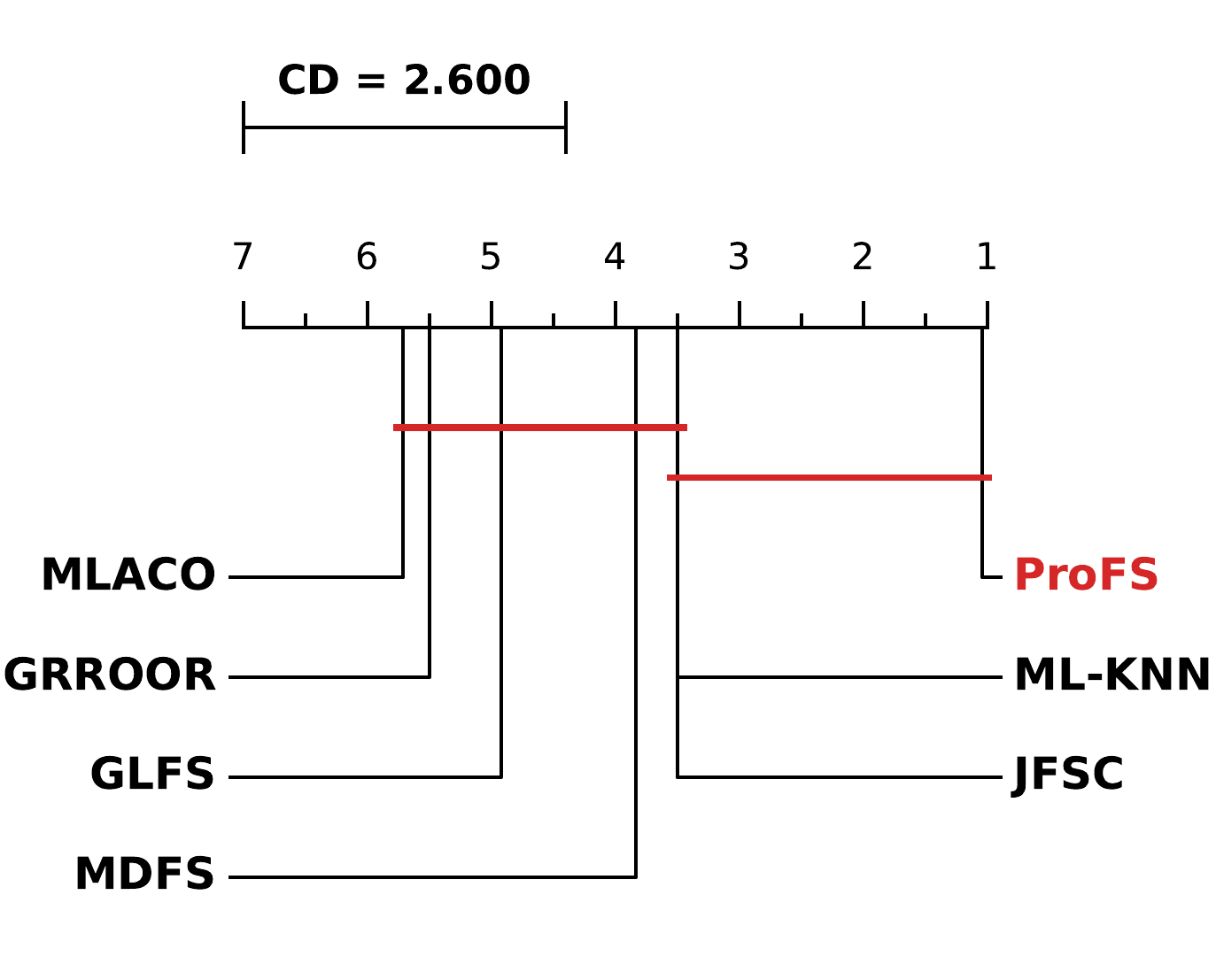} }
	\label{fig:fifth}  
	\caption{Comparison of ProFS against six conventional comparing algorithms with the \textit{Nemenyi test}.}  
    \vspace{-1mm}
	\label{fig:five-cd} 
\end{figure*}

Table \ref{tab:ff} reports the Friedman statistic $F_F$ and the corresponding critical values on each evaluation metric $(k\!=\!7,N\!=\!12)$. It can be observed that at significance level $\alpha=0.05$, the \textit{Friedman test} result rejects the null hypothesis, which states that there is no significant difference in the performance of all algorithms. Hence, to further explore the specific performance differences among the algorithms, particularly with ProML as the  control algorithm, it is necessary to conduct the \textit{Nemenyi post-hoc test}. In \textit{Nemenyi test}, the performance between two algorithms is significantly different if the distance of the average ranks exceeds the following critical difference (CD):
\begin{equation*}
	CD_{\alpha} = q_{\alpha}\sqrt{\frac{k(k+1)}{6N}}\, ,
\end{equation*}
where $q_{\alpha} = 2.949$ when $k = 7$, then $CD_{\alpha}=2.600$.

Fig. \ref{fig:five-cd} shows the CD diagrams for the Nemenyi post-hoc statistical test. From the figure, it can be observed that ProML achieved the best performance across all metrics, and there were statistically significant differences compared to the de-redundancy methods GRROOR and MLACO. These experimental results strongly validate the effectiveness of the proposed method.

\begin{table*}[!ht]
\centering
\renewcommand\arraystretch{0.85}
\caption{{Prediction performance of each asymmetric-related comparing algorithm (mean$\pm$std) in terms of \textit{Hamming Loss} and \textit{Ranking Loss}, where $\bullet$/$\circ$ indicates whether ProFS is superior/inferior to the other method.}}
\footnotesize
\setlength{\tabcolsep}{4.5pt}
\label{tab:hl_asy}
\begin{tabular}{clcccccccc}
    \toprule
    \multirow{2}{*}{Id}&\multirow{2}{*}{Datasets} & \multicolumn{8}{c}{\textit{Hamming Loss} ($\downarrow$)} \\ \cline{3-10}
    & & ML2G 		 & ML2G-S 			& MFSRD 				& ACML 				& MAHR 		& LSDF 	& ProML-S & ProML \\ \toprule
    1  & Emotions      & 0.208$\pm$0.013$\circ$ & 0.211$\pm$0.014$\circ$  & 0.257$\pm$0.031$\bullet$ 	& 0.201$\pm$0.010$\circ$ 	&0.216$\pm$0.017$\bullet$ &0.184$\pm$0.011$\circ$	&0.230$\pm$0.010$\bullet$  & 0.212$\pm$0.012 \\
    2  & Cal500        & 0.850$\pm$0.002$\bullet$ & 0.850$\pm$0.002$\bullet$ & 0.140$\pm$0.030$\bullet$ & 0.137$\pm$0.004$\circ$ & 0.967$\pm$0.003$\bullet$ & 0.138$\pm$0.003$\bullet$ & 0.147$\pm$0.003$\bullet$ & {0.137$\pm$0.003}\\
    3  & Birds         & 0.053$\pm$0.005$\bullet$ & 0.053$\pm$0.005$\bullet$ & 0.057$\pm$0.048$\bullet$ & 0.052$\pm$0.007$\circ$ & 0.058$\pm$0.007$\bullet$ & 0.047$\pm$0.006$\circ$ & 0.073$\pm$0.006$\bullet$ & {0.053$\pm$0.007} \\
    4  & Image         & 0.199$\pm$0.008$\bullet$ & 0.197$\pm$0.007$\bullet$ & 0.212$\pm$0.066$\bullet$ & 0.200$\pm$0.004$\bullet$ &0.202$\pm$0.007$\bullet$ & 0.160$\pm$0.009$\circ$ & 0.201$\pm$0.012$\bullet$ & 0.183$\pm$0.013 \\
    5  & Scene         & 0.120$\pm$0.005$\bullet$ & 0.119$\pm$0.003$\bullet$ & 0.141$\pm$0.025$\bullet$ & 0.110$\pm$0.002$\circ$ & 0.129$\pm$0.003$\bullet$ & 0.080$\pm$0.005$\circ$ & 0.115$\pm$0.003$\bullet$ & {0.111$\pm$0.003}\\
    6  & Yeast         & 0.205$\pm$0.006$\bullet$ & 0.204$\pm$0.006$\bullet$ & 0.225$\pm$0.073$\bullet$ & 0.201$\pm$0.007$\bullet$ & 0.217$\pm$0.004$\bullet$ & 0.198$\pm$0.006$\circ$ & 0.207$\pm$0.003$\bullet$ & {0.200$\pm$0.003} \\
    7  & Business      & 0.026$\pm$0.001$\bullet$ & 0.026$\pm$0.001$\bullet$ & 0.029$\pm$0.030$\bullet$ & 0.026$\pm$0.001$\bullet$ & 0.028$\pm$0.001$\bullet$ & 0.025$\pm$0.001$\circ$ & 0.026$\pm$0.001$\bullet$ & {0.025$\pm$0.001}\\
    8  & Entertainment & 0.054$\pm$0.002$\circ$ & 0.053$\pm$0.002$\circ$ & 0.062$\pm$0.036$\bullet$ & 0.051$\pm$0.002$\bullet$ & 0.061$\pm$0.002$\bullet$ & 0.052$\pm$0.001$\circ$ & 0.058$\pm$0.001$\bullet$ & {0.054$\pm$0.001}\\
    9  & Education     & 0.039$\pm$0.001$\bullet$ & 0.038$\pm$0.001$\circ$ & 0.044$\pm$0.020$\bullet$ & 0.037$\pm$0.001$\circ$ & 0.042$\pm$0.001$\bullet$ & 0.039$\pm$0.001$\bullet$ & 0.041$\pm$0.001$\bullet$ & {0.039$\pm$0.001}\\
    10 & Computers     & 0.035$\pm$0.001$\bullet$ & 0.036$\pm$0.001$\bullet$ & 0.040$\pm$0.025$\bullet$ & 0.034$\pm$0.001$\circ$ & 0.041$\pm$0.001$\bullet$ & 0.033$\pm$0.001$\circ$ & 0.036$\pm$0.004$\bullet$ & {0.034$\pm$0.004}\\
    11 & Health        & 0.033$\pm$0.001$\circ$ & 0.033$\pm$0.001$\circ$ & 0.049$\pm$0.026$\bullet$ & 0.033$\pm$0.001$\circ$ & 0.044$\pm$0.001$\bullet$ & 0.033$\pm$0.002$\circ$ & 0.038$\pm$0.001$\bullet$ & {0.037$\pm$0.001}\\
    12 & Social        & 0.020$\pm$0.001$\circ$ & 0.020$\pm$0.001$\circ$ & 0.025$\pm$0.018$\bullet$ & 0.020$\pm$0.001$\circ$ & 0.026$\pm$0.001$\bullet$ & 0.021$\pm$0.001$\circ$ & 0.023$\pm$0.001$\bullet$ & {0.023$\pm$0.001} \\ 
    \hline
    & Average value 	& 0.154$\bullet$			& 0.153$\bullet$				& 0.107$\bullet$				& 0.092				& 0.169$\bullet$			& 0.084$\circ$				& 0.100$\bullet$    & 0.092\\
    \bottomrule
\end{tabular}
\begin{tabular}{clcccccccc}
    \multirow{2}{*}{Id}&\multirow{2}{*}{Datasets} & \multicolumn{8}{c}{\textit{Ranking Loss} ($\downarrow$)} \\ \cline{3-10}
    & & ML2G 		 & ML2G-S 			& MFSRD 				& ACML 				& MAHR 		& LSDF 	& ProML-S & ProML \\ \toprule
    1  & Emotions      & 0.178$\pm$0.020$\bullet$ & 0.180$\pm$0.021$\bullet$ & 0.212$\pm$0.029$\bullet$ & 0.163$\pm$0.023$\bullet$ & 0.187$\pm$0.010$\bullet$ & 0.148$\pm$0.006$\circ$ & 0.189$\pm$0.025$\bullet$ & {0.154$\pm$0.026} \\
    2  & Cal500        & 0.234$\pm$0.007$\bullet$ & 0.232$\pm$0.007$\bullet$ & 0.192$\pm$0.011$\bullet$ & 0.180$\pm$0.003$\circ$ & 0.777$\pm$0.024$\bullet$ & 0.181$\pm$0.010$\circ$ & 0.299$\pm$0.007$\bullet$ & {0.188$\pm$0.007}\\
    3  & Birds         & 0.114$\pm$0.017$\bullet$ & 0.113$\pm$0.019$\circ$ & 0.137$\pm$0.078$\bullet$ & 0.092$\pm$0.023$\circ$ & 0.194$\pm$0.035$\bullet$ & 0.100$\pm$0.009$\circ$ & 0.129$\pm$0.020$\bullet$ & {0.114$\pm$0.021} \\
    4  & Image         & 0.197$\pm$0.018$\bullet$ & 0.193$\pm$0.017$\bullet$ & 0.257$\pm$0.042$\bullet$ & 0.183$\pm$0.008$\bullet$ & 0.200$\pm$0.013$\bullet$ & 0.147$\pm$0.014$\circ$ & 0.202$\pm$0.017$\bullet$ & {0.172$\pm$0.018} \\
    5  & Scene         & 0.111$\pm$0.008$\bullet$ & 0.111$\pm$0.007$\bullet$ & 0.154$\pm$0.015$\bullet$ & 0.090$\pm$0.005$\circ$ & 0.131$\pm$0.009$\bullet$ & 0.066$\pm$0.006$\circ$ & 0.108$\pm$0.006$\bullet$ & {0.105$\pm$0.007} \\
    6  & Yeast         & 0.199$\pm$0.013$\bullet$ & 0.202$\pm$0.011$\bullet$ & 0.207$\pm$0.090$\bullet$ & 0.171$\pm$0.009$\circ$ & 0.227$\pm$0.010$\bullet$ & 0.165$\pm$0.011$\circ$ & 0.184$\pm$0.004$\bullet$  & {0.175$\pm$0.005}\\
    7  & Business      & 0.060$\pm$0.007$\bullet$ & 0.061$\pm$0.009$\bullet$ & 0.051$\pm$0.039$\bullet$ & 0.045$\pm$0.003$\bullet$ & 0.242$\pm$0.011$\bullet$ & 0.031$\pm$0.003$\circ$ & 0.059$\pm$0.002$\bullet$ & {0.037$\pm$0.002} \\
    8 & Entertainment & 0.112$\pm$0.005$\bullet$ & 0.112$\pm$0.007$\bullet$ & 0.144$\pm$0.042$\bullet$ & 0.118$\pm$0.005$\bullet$ & 0.594$\pm$0.013$\bullet$ & 0.094$\pm$0.004$\circ$ & 0.124$\pm$0.004$\bullet$ & {0.098$\pm$0.004} \\
    9 & Education     & 0.105$\pm$0.012$\bullet$ & 0.101$\pm$0.012$\bullet$ & 0.103$\pm$0.012$\bullet$ & 0.110$\pm$0.005$\bullet$ & 0.765$\pm$0.014$\bullet$ & 0.075$\pm$0.004$\circ$ & 0.119$\pm$0.002$\bullet$ & {0.076$\pm$0.002} \\
    10 & Computers     & 0.102$\pm$0.009$\bullet$ & 0.107$\pm$0.008$\bullet$ & 0.116$\pm$0.082$\bullet$ & 0.099$\pm$0.006$\bullet$ & 0.481$\pm$0.013$\bullet$ & 0.070$\pm$0.002$\circ$ & 0.109$\pm$0.003$\bullet$ & {0.073$\pm$0.003} \\
    11 & Health        & 0.091$\pm$0.007$\bullet$ & 0.092$\pm$0.007$\bullet$ & 0.078$\pm$0.028$\bullet$ & 0.067$\pm$0.004$\bullet$ & 0.352$\pm$0.014$\bullet$ & 0.042$\pm$0.001$\circ$ & 0.076$\pm$0.003$\bullet$ & {0.045$\pm$0.003} \\
    12 & Social        & 0.086$\pm$0.007$\bullet$ & 0.084$\pm$0.008$\bullet$ & 0.090$\pm$0.090$\bullet$ & 0.073$\pm$0.005$\bullet$ & 0.326$\pm$0.020$\bullet$ & 0.055$\pm$0.002$\circ$ & 0.083$\pm$0.003$\bullet$ & {0.062$\pm$0.003} \\ 
    \hline
    & Average value 	& 0.132$\bullet$			& 0.132$\bullet$				& 0.145$\bullet$				& 0.116$\bullet$				& 0.373$\bullet$			& 0.098$\circ$			& 0.140$\bullet$     & 0.108\\
    \bottomrule
\end{tabular}
\begin{tabular}{clcccccccc}
    \multirow{2}{*}{Id}&\multirow{2}{*}{Datasets} & \multicolumn{8}{c}{\textit{Average Precision} ($\uparrow$)} \\ \cline{3-10}
    & & ML2G 		 & ML2G-S 			& MFSRD 				& ACML 				& MAHR 		& LSDF 	& ProML-S & ProML \\ \toprule
    1  & Emotions      & 0.778$\pm$0.029$\bullet$ & 0.776$\pm$0.035$\bullet$ & 0.744$\pm$0.075$\bullet$ & 0.799$\pm$0.024$\bullet$ & 0.774$\pm$0.019$\bullet$ & 0.820$\pm$0.004$\circ$ & 0.780$\pm$0.024$\bullet$ & {0.810$\pm$0.030}\\
    2  & Cal500        & 0.375$\pm$0.018$\bullet$ & 0.373$\pm$0.018$\bullet$ & 0.480$\pm$0.028$\bullet$ & 0.501$\pm$0.005$\bullet$ & 0.286$\pm$0.013$\bullet$ & 0.511$\pm$0.016$\bullet$ & 0.366$\pm$0.015$\bullet$ & {0.513$\pm$0.014} \\
    3  & Birds         & 0.744$\pm$0.018$\circ$ & 0.745$\pm$0.016$\circ$ & 0.678$\pm$0.131$\bullet$ & 0.759$\pm$0.034$\circ$ & 0.678$\pm$0.041$\bullet$ & 0.602$\pm$0.029$\bullet$ & 0.718$\pm$0.017$\circ$ & {0.715$\pm$0.018} \\
    4  & Image         & 0.763$\pm$0.022$\bullet$ & 0.767$\pm$0.017$\bullet$ & 0.708$\pm$0.030$\bullet$ & 0.775$\pm$0.006$\bullet$ & 0.759$\pm$0.010$\bullet$ & 0.820$\pm$0.015$\circ$ & 0.762$\pm$0.016$\bullet$ &{0.791$\pm$0.016} \\
    5  & Scene         & 0.818$\pm$0.010$\bullet$ & 0.817$\pm$0.009$\bullet$ & 0.751$\pm$0.021$\bullet$ & 0.845$\pm$0.007$\circ$ & 0.784$\pm$0.009$\bullet$ & 0.879$\pm$0.010$\circ$ & 0.826$\pm$0.007$\bullet$ &{0.830$\pm$0.007} \\
    6  & Yeast         & 0.740$\pm$0.012$\bullet$ & 0.737$\pm$0.010$\bullet$ & 0.714$\pm$0.203$\bullet$ & 0.756$\pm$0.014$\circ$ & 0.723$\pm$0.007$\bullet$ & 0.767$\pm$0.017$\circ$ & 0.751$\pm$0.006$\bullet$ &{0.755$\pm$0.006} \\
    7  & Business      & 0.821$\pm$0.010$\bullet$ & 0.821$\pm$0.021$\bullet$ & 0.856$\pm$0.103$\bullet$ & 0.878$\pm$0.008$\bullet$ & 0.751$\pm$0.009$\bullet$ & 0.895$\pm$0.009$\circ$ & 0.858$\pm$0.003$\bullet$ &{0.881$\pm$0.003} \\
    8 & Entertainment & 0.682$\pm$0.011$\bullet$ & 0.683$\pm$0.012$\bullet$ & 0.523$\pm$0.051$\bullet$ & 0.690$\pm$0.002$\circ$ & 0.425$\pm$0.013$\bullet$ & 0.679$\pm$0.005$\bullet$ & 0.675$\pm$0.005$\bullet$ &{0.684$\pm$0.005} \\
    9 & Education     & 0.574$\pm$0.068$\bullet$ & 0.569$\pm$0.074$\bullet$ & 0.502$\pm$0.011$\bullet$	& 0.633$\pm$0.009$\circ$ & 0.272$\pm$0.011$\bullet$ & 0.621$\pm$0.013$\bullet$ & 0.618$\pm$0.007$\bullet$ &{0.630$\pm$0.007} \\
    10 & Computers     & 0.674$\pm$0.045$\bullet$ & 0.673$\pm$0.042$\bullet$ & 0.612$\pm$0.064$\bullet$ & 0.706$\pm$0.011$\circ$ & 0.499$\pm$0.009$\bullet$ & 0.707$\pm$0.004$\circ$ & 0.682$\pm$0.009$\bullet$ &{0.695$\pm$0.010} \\
    11 & Health        & 0.677$\pm$0.010$\bullet$ & 0.679$\pm$0.010$\bullet$ & 0.630$\pm$0.056$\bullet$ & 0.781$\pm$0.013$\bullet$ & 0.611$\pm$0.014$\bullet$ & 0.789$\pm$0.008$\circ$ & 0.752$\pm$0.006$\bullet$ &{0.781$\pm$0.006} \\
    12 & Social        & 0.690$\pm$0.017$\bullet$ & 0.685$\pm$0.020$\bullet$ & 0.681$\pm$0.144$\bullet$ & 0.776$\pm$0.005$\circ$	& 0.631$\pm$0.016$\bullet$ & 0.764$\pm$0.004$\circ$ & 0.751$\pm$0.017$\bullet$ &{0.752$\pm$0.016} \\ 
    \hline
    & Average value 	& 0.695$\bullet$			& 0.694$\bullet$				& 0.657$\bullet$				& 0.742$\circ$				& 0.599$\bullet$			& 0.738$\circ$			& 0.712$\bullet$     & 0.736\\
    \bottomrule
\end{tabular}
\vspace{-2mm}
\end{table*}

\subsection{Comparison with Asymmetric-related Methods}
To verify the advantage of the proposed asymmetric architecture design, first, we construct an equivalent symmetric variant by replacing the \textit{Sugeno integral} with the \textit{inner product} in the fusion step, (\textit{i.e.}, $\boldsymbol{r}(y_i, y_j) = r(F, \tilde{y}_i)^\top r(F, \tilde{y}_j), i\neq j$). By preserving the structure of Eq. (\ref{equ:deri}), the correlation matrix degenerates into a symmetric matrix, and this algorithm is denoted as ProF2S. Secondly, we compare the proposed method with the existing representative classification methods focusing on asymmetric label dependence, namely ACML, MAHR, LSDF, and feature selection methods including ML2G, ML2G-s, MFSRD. To ensure the fairness of the experiment, the hyper-parameters of all comparison methods follow the recommended settings in existing literature.

{Table \ref{tab:hl_asy} and Table XII.2 (see in Supplementary Material XII)} and Fig. \ref{fig:rank_asy} shows the comprehensive performance of seven asymmetric-related algorithms across twelve datasets, where ML2G-s is the symmetric label correlation algorithm of ML2G. There are several observations as follows. 
\begin{enumerate}
    \item ProFS occupies the outermost and sub-outer positions in all radar plots. This not only indicates that ProFS can overcome the trade-off phenomenon among multiple indicators, but also exhibits good generalization performance on various datasets.
    \item ProF2S exhibits significant inward collapse along the axis on nearly all datasets. Under the same feature mapping architecture and regularization constraints, the radial distance difference between the two methods on the radar plots is purely determined by the single variable of whether to maintain the directionality implied by labels. This visual quantification indicates that asymmetry is more consistent with real-world distributions.
    \item Compared with ML2G, MFSRD and MAHR, ProFS exhibits a relatively perfect envelope advantage. The deeper mechanism lies in the completeness of the modeling dimension. MAHR only focuses on global hypothesis reuse and ignores the local heterogeneity of the sample space, while ML2G and MFSRD generate global information from local information, they are essentially still local relations. In contrast, ProFS relies on the dual-track fusion mechanism combining global and local perspectives to achieve a superior performance.
    \item In subfigures (a), (b), (c), and (e), it can be observed that ProFS is only slightly inferior to LSDF and ACML. The fundamental task of LSDF and ACML is multi-label classification,  which can fully utilize all original information in feature space. In contrast, as a feature selection paradigm, ProFS, can still approximate the performance of the full-information method. This performance highlights the efficiency of the dual-track asymmetric fusion mechanism in extracting discriminative structural information.
\end{enumerate}

\begin{figure*}[!t] 
    \vspace{-4mm}
	\centering 
	\subfloat[Hamming loss]{\includegraphics[width=0.33\textwidth]{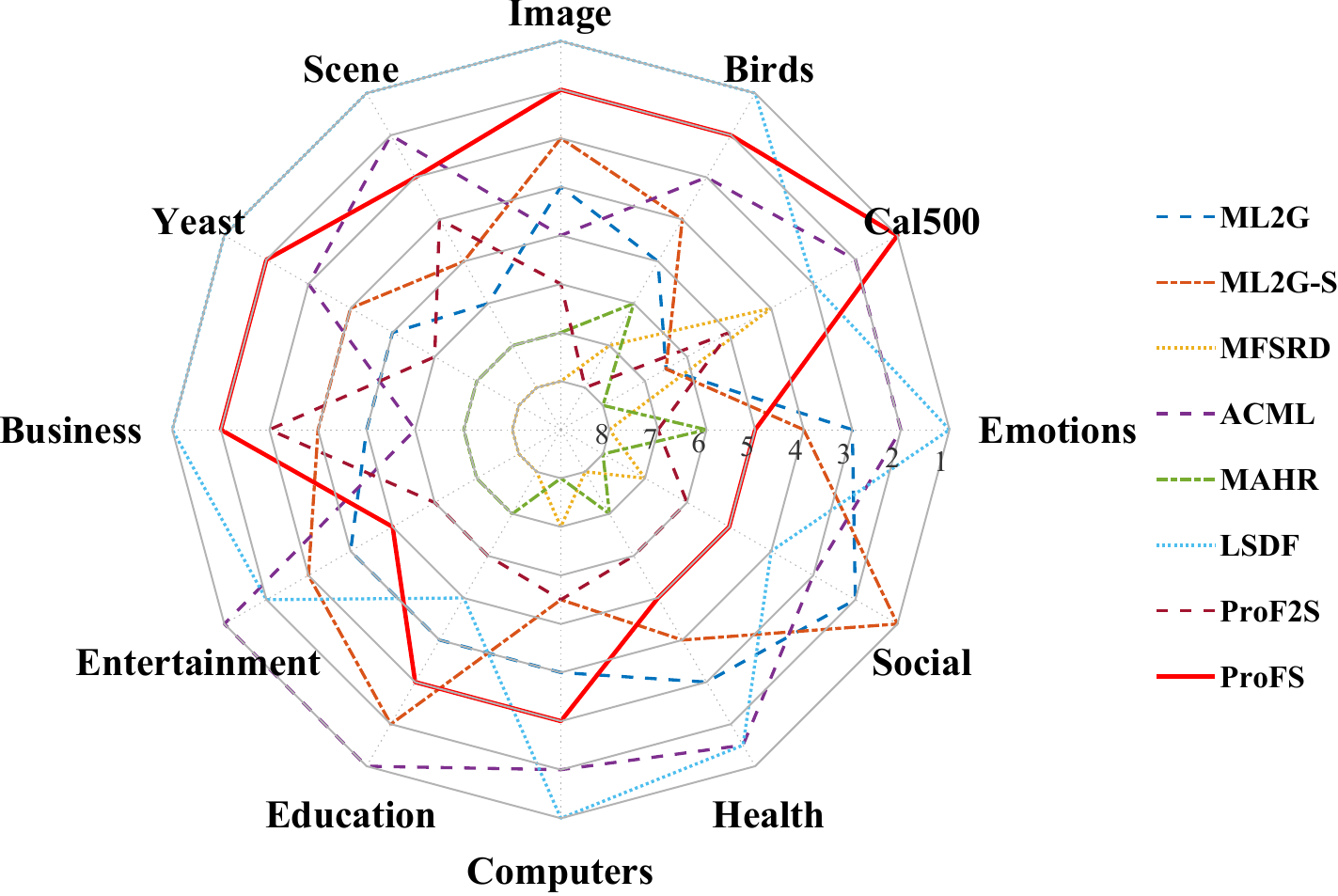}} 
    \hfill
	\subfloat[Ranking loss]{\includegraphics[width=0.33\textwidth]{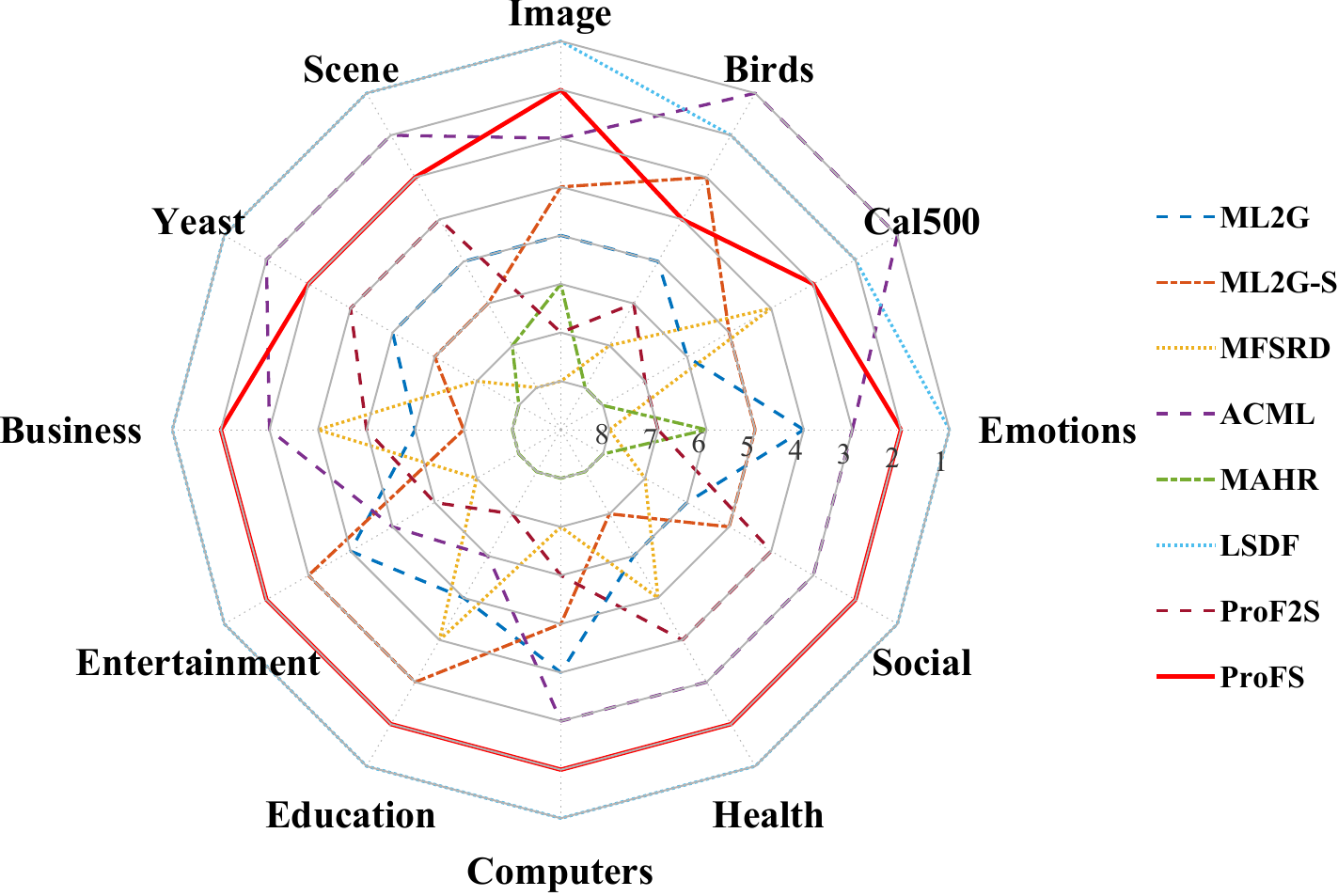}}
    \hfill
	\subfloat[One error]{\includegraphics[width=0.33\textwidth]{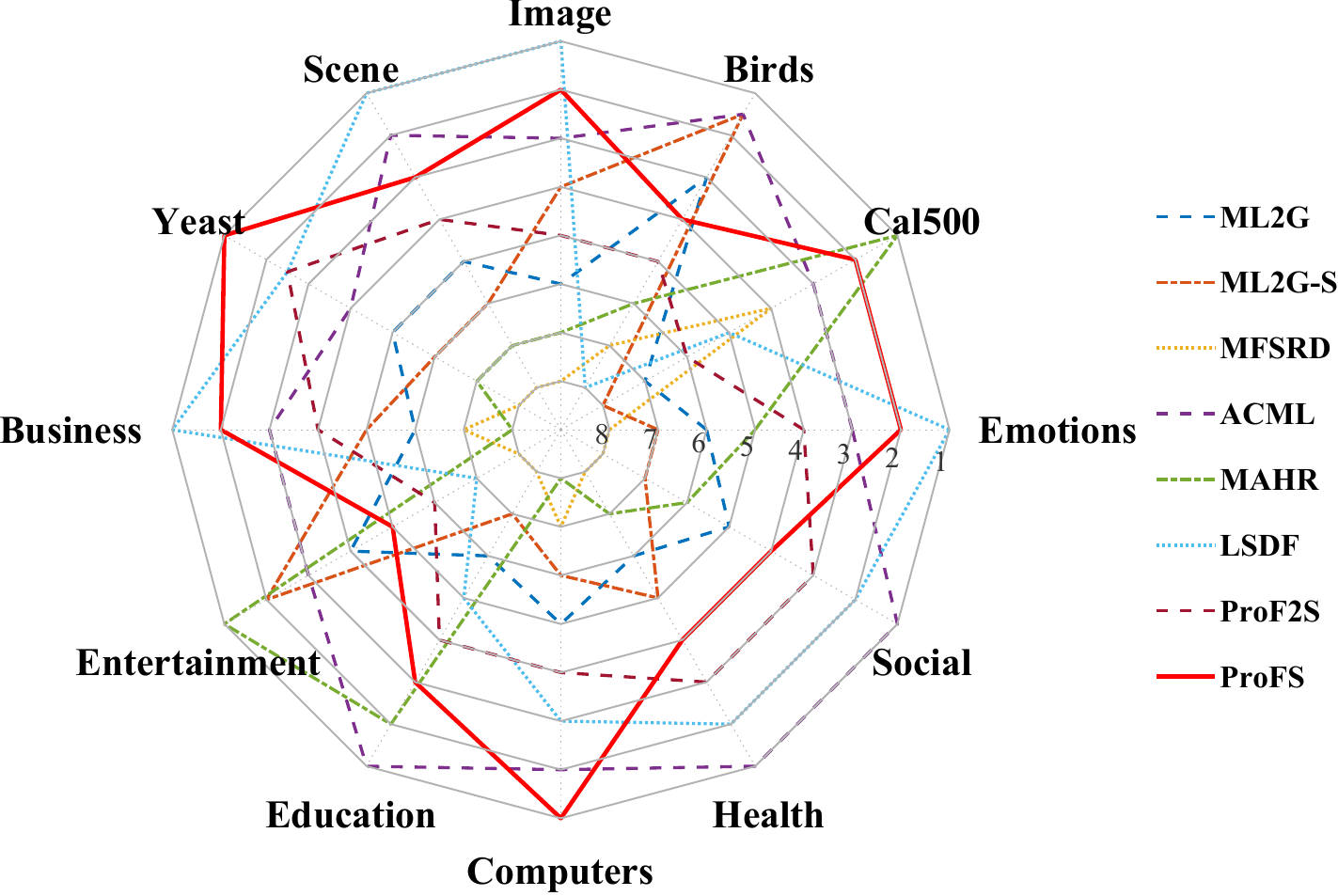}}

    \subfloat[Coverage]{\includegraphics[width=0.33\textwidth]{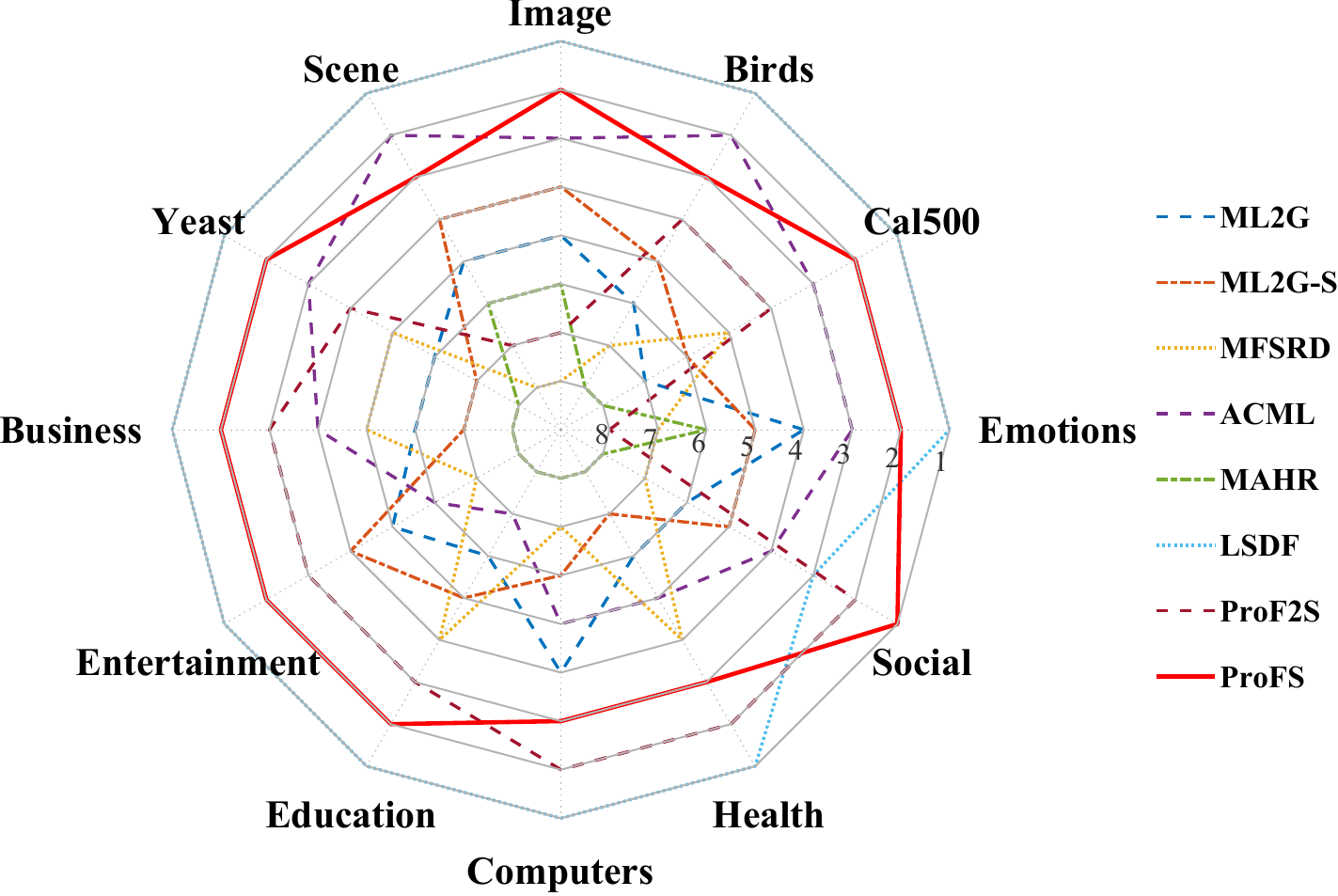}}  
    \subfloat[Average precision]{\includegraphics[width=0.33\textwidth]{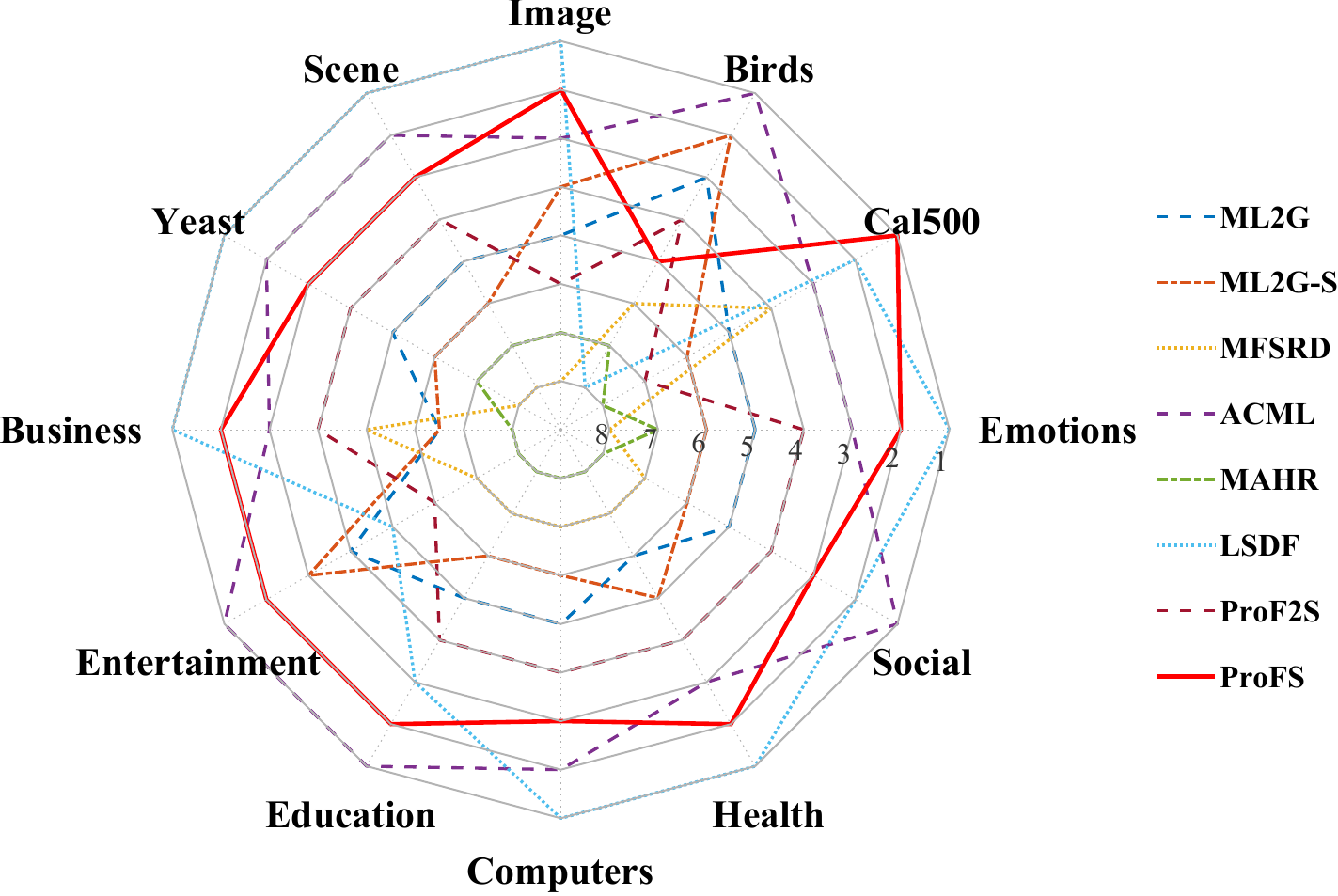}}
	\caption{Comparison of ProFS and asymmetric-related comparing algorithms on the five evaluation metrics.}  
	\label{fig:rank_asy}  
\end{figure*}

\begin{figure*}[t]
	\centering 
	\subfloat[Hamming loss]{\includegraphics[width=0.2\textwidth]{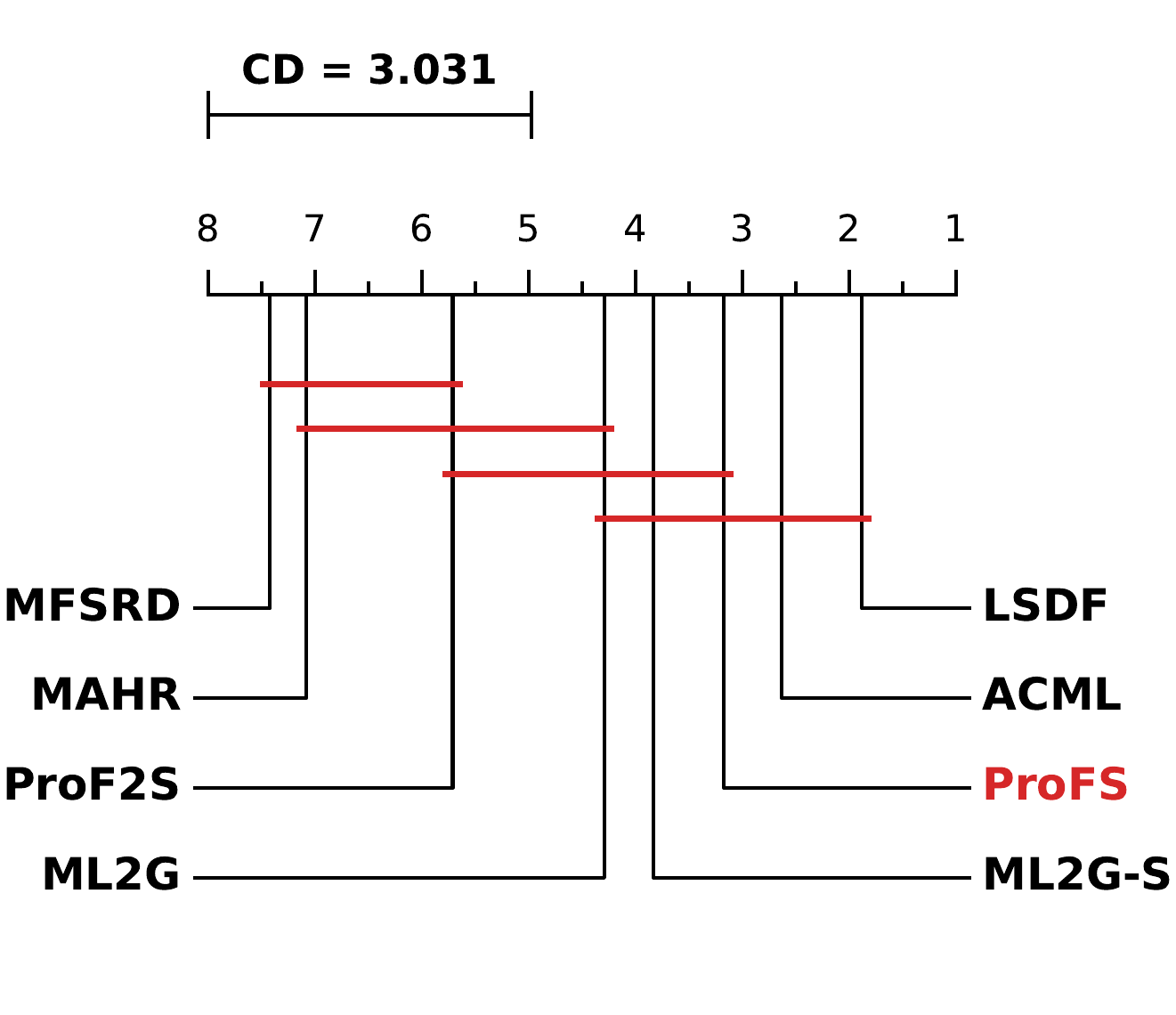}}  
	\subfloat[Ranking loss]{\includegraphics[width=0.2\textwidth]{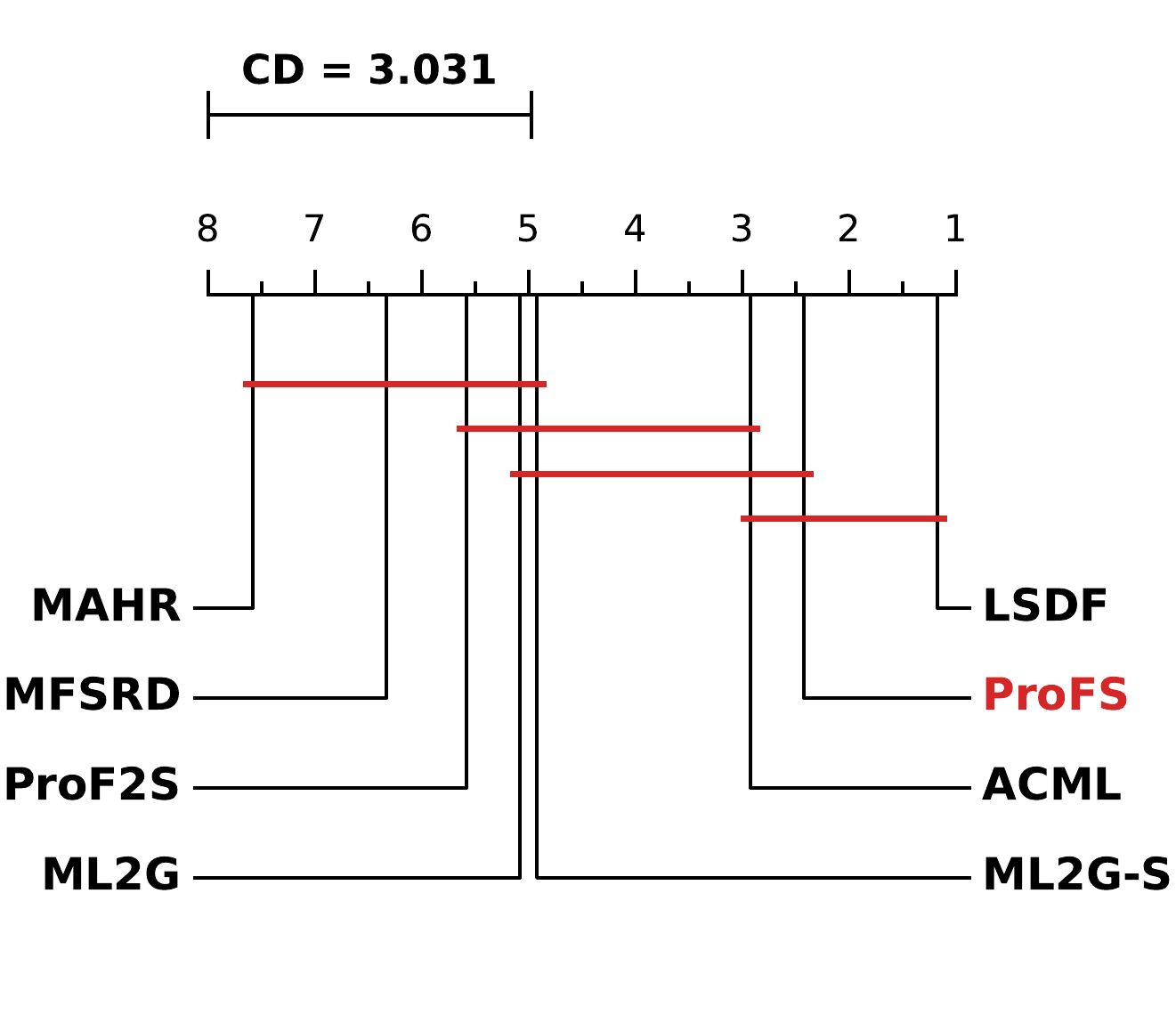}} 
	\subfloat[One error]{\includegraphics[width=0.2\textwidth]{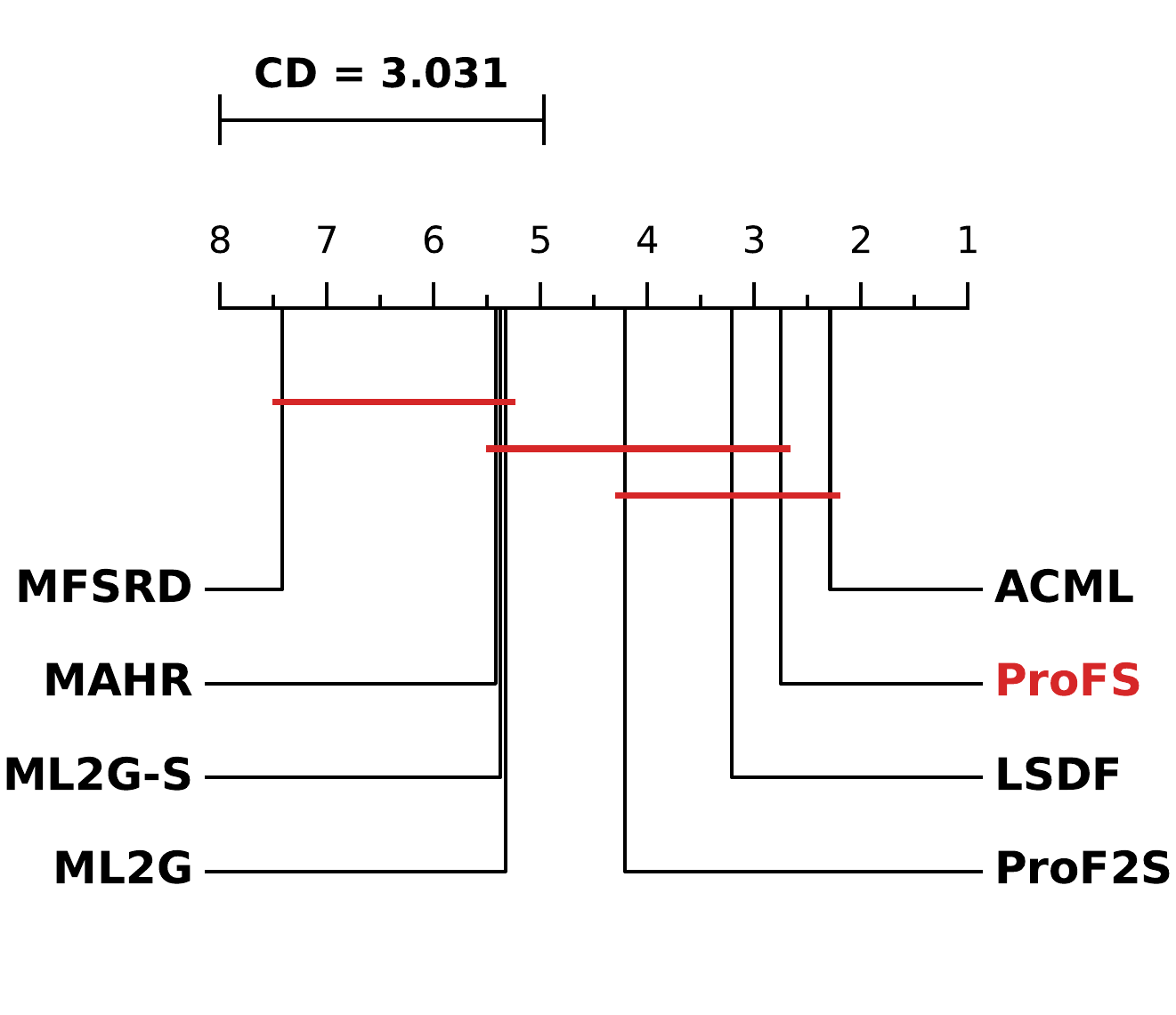} } 
	\subfloat[Coverage]{\includegraphics[width=0.2\textwidth]{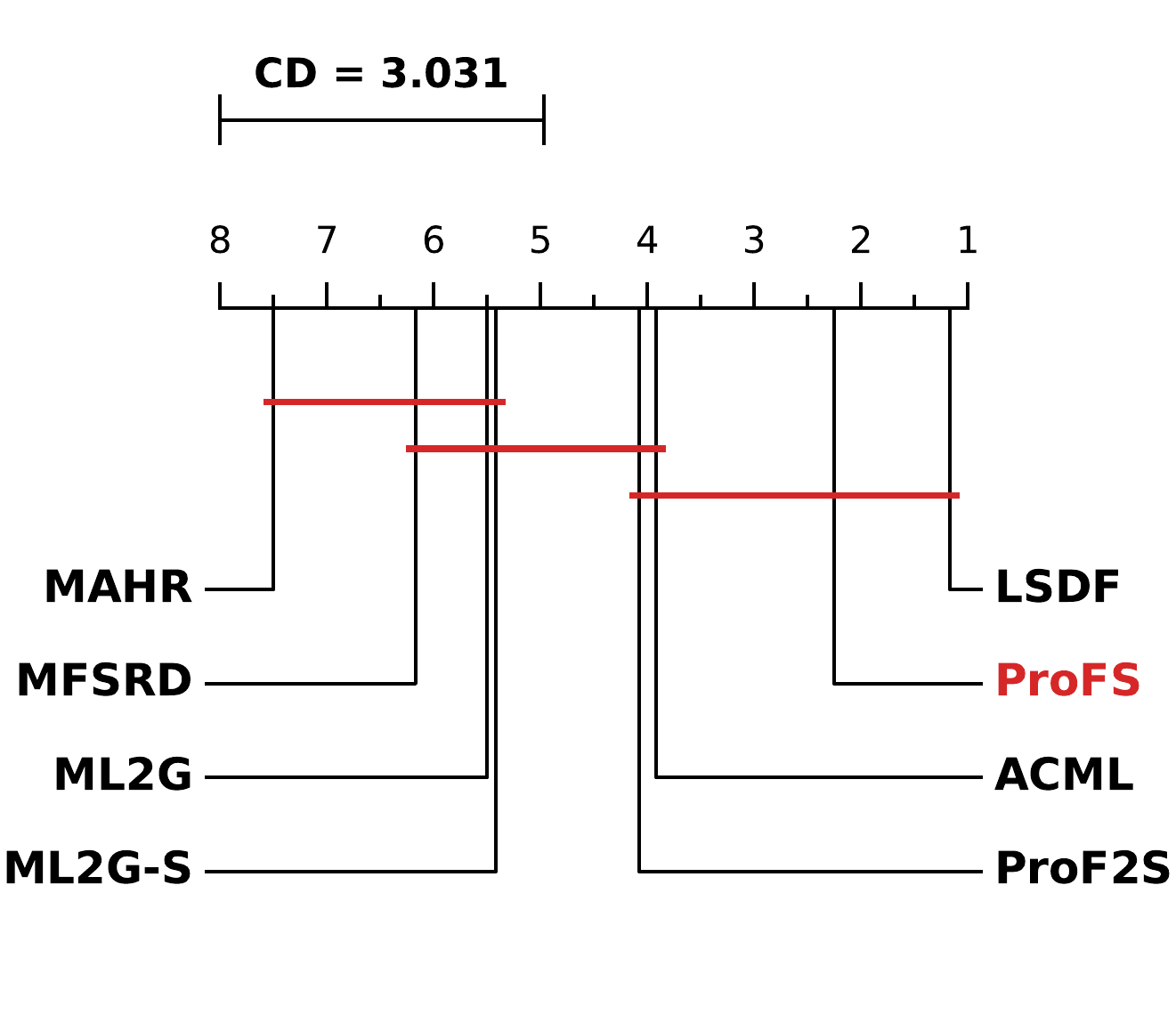} }
	\subfloat[Average precision]{\includegraphics[width=0.2\textwidth]{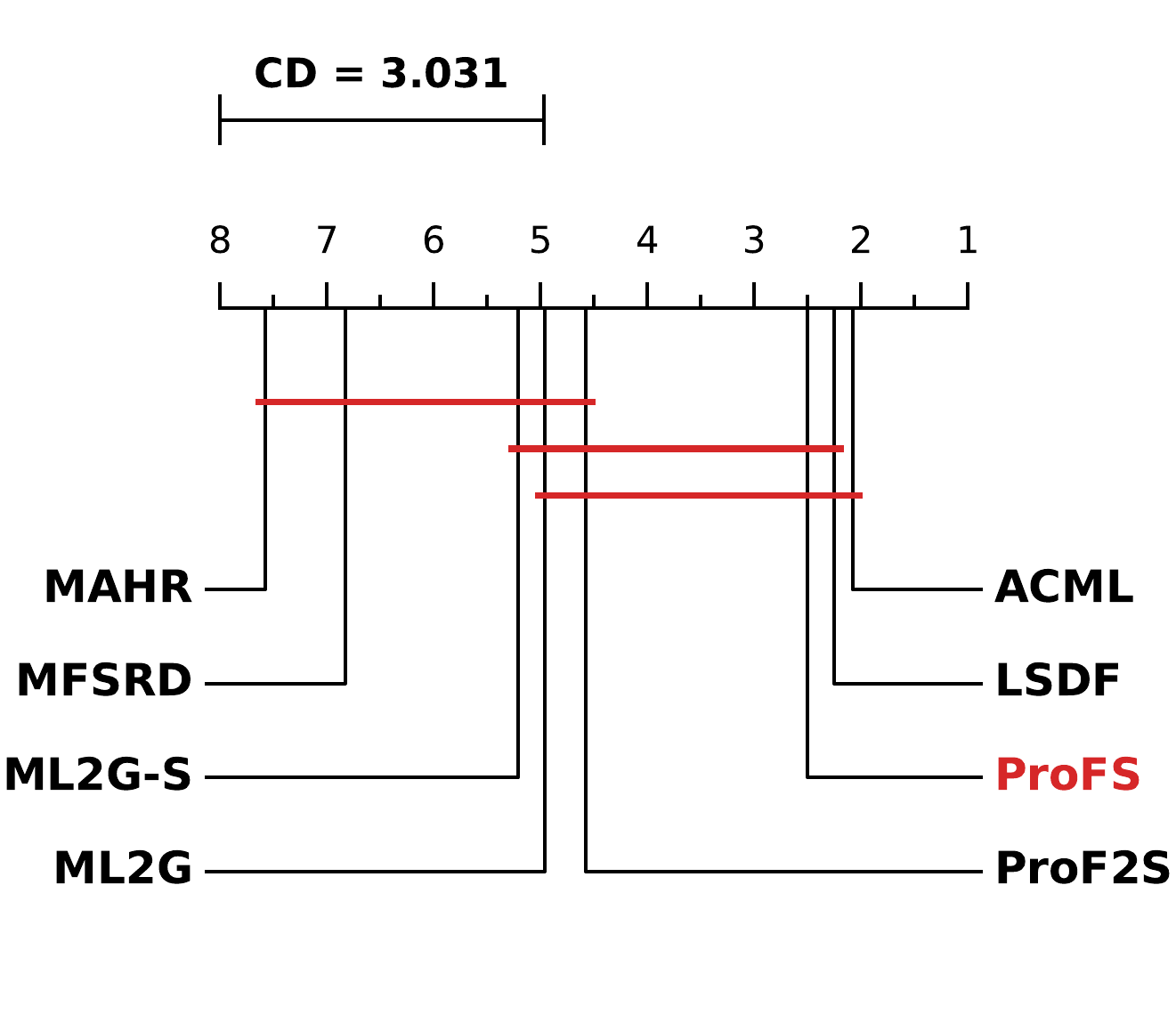} }
	\caption{Comparison of ProFS against seven asymmetric-related comparing algorithms with the \textit{Nemenyi test}.}  
	\label{fig:five-cd_asy}  
\end{figure*}

\begin{table}[!hbt]
	\scriptsize
	\centering
	\setlength{\tabcolsep}{20pt}
	\caption{{Wilcoxon signed-ranks test for ProFS and ProF2S}}
	\label{tab:5}
    \footnotesize
		\begin{tabular}{cc}
			\hline
			Evaluation metric	& ProFS against ProF2S \\ \hline
			\textit{Hamming loss}		&   win[9.8e-4]                    \\
			\textit{Ranking loss}		&   win[4.9e-4]                    \\
			\textit{One-error}			&   win[2.2e-2]                    \\
			\textit{Coverage}			&	win[9.3e-3]						\\
			\textit{Average precision}	&   win[3.9e-3]                    \\ \hline
	\end{tabular}
    \vspace{-3mm}
\end{table}

Table \ref{tab:5} reported the results of the Wilcoxon signed-rank test. The results show that ProFS achieved a statistically significant advantage over ProF2S across all metrics, with extremely low $p$-values (all well below 0.05). This indicates that the improved performance of multi-label feature selection is primarily due to the asymmetric topological structure of the labels, rather than solely the inclusion of label correlations.

\begin{figure*}[ht]  
	\vspace{-1mm}
	\centering  
	\subfloat{\includegraphics[width=0.16\textwidth]{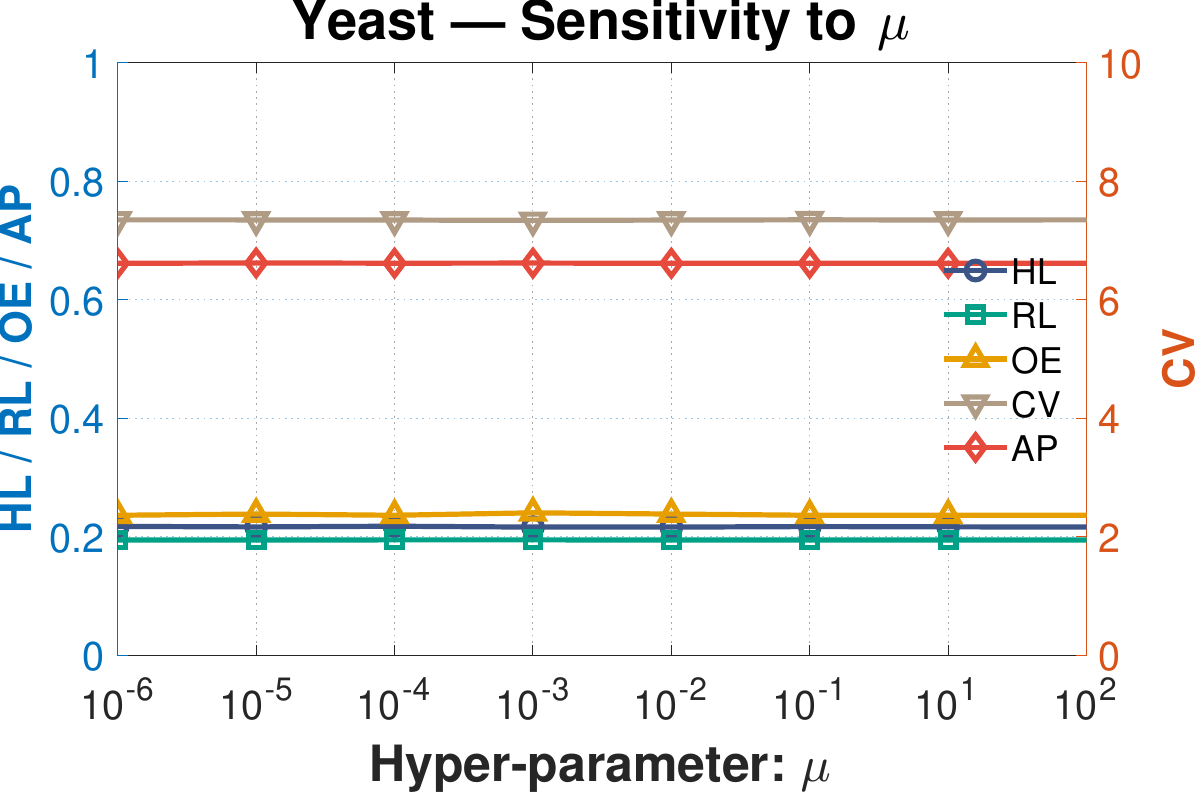}}
	\hfill
	\subfloat{\includegraphics[width=0.16\textwidth]{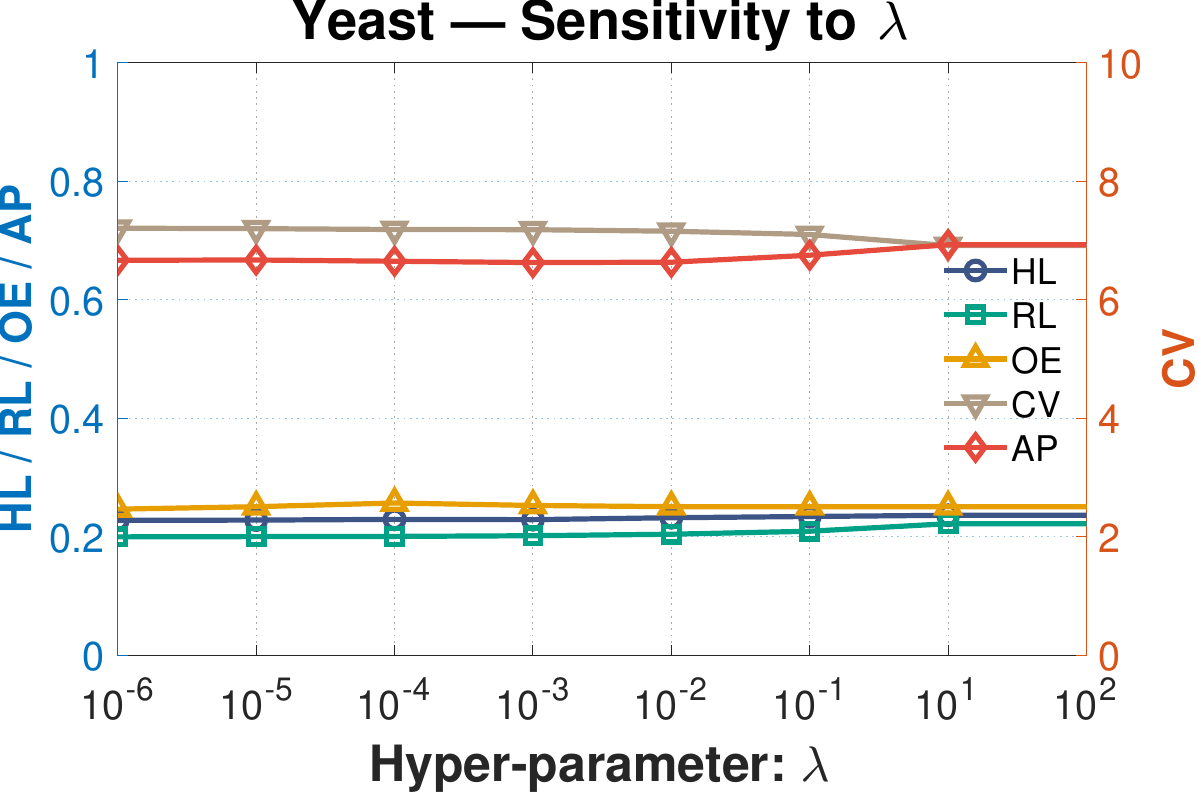}}
    \hfill
    \subfloat{\includegraphics[width=0.16\textwidth]{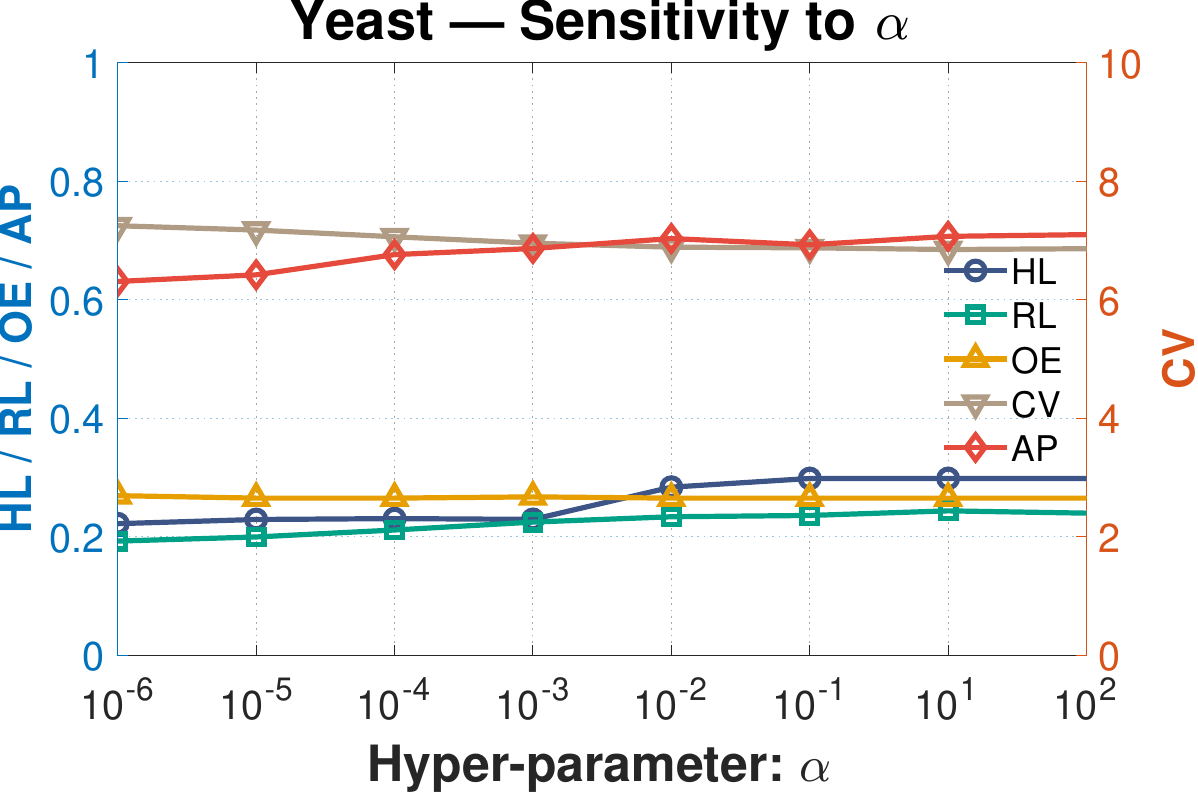}}
    \hfill
    \subfloat{\includegraphics[width=0.16\textwidth]{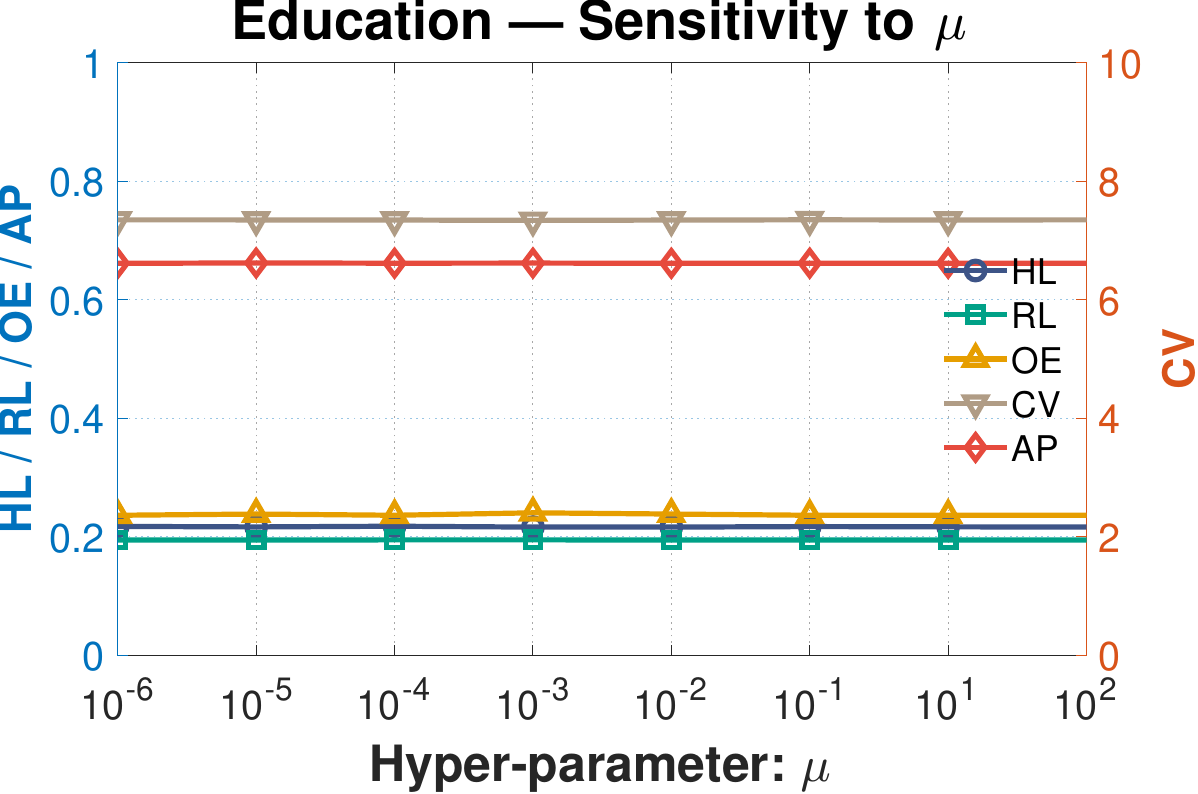}}	
    \hfill
    \subfloat{\includegraphics[width=0.16\textwidth]{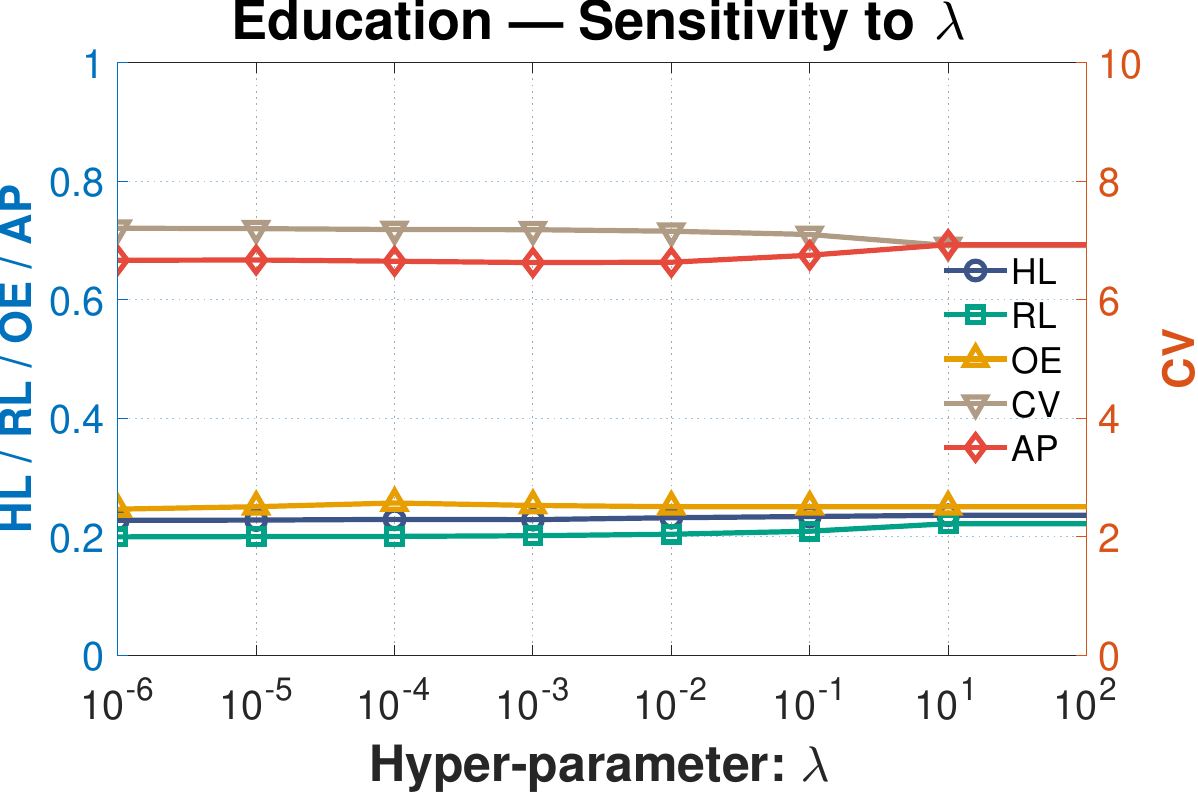}}  
    \hfill
    \subfloat{\includegraphics[width=0.16\textwidth]{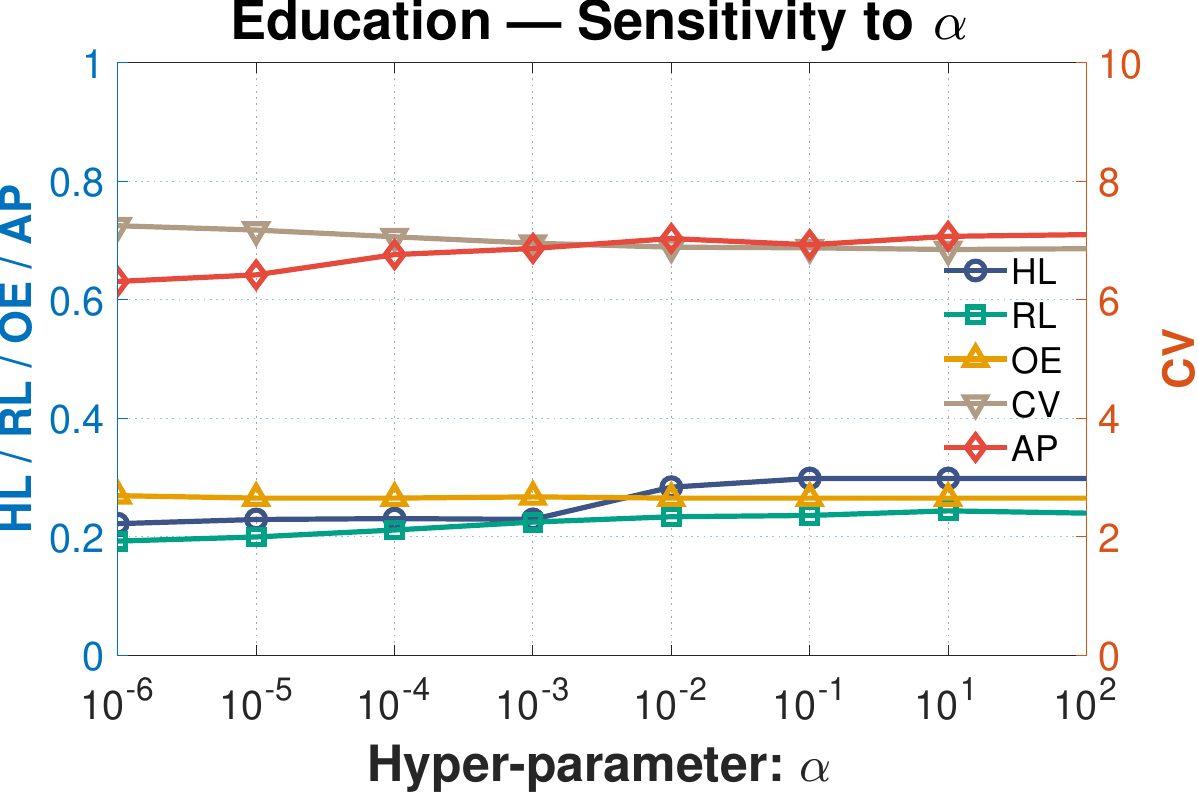}}  
	\caption{Sensitivity analysis. The value of \textit{CV} is based on the right $y$-axis, while the other four metrics are based on the left one.}
    \vspace{-1mm}
	\label{fig:para_ana}  
\end{figure*}

\begin{figure*}[ht]  
	\centering  
    \subfloat[CAL500]{\includegraphics[width=0.16\textwidth]{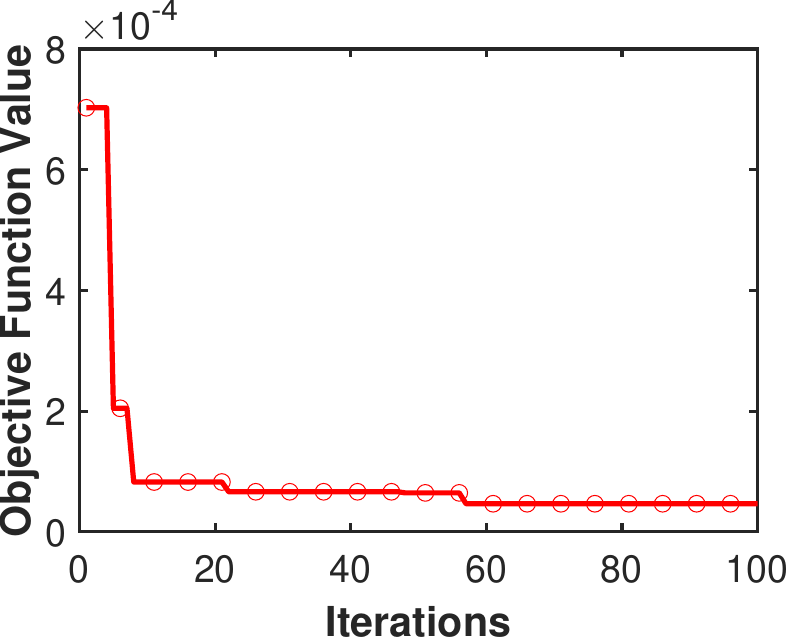}}\hfill
    \subfloat[Emotions]{\includegraphics[width=0.16\textwidth]{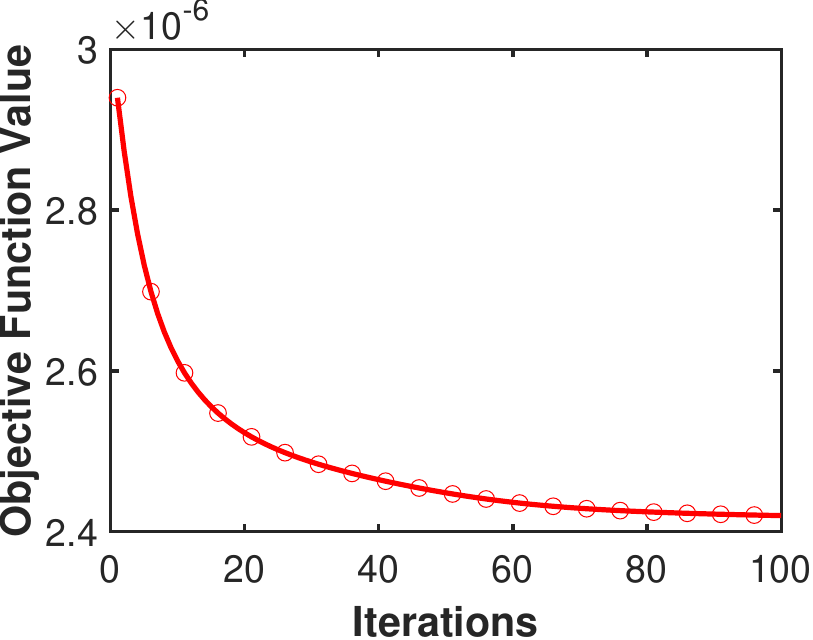}}\hfill
    \subfloat[Birds]{\includegraphics[width=0.16\textwidth]{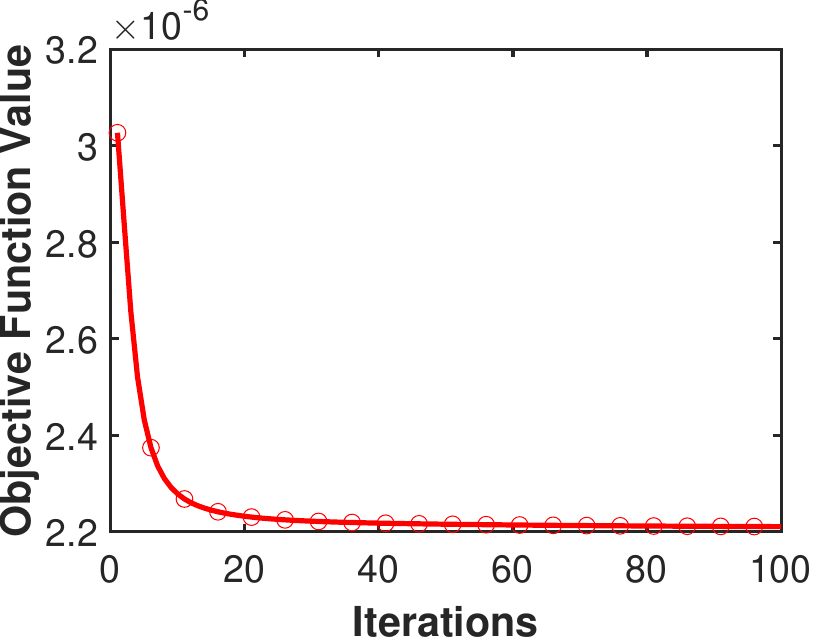}}\hfill
    \subfloat[Image]{\includegraphics[width=0.16\textwidth]{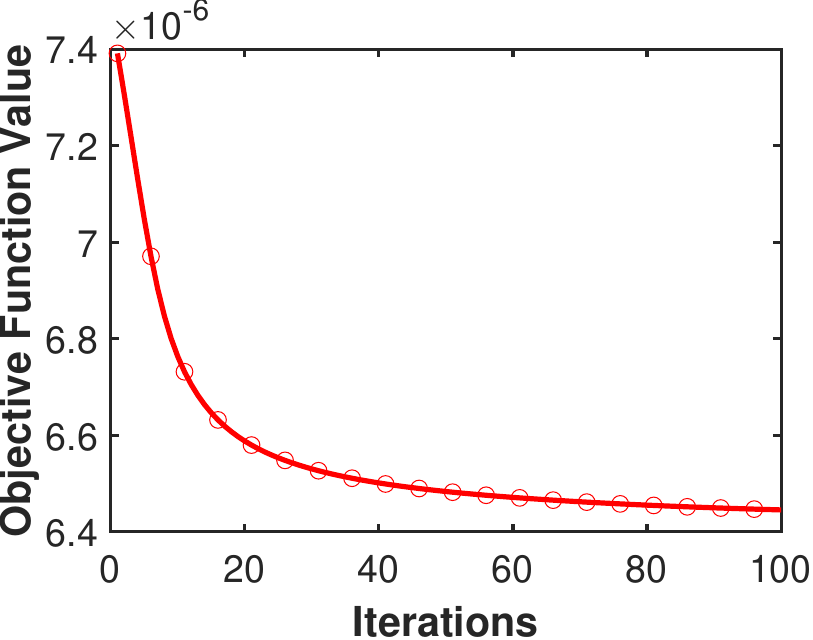}}\hfill
    \subfloat[Scene]{\includegraphics[width=0.16\textwidth]{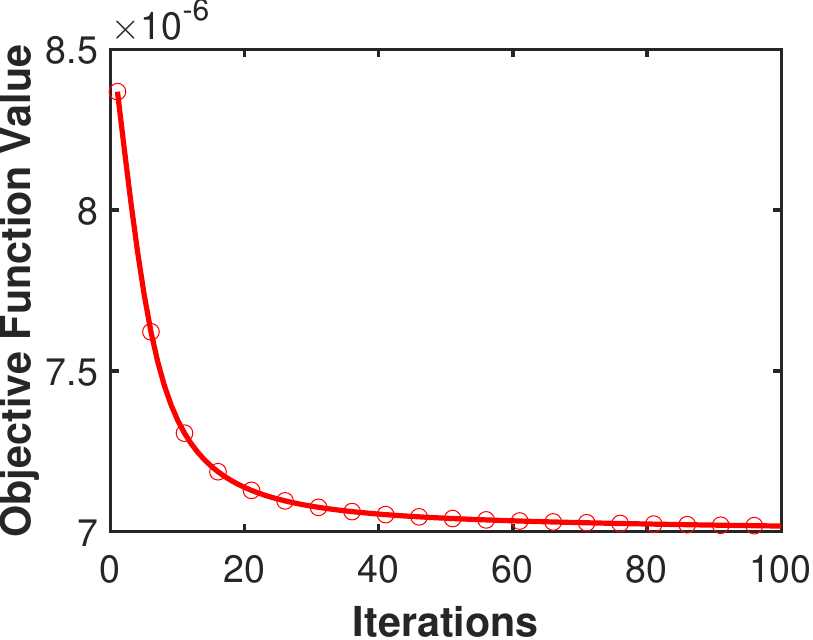}}\hfill
    \subfloat[Yeast]{\includegraphics[width=0.16\textwidth]{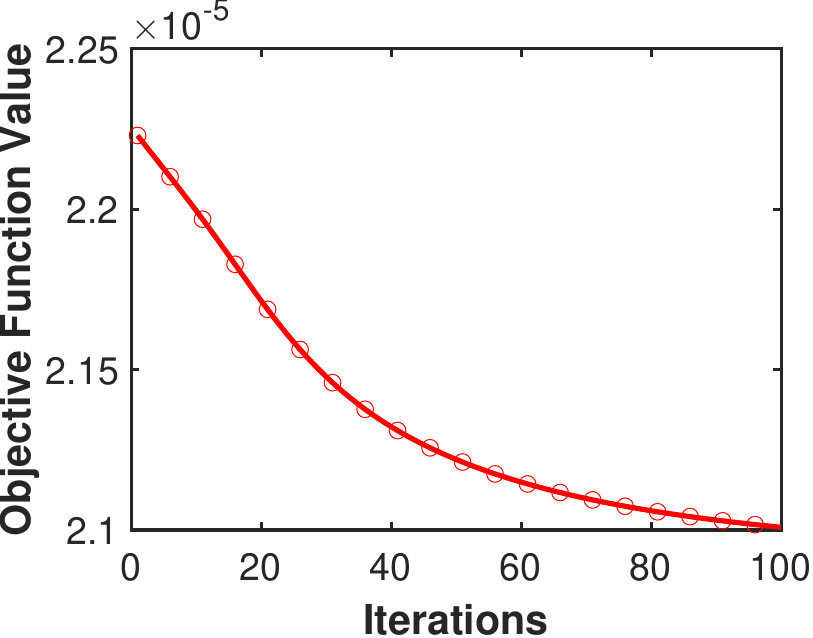}}

    \subfloat[Business]{\includegraphics[width=0.16\textwidth]{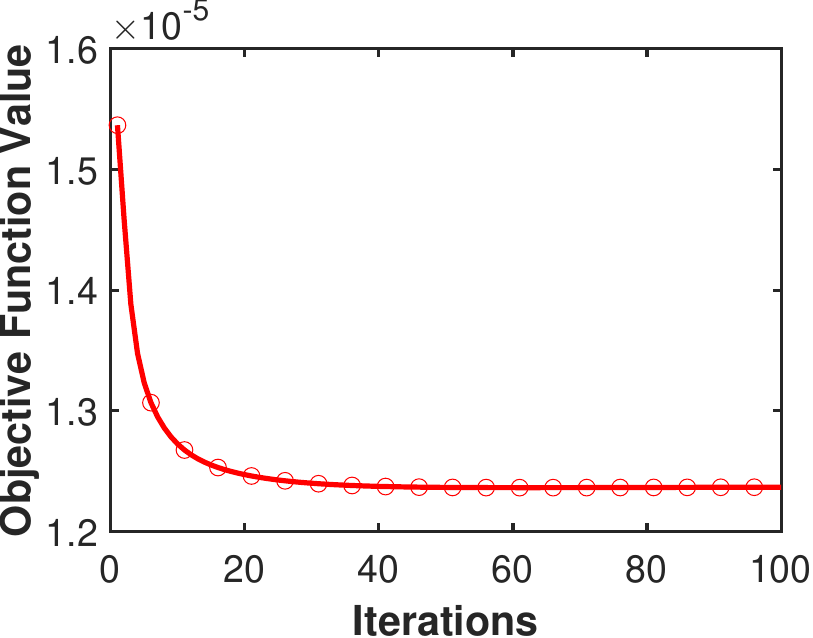}}\hfill
    \subfloat[Entertainment]{\includegraphics[width=0.16\textwidth]{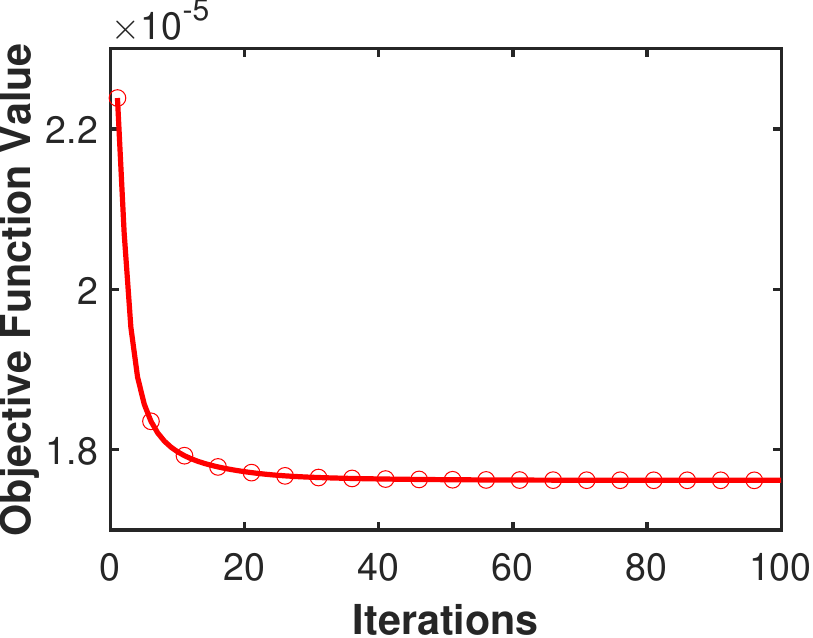}} \hfill
	\subfloat[Education]{\includegraphics[width=0.16\textwidth]{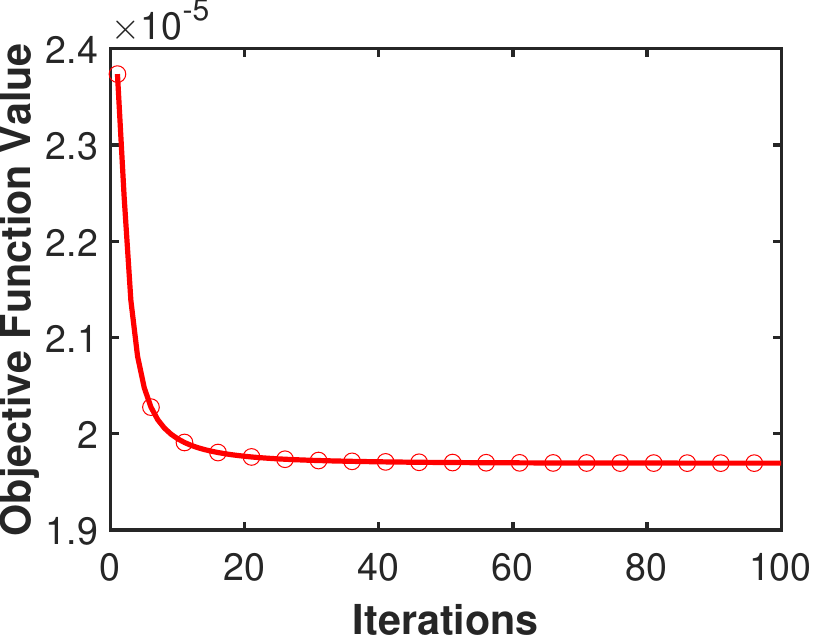}}\hfill
	\subfloat[Computers]{\includegraphics[width=0.16\textwidth]{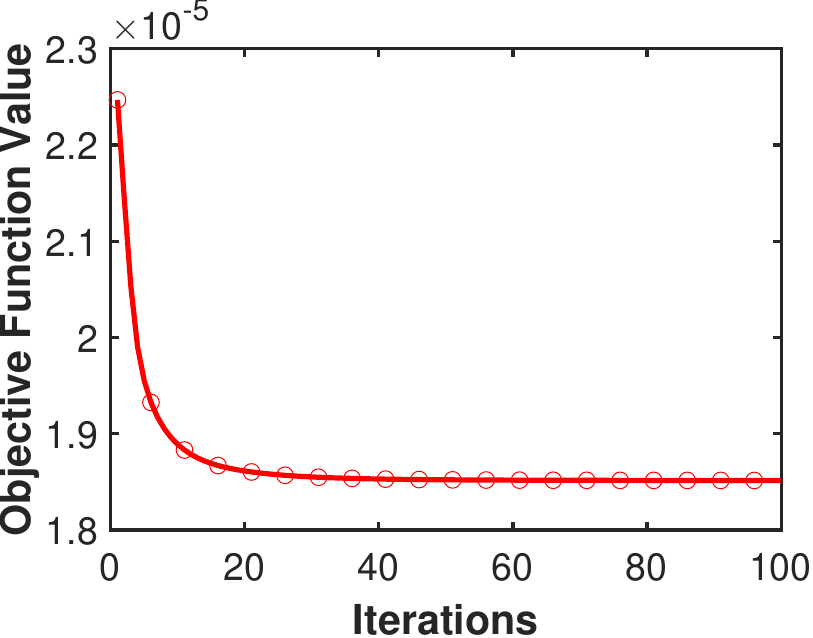}} \hfill
	\subfloat[Health]{\includegraphics[width=0.16\textwidth]{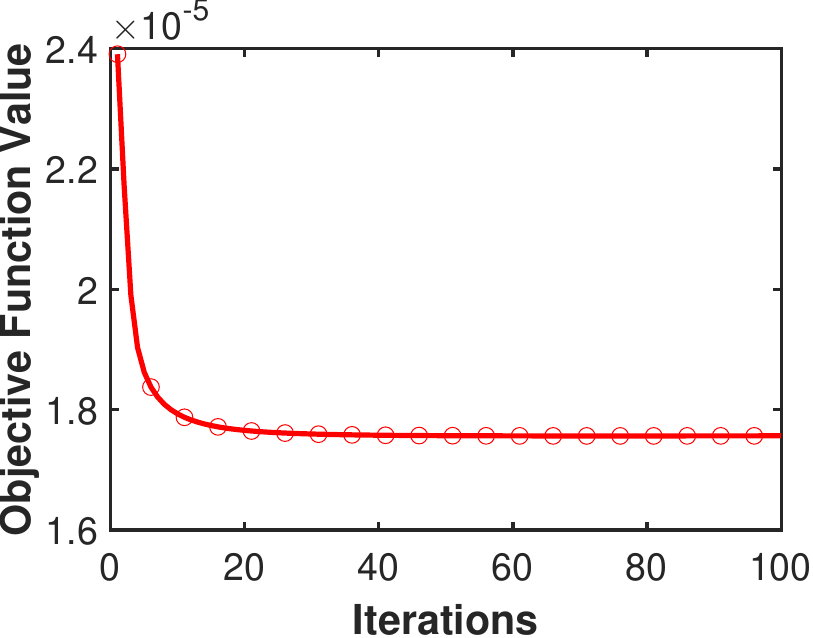}}  \hfill
    \subfloat[Social]{\includegraphics[width=0.16\textwidth]{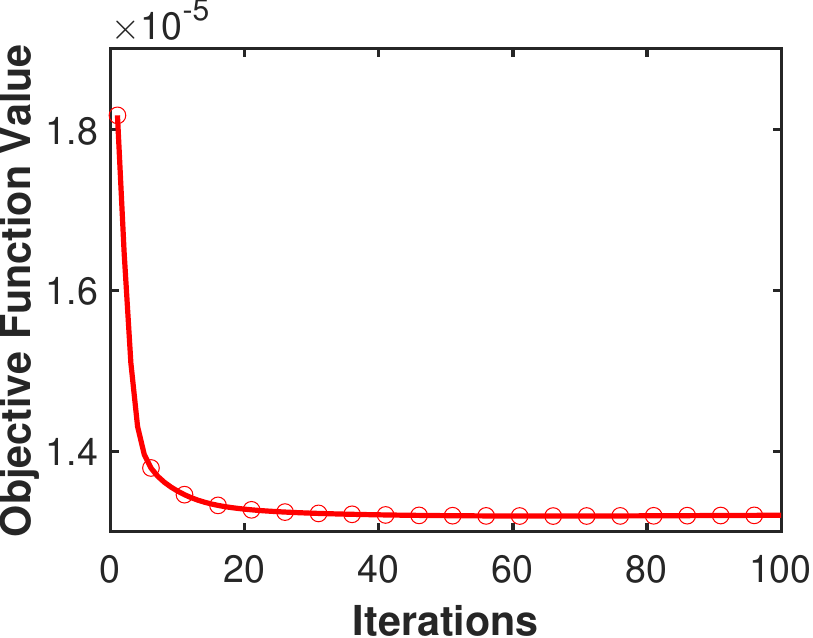}}  
	\caption{Convergence curves on various datasets.}
	\label{fig:subfloats}  
\end{figure*}

\begin{table}[h]
	\scriptsize
	\centering
    \vspace{-1mm}
	\caption{{Summary of Friedman statistics $F_F$ \textit{w.r.t.} asymmetric-related methods}}
	\label{tab:ff_asy}
    \footnotesize
    \setlength{\tabcolsep}{10pt}
    \begin{tabular}{lcc}
        \toprule
        Evaluation metric & $F_F$ & \multicolumn{1}{l}{Critical value ($\alpha=0.05$)} \\ \toprule
        \textit{Hamming loss}      &  25.5509 & \multirow{5}{*}{2.219}                  \\
        \textit{Ranking loss}      &  37.2567 &                                    \\
        \textit{One-error}         &  10.6099 &                                    \\
        \textit{Coverage}          &  28.7916 &                                    \\
        \textit{Average precision} &  29.0565  &                                    \\ 
        \bottomrule
	\end{tabular}
    \vspace{-3mm}
\end{table}

Additionally, Table \ref{tab:ff_asy} reports the Friedman statistic $F_F$ and its corresponding critical values for each evaluation metric ($k\!=\!8,N\!=\!12$). For each metric, at the significance level $\alpha=0.05$, the null hypothesis that the performance of the comparison algorithms is indistinguishable is rejected. Subsequently, the critical difference for the \textit{Nemenyi post-hoc test} as shown in Fig. \ref{fig:five-cd_asy}. Under the strict scale $CD_{\alpha}\!=\!3.031$ ($q_\alpha\!=\!3.031$, $k\!=\!8$), subfigures (b), (c) and (e) clearly show that ProFS exhibits no significant difference from LSDF and ACML, while achieving an extremely high average ranking. In contrast, it exhibits statistically significant differences from ML2G and MFSRD, which also belong to the same feature selection paradigm. These results confirm that the global-local dual-track asymmetric fusion mechanism achieves high information retention rate in capturing core discriminative topological structures, as well as the architectural advantages of introducing the prediction matrix $\mathbf{P}$ to decouple task.

\subsection{Sensitivity Analysis}\label{ssec:parana}
There are three parameters in our formulation. In this experiment, we conduct parameter analysis via varying one parameter while keeping the other two parameters fixed. Fig. \ref{fig:para_ana} shows the sensitivity curve results for \textit{Yeast} and \textit{Education} datasets. The experimental results indicate that the model is insensitive to changes in the hyperparameter $\mu$, which is responsible for controlling the profiling-induced scaling factor $s=\frac{\mu}{1+\mu}$, thereby affecting the internal optimization geometry and the fidelity-regularization balance. 
However, due to the saturation behavior of $s$ and the joint effect of structural and sparsity regularization, the final prediction metrics may remain stable across a wide range of $\mu$, indicating that ProFS is robust to the choice of $\mu$. 
While for the parameters $\lambda$ (control sparsity) and $\alpha$ (control regularization) exhibit slight fluctuations. When $\lambda$ is too large, the performance shows minor variations during generalization. This is because the $\ell_{2,1/2}$-matrix norm imposes a more severe row-sparsity penalty, and overly large weight forces the exclusion of some discriminative features. Conversely, an excessive $\alpha$ leads to the label prior completely dominates the data-driven feature mapping, causing the model to exhibit bias. Based on the above analysis, these parameters can be safely set within verified reasonable ranges in practical applications to ensure the stability and effectiveness of the algorithm.
 
\subsection{Convergence Behavior}
Proposition \ref{cor:conv} proves theoretically that the objective function value monotonically decreases to a convergent value through iterative updates using Algorithm \ref{alg:ProML}. Fig. \ref{fig:subfloats} illustrates the convergence trend of the objective function values across all datasets. 
As can be observed in Fig. \ref{fig:subfloats}, with the exception of the \textit{CAL500} dataset, the objective values exhibit a smooth descent trend, and the number of iterations not exceeding 20 on most datasets. This observation not only highlights the efficiency of the optimization solution of our method but also empirically supports the convergence property.

\begin{figure*}[t] 
    \vspace{1mm}
	\centering 
	\includegraphics[width=\textwidth]{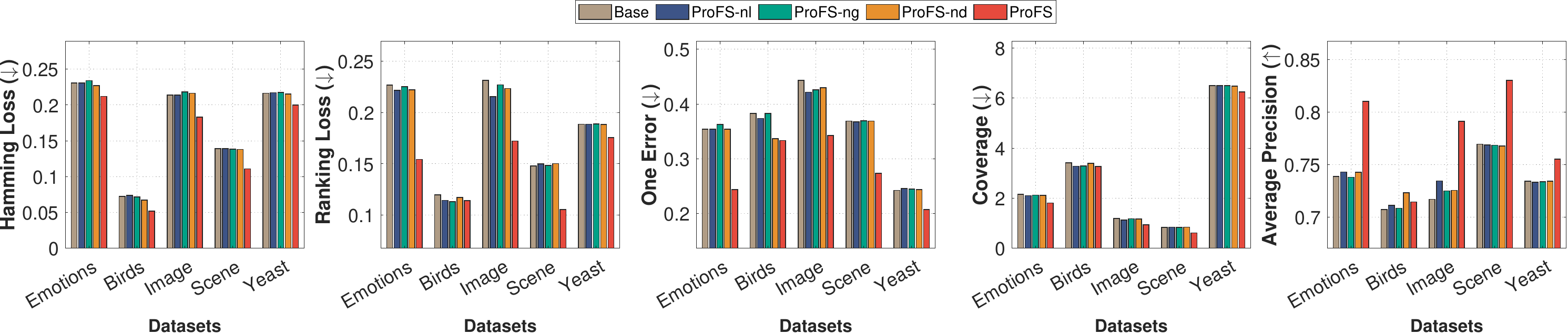}
	\caption{Ablation results of the proposed method on five datasets in terms of five metrics.}  
    \vspace{-1mm}
    \label{fig:ablation}
\end{figure*}

\subsection{Ablation Study}
We further verified the effectiveness of ProFS by decomposing the model. Specifically, four degraded versions of the scheme are designed: (1) \textit{Base}: the original version that doesn't consider dependency and correlation; (2) \textit{ProFS-nl}: the version without the local-correlation module; (3) \textit{ProFS-ng}: the version with the global-dependency module; (4) \textit{ProFS-nd}: the version without the decoupling module.

Fig. \ref{fig:ablation} presents the ablation results of ProFS and its four degraded versions across five evaluation metrics and five datasets. It is clearly observed that using only a single correlation module (\textit{ProFS-nl} or \textit{ProFS-ng}) or a coupled framework (\textit{ProFS-nd}) in relatively modest improvements across all metrics. In contrast, the full version of ProFS achieves the best performance, displaying a sharp drop in \textit{RL}, \textit{OE}, and \textit{CV}, while showing a dramatic increase in \textit{AP}. 
That is, the combined of fusion and decoupling contributes to performance improvement and yields the optimal performance in most scenarios. 
These results strongly indicate that global dependencies and local correlation structures in multi-label feature selection are not simply linearly additive, and that decoupling mechanisms enhance the noise-robustness of sparse feature learning, both playing a crucial role in enhancing model performance. 

\subsection{Further Analysis}
As discussed in the Introduction, coupled learning methods directly use the observed label matrix to fit the predictive model, which may allow label noise to propagate into the model parameters, thereby impairing the model's generalization performance. To further investigate whether the ProML paradigm can alleviate this issue, we compare GMLL with ProML on the \textit{Emotions} dataset under varying label-noise rates. To rule out the influence of a particular regularization form, we consider four representative regularization terms based on the Frobenius, $\ell_1$, $\ell_{2,1}$, and $\ell_{2,1/2}$ norms. The corresponding results are presented in Fig. XIII.1-XIII.4 of the Supplementary Material. 

As can be observed, under all four regularization settings, ProML consistently outperforms GMLL across different label-noise rates on all metrics. Notably, ProML retains a clear advantage even in the noise-free setting, and this advantage remains evident as the noise rate increases. These results demonstrate that the performance improvement of ProML is not tied to any specific form of regularization. By decoupling model optimization from the direct fitting of observed labels, ProML effectively reduces the influence of corrupted supervisory information on the learning process, thereby providing improved robustness to label noise and better generalization performance.

\section{Conclusion}\label{sec:conclu}
In this paper, we investigated the optimization conflicts arising from the general multi-label learning paradigm, in which label fitting, structural modeling, and complexity control are coupled within a single parameter matrix. To address this limitation, we proposed a Progressive Multi-Label Learning paradigm, termed ProML, which introduces a learnable auxiliary predictive variable to establish a progressive fitting path between model predictions and observed supervision, thereby effectively decoupling label fitting from structural learning. By eliminating the auxiliary variable, we derived an equivalent profiled objective and established a unified view in which GMLL can be recovered as a limiting case of ProML. Our theoretical analysis showed that, under quadratic positive semi-definite structural regularization, ProML attenuates the dominant curvature of the fitting term, reshapes the spectral distribution of the objective, enlarges the stable learning-rate region, and provides a controllable trade-off between label fidelity and structural regularization. Moreover, ProML reduces model complexity through an effective contraction of the hypothesis space, providing theoretical support for its generalization capability.

To demonstrate the applicability of ProML to a concrete multi-label task, we further instantiated it as a multi-label feature selection method, termed ProFS, which characterizes asymmetric label relationships from two complementary perspectives: global feature-label dependencies and local sample structures. The resulting label prior is integrated with an $\ell_{2,1/2}$ row-sparsity regularizer within the progressive learning framework, enabling ProFS to jointly perform label-dependency modeling, and discriminative feature selection. We also analyzed the optimization and generalization properties of ProFS in terms of objective convergence, Rademacher complexity, and algorithmic uniform stability. Extensive experiments on 12 benchmark datasets demonstrated that ProFS achieves robust and competitive predictive performance compared with existing multi-label feature selection methods. 

In future work, we will extend the progressive decoupling mechanism to scenarios involving missing labels, candidate labels, label distribution shifts, and more complex instances of label noise.

\section*{Acknowledgments}
This work was supported by the Major Program of the National Natural Science Foundation of China (T2495250, T2495251).

\bibliographystyle{IEEEtran}
\bibliography{IEEEabrv,references}

\begin{IEEEbiography}[{\includegraphics[width=1in,height=1.25in,clip,keepaspectratio]{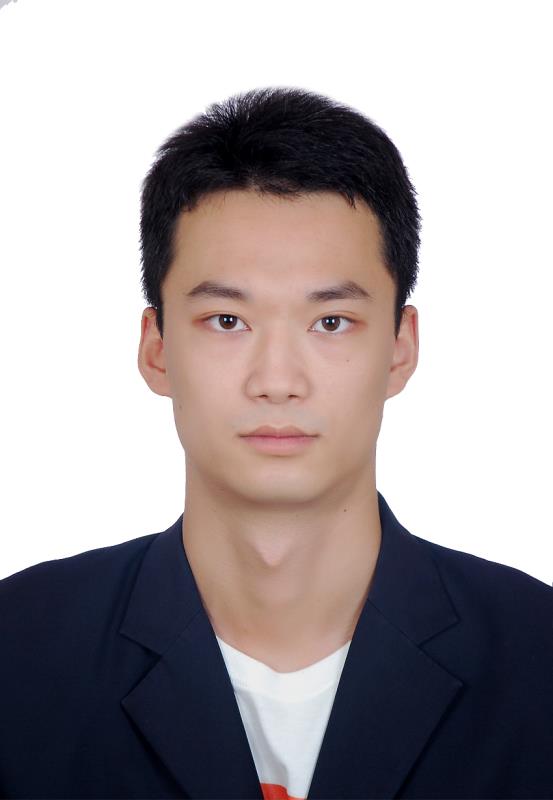}}]{Yi Kou}
	received the M.S. degree at the School of Mathematics and Statistics from Minnan Normal University, Zhangzhou, China. Currently, he is a Ph.D candidate at the Institute of Big Data Science and Industry, Shanxi University. His main research interests include data mining and machine learning. He has published more than 5 papers in his field of research, including in journals such as IEEE Transactions on Fuzzy Systems, Information Sciences, and Expert Systems with Applications.
\end{IEEEbiography}

\begin{IEEEbiography}[{\includegraphics[width=1in,height=1.25in,clip,keepaspectratio]{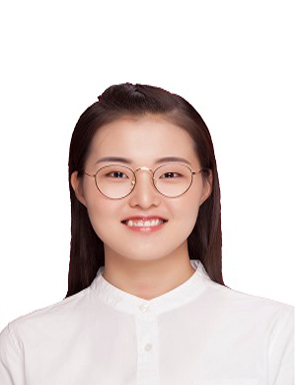}}]{Jieting Wang}
	received the Ph.D. degree in computers with applications from Shanxi University, Taiyuan, China, in 2021. She is currently a teacher at the Institute of Big Data Science and Industry and the School of Artificial Intelligence, Shanxi University. Her research focuses on statistical machine learning and ensemble learning. She has published more than 20 papers in her research fields, including journals such as Artificial Intelligence and IEEE Transactions on Pattern Analysis and Machine Intelligence (TPAMI), as well as top international conferences including ICML, ICLR, AAAI, IJCAI, and KDD.
\end{IEEEbiography}

\begin{IEEEbiography}[{\includegraphics[width=1in,height=1.25in,clip,keepaspectratio]{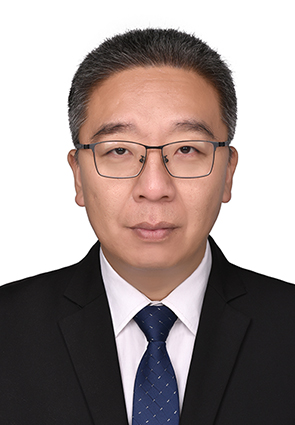}}]{Yuhua Qian} received the M.S. and Ph.D. degrees in computers with applications from Shanxi University, Taiyuan, China, in 2005 and 2011, respectively. He is the Director of the Institute of Big Data Science and Industry and Dean of the School of Artificial Intelligence at Shanxi University. He is also a Director of the Key Laboratory of Evolutionary Science of Shanxi Province. He is best known for artificial intelligence, machine learning theory, and computer vision. He has published more than 300 articles in his research fields, including the journal Artificial Intelligence, the Journal of Machine Learning Research, IEEE Transactions on Pattern Analysis and Machine Intelligence, ICML, NeurIPS, and ICLR.
\end{IEEEbiography}

\end{document}


\clearpage

\onecolumn
\centerline{{\Large \textbf{Supplementary Material for ``ProML: Progressive Multi-label Learning'' }}} 

\setcounter{section}{0}
\renewcommand{\thesection}{\Roman{section}}

\makeatletter
\@addtoreset{theorem}{section}
\makeatother

\setcounter{figure}{0}
\setcounter{table}{0}
\setcounter{equation}{0}
\setcounter{theorem}{0}
\renewcommand{\thetheorem}{\thesection.\arabic{theorem}}
\renewcommand{\thefigure}{\thesection.\arabic{figure}}
\renewcommand{\thetable}{\thesection.\arabic{table}}
\renewcommand{\theequation}{\thesection.\arabic{equation}}

\section{Preliminaries}
\noindent To make the subsequent analysis self-contained, we recall several standard tools and results from learning theory.

\begin{definition}[Rademacher complexity \cite{mohri2018foundations}]\label{preli_rademac}
Let $\mathcal{G}$ be a family of real-valued functions mapping from $\mathbf{X}$ to $\mathbb{R}$, and let $S=\{\boldsymbol{x_1},\boldsymbol{x_2},\cdots,\boldsymbol{x_n}\}\in \mathbf{X}^n$ is a fixed sample of size $n$. Then, the empirical Rademacher complexity of $\mathcal{G}$ with respect to the sample $S$ is defined as
\begin{equation*}
	\hat{\boldsymbol{R}}_S(\mathcal{G}) = \mathbb{E}_{\sigma } \left [ \sup_{g\in \mathcal{G}}\frac{1}{n}\sum_{i=1}^{n}\sigma_{i}g(\boldsymbol{x_i}) \right ],
\end{equation*}
where $\mathbf{\sigma}\!=\!\{\sigma_{1},\sigma_{2},\!\cdots\!,\sigma_{n}\}$ are Rademacher variables with $\sigma_{i}$ independent uniform random variables taking value in $\{-1,+1\}$.
\end{definition}

\begin{definition}[$\beta$-Smooth \cite{hardt2016train}]
    A function $f:\Omega\to \mathbb R$ is $\beta$-smooth if for all $u,v\in \Omega$, we have
\begin{equation*}
    \|\nabla f(u) - \nabla f(v)\|\le \beta\|u-v\|.
\end{equation*}
\end{definition}

\begin{definition}[$\epsilon_{\text{stab}}$-Uniform Stability \cite{bousquet2002stability}]\label{preli_uniform}
    A stochastic algorithm $A$ has uniform stability $\epsilon_{\text{stab}}$ \textit{w.r.t.} the loss function $\ell$ if for all datasets $D, D'$ that differ in at most one example, we have
    \begin{equation*}
        \sup_z \mathbb{E}_A [\ell(A(D),z) - \ell(A(D'),z)] \le \epsilon_{\text{stab}}.
    \end{equation*}
\end{definition}

\begin{lemma}[Corollary 1 in \cite{maurer2016vector}]\label{lemma:contro}
    Let $X$ be any set, let members of $\mathcal{F}$ take values in $\mathbb{R}^c$ with component functions $f_k(\cdot)$. Assume loss function $\ell$ is $\mu$-Lipschitz function from $\mathbb{R}^c$ to $\mathbb{R}$. Then, the following inequality holds \label{lem:VCI}
\begin{equation*}
	\mathbb{E}\left[\sup _{f\in \mathcal{F}} \sum_{i=1}^{n} \sigma_{i} \ell(f(x_i))\right ] \leq \sqrt{2} \mu \mathbb{E}\left[\sup _{f\in \mathcal{F}} \sum_{i=1}^{n} \sum_{j=1}^{c} \sigma_{ij}f_j(x_i)\right],
\end{equation*}
where $[\sigma_{ij}]_{n\times c}$ is an $n\times c$ matrix of Rademacher variables, with $\sigma_{i,j}$ independent uniform random variables taking values in $\{-1,+1\}$ and $f_j(x_i)$ represents the $j$-th component of $f(x_i)$.
\end{lemma}

\begin{lemma}[Theorem 3.3 in \cite{mohri2018foundations}]\label{lemma:gengen}
    Let $\mathcal{G}$ be a family of functions from $Z$ to $[0,M]$. Then, for any $\delta>0$, with a probability at least $1-\delta$ over the draw of an \textit{i.i.d.} sample set $D$ of size $n$, the following generalization bound hold for all $g\in \mathcal{G}$:
\begin{equation*}
	\mathbb{E}[g(z)] \leq \frac{1}{n}\sum_{i=1}^n g(z_i) + 2\hat{R}_S(\mathcal{G}) + 3M\sqrt{\frac{\log\frac{2}{\delta}}{2n}}.
\end{equation*}
\end{lemma}

\begin{lemma}[Lemma 3.3 in \cite{hardt2016train}]\label{lem:bounded}
Assume that $f$ is $L$-Lipschitz. Then, the gradient update $G_{f,\alpha}(w)\!=\!w\!-\!\alpha\nabla f(w)$ is $(\alpha L)$-bounded.
\end{lemma}

\begin{lemma}[Lemma 3.7 in \cite{hardt2016train}]\label{lem:smooth_expansive}
Assume that $f$ is $\beta$-smooth, and further assume that $f$ is convex. Then, for any $\alpha \leq 2/\beta$, the gradient update $G_{f,\alpha}$ is $1$-expansive.
\end{lemma}

\begin{lemma}[Lemma 2.5 in \cite{hardt2016train}] \label{lem:growth}
Fix an arbitrary sequence of updates $G_1, \ldots, G_T$ and another sequence $G'_1, \ldots, G'_T$. 
Let $w_0 = w'_0$ be a starting point in $\Omega$ and define $\delta_t = \|w'_t - w_t\|$ where $w_t, w'_t$ 
are defined recursively through
\begin{equation*}
    w_{t+1} = G_t(w_t), \qquad w'_{t+1} = G'_t(w'_t).
\end{equation*}
Then, we have the recurrence relation $\delta_0 = 0,$
\begin{equation*}
\delta_{t+1} \leq 
\begin{cases}
\eta \delta_t, & \text{if } G_t = G'_t \text{ is } \eta\text{-expansive},\\[4pt]
\min(\eta, 1)\delta_t + 2\sigma_t, & \text{if } G_t \text{ and } G'_t \text{ are } \sigma\text{-bounded},\\[4pt]
 & \phantom{\text{if }} G_t \text{ is } \eta\text{-expansive}.
\end{cases}
\end{equation*}
\end{lemma}

\begin{lemma}[McDiarmid's Inequality \cite{mohri2018foundations}] \label{lemma:McDia}
    Let $X_1\cdots,X_m\in \mathcal{X}^m$ be a set of $m\ge 1$ independent random variables and assume that there exist $c_1,\cdots,c_m>0$ such that $f:\mathcal{X}^m$ satisfies the following conditions: 
    \begin{equation*}
        |f(x_1, \ldots, x_i, \ldots, x_m) - f(x_1, \ldots, x'_i, \ldots, x_m)| \leq c_i,
    \end{equation*}
    for all $i\in [m]$ and any points $x_1,\cdots, x_m, x'_i\in \mathcal{X}$. Let $f(S)$ denote $f(X_1,\cdots, X_m)$, then, for all $\epsilon > 0$, the following inequalities hold:
    \begin{align*}
        \mathbb{P}[f(S) - \mathbb{E}[f(S)] \geq \epsilon] \leq \exp\left(-\frac{2\epsilon^2}{\sum_{i=1}^m c_i^2}\right),\\
        \mathbb{P}[f(S) - \mathbb{E}[f(S)] \leq -\epsilon] \leq \exp\left(-\frac{2\epsilon^2}{\sum_{i=1}^m c_i^2}\right).
    \end{align*}
\end{lemma}

\section{Proof of Proposition 1: Profiled Equivalence of ProML}\label{appx:prop_1}
\noindent\textbf{Proposition 1.} \textit{For any fixed parameter matrix $\boldsymbol{W}$, the auxiliary prediction variable $\mathbf{P}$ admits a unique closed-form solution
\begin{equation}
    \mathbf{P}^{\ast}(\boldsymbol{W}) \!=\!\frac{\mathbf{Y}+\mu \mathbf{X}\boldsymbol{W}}{1+\mu}, 
\end{equation}
where $\mu>0$ is the tradeoff parameter. Consequently, profiling out $\mathbf{P}$ from the original objective in Eq.(3) yields an equivalent reduced optimization problem over $\boldsymbol{W}$:
\begin{equation}
    \min_{\boldsymbol{W}} \frac{s}{2}\|\mathbf{X}\boldsymbol{W}-\mathbf{Y}\|_F^2+\Omega(\boldsymbol{W};\boldsymbol{C}),
\end{equation}
where $\|\cdot\|_F$ denotes the Frobenius norm, and the supervised fitting term is scaled by $s = \mu / (1 + \mu)\in (0,1)$.}

\begin{proof}[\textbf{Proof}]
    For a fixed $\boldsymbol{W}$, the optimization problem \textit{w.r.t.} $\mathbf{P}$ is 
\begin{equation}
    \mathcal{L}(\mathbf{P}) = \frac{1}{2}\|\mathbf{P-Y}\|_F^2 + \frac{\mu}{2}\|\mathbf{X}\boldsymbol{W}-\mathbf{P}\|_F^2,
    \label{supp_lp}
\end{equation}

Taking the first-order derivative \textit{w.r.t.} $\mathbf{P}$ and setting the derivative to zero:  
\begin{equation}
    (\mathbf{P-Y})+\mu(\mathbf{P}-\mathbf{X}\boldsymbol{W})=0.
\end{equation}

Since $1+\mu > 0$ for any $\mu > 0$, we can solve for $P$:
\begin{equation}
    \mathbf{P}^\ast(\boldsymbol{W})=\frac{\mathbf{Y}+\mu \mathbf{X}\boldsymbol{W}}{1+\mu}.
    \label{supp_p8}
\end{equation}

To verify the well-posedness of this subproblem, we examine its second-order derivative behavior. The Hessian of $\mathcal{L}(\mathbf{P})$ \textit{w.r.t.} $\mathbf{P}$ is given by 
\begin{equation}
    \nabla^2\mathcal{L}(\mathbf{P}) = (1+\mu)I.
\end{equation}

Since $\mu>0$, we have $1+\mu>0$, so the Hessian is strictly positive definite. Hence, $\mathcal{L}(\mathbf{P})$ is strictly convex in $\mathbf{P}$, guaranteeing the existence and uniqueness of the global minimizer. 

This closed-form solution reveals that the optimal target is a convex combination of the hard labels $\mathbf{Y}$ and the model's current predictions $\mathbf{X}\boldsymbol{W}$, with the weight controlled by $\mu$.

Substituting the optimal surrogate Eq.\eqref{supp_p8} into Eq.\eqref{supp_lp}, we obtain the following reduction:
\begin{equation}
    \frac{1}{2}\|\mathbf{P}^\ast-\mathbf{Y}\|_F^2 + \frac{\mu}{2}\|\mathbf{X}\boldsymbol{W}-\mathbf{P}^\ast\|_F^2 = \frac{\mu}{2(1+\mu)}\|\mathbf{X}\boldsymbol{W}-\mathbf{Y}\|_F^2.
\end{equation}

Consequently, the original ProML optimization problem, after the elimination of $\mathbf{P}$, is equivalently reformulated as the following profiled objective over $\boldsymbol{W}$:
\begin{equation}
    \min_{\boldsymbol{W}} \frac{s}{2}\|\mathbf{X}\boldsymbol{W}-\mathbf{Y}\|_F^2+\Omega(\boldsymbol{W};\boldsymbol{C}),
\end{equation}
where the effective scaling factor $s=\frac{\mu}{1+\mu}$. Since $\mu>0$, it follows that $0<s<1$. 

This completes the proof.
\end{proof}

\section{Proof of Corollary 2: Limiting case of ProML}\label{appx:coro_2}
\noindent\textbf{Corollary 2.} \textit{As the trade-off parameter $\mu\to \infty$, the profiled ProML objective degenerates to the standard GMLL formulation in Eq.(2).}

\begin{proof}[\textbf{Proof}]
From Proposition 1, the ProML is equivalently reformulated as the following profiled objective: 
\begin{equation}
    \mathcal{L}_{P}(\boldsymbol{W}) = \frac{s}{2}\|\mathbf{X}\boldsymbol{W}-\mathbf{Y}\|_F^2+\Omega(\boldsymbol{W};\boldsymbol{C}),\ \text{where} \ s=\frac{\mu}{1+\mu}.
\end{equation}

Then, taking the limit of the scaling factor $s$ when $\mu\to \infty$, we have 
\begin{equation}
    \lim_{ \mu \to \infty } s=\lim_{ \mu \to \infty }  \frac{\mu}{1+\mu} = 1.
\end{equation}

Therefore, for any fixed parameter matrix $\boldsymbol{W}$, the profiled ProML objective converges to:
\begin{equation}
    \lim_{ \mu \to \infty }\mathcal{L}_{P}(\boldsymbol{W}) = \frac{1}{2}\|\mathbf{X}\boldsymbol{W}-\mathbf{Y}\|_F^2+\Omega(\boldsymbol{W};\boldsymbol{C}),
\end{equation}
which is exactly the GMLL formulation in Eq.(2). 

This completes the proof.
\end{proof}

\section{Proof of Proposition 3: Smoothness Compression}\label{appx:prop_3}
\noindent\textbf{Proposition 3.} \textit{Let $\mathbf{H}_{P}$ and $\mathbf{H}_{G}$ denote the Hessian matrices of the ProML and GMLL objective, respectively. For any $\mu>0$, and the quadratic structural regularization, \textit{i.e.}, $\Omega(\boldsymbol{W};\boldsymbol{C})=\Tr(\boldsymbol{WCW}^{\top})$, with $\boldsymbol{C}\succeq 0$. Then, we have:
\begin{equation}
    \mathbf{H}_{P}\preceq \mathbf{H}_{G},
\end{equation}
and consequently, the dominant eigenvalues satisfy:
\begin{equation}
    \lambda_{\max}(\mathbf{H}_{P}) \le\lambda_{\max}(\mathbf{H}_{G}).
\end{equation} }

\begin{proof}[\textbf{Proof}]
Let $s=\frac{\mu}{1+\mu}\in (0,1)$. Under the quadratic structural regularization, the GMLL and profiled ProML objectives are given by
\begin{equation}
    \mathcal{L}(\boldsymbol{W}) = \frac{s}{2} \|\mathbf{X}\boldsymbol{W}-\mathbf{Y}\|_F^2 + \frac{\lambda}{2} \Tr(\boldsymbol{WCW^{\top}}), 
\end{equation}
where if $s=1$, $\mathcal{L}(\boldsymbol{W})$ denotes GMLL, otherwise, it denotes ProML.

Taking the first-order derivative \textit{w.r.t.} $\boldsymbol{W}$, we have
\begin{equation}
    \begin{aligned}
        \nabla_{\boldsymbol{W}} \mathcal{L}_{G}(\boldsymbol{W}) &=  \mathbf{X}^{\top} \mathbf{X}\boldsymbol{W} - \mathbf{X}^{\top}\mathbf{Y} + \lambda \boldsymbol{WC},\\
        \nabla_{\boldsymbol{W}} \mathcal{L}_{P}(\boldsymbol{W}) &=  s\mathbf{X}^{\top} \mathbf{X}\boldsymbol{W} - s\mathbf{X}^{\top}\mathbf{Y} + \lambda \boldsymbol{WC}.
    \end{aligned}
\end{equation}

Next, we vectorize these matrix gradients. Let $w=\text{vec}(\boldsymbol{W})$, $y=\text{vec}(\mathbf{Y})$, where $\text{vec}(\cdot)$ stacks the columns of a matrix. Using the identity of Kronecker product (\textit{i.e.}, $\text{vec}(\mathbf{AX})=(\mathbf{I}\otimes \mathbf{A})\text{vec}(\mathbf{X})$, $\text{vec}(\mathbf{XC})=(\mathbf{C}^{\top}\otimes \mathbf{I})\text{vec}(\mathbf{X})$), the corresponding Hessian matrices are:
\begin{equation}
    \begin{aligned}
    \mathbf{H}_{G} &= I_c\otimes \mathbf{X}^{\top} \mathbf{X}  +\lambda \boldsymbol{C}^{\top} \otimes I_d,\\
\mathbf{H}_{P} &= sI_c\otimes \mathbf{X}^{\top} \mathbf{X}  +\lambda \boldsymbol{C}^{\top} \otimes I_d.
\end{aligned}
\end{equation}

Since $\mathbf{X}^{\top}\mathbf{X}\succeq{0}$, its Kronecker product with the identity matrix, $\mathbf{A} := I_c\otimes \mathbf{X}^{\top} \mathbf{X}$, also satisfies $\mathbf{A}\succeq 0$. Given that $0\le s \le 1$, we have
\begin{equation}
    \mathbf{A}  - s \mathbf{A}  = (1-s) \mathbf{A}  \succeq 0.
\end{equation}

Thus,  
\begin{equation}
    \mathbf{H}_{G}-\mathbf{H}_{P} = (1-s)\mathbf{A}\succeq 0
\end{equation}

Therefore, $\mathbf{H}_{P}\preceq \mathbf{H}_{G}$. By eigenvalue monotonicity for symmetric matrices, it follows that 
\begin{equation}
    \lambda_{\max}(\mathbf{H}_{P}) \le\lambda_{\max}(\mathbf{H}_{G}).
\end{equation}

This completes the proof.
\end{proof}

\section{Proof of Corollary 4: Condition number compression}\label{appx:coro_4}
\noindent\textbf{Corollary 4.} \textit{Let $\kappa(M)=\frac{\lambda_{\max}(M)}{\lambda_{\min}(M)}$ denote the condition number of a symmetric positive definite matrix. Under the same quadratic structural regularization setting as in Proposition 3 with $\boldsymbol{C}\succ 0$, assume that the feature space is strictly more ill-conditioned than the label correlation space in the following spectral sense:
\begin{equation}
    \frac{\lambda_{\max}(\mathbf{X}^{\top}\mathbf{X})}{\lambda_{\min}(\mathbf{X}^{\top}\mathbf{X})}>\frac{\lambda_{\max}(\boldsymbol{C})}{\lambda_{\min}(\boldsymbol{C})}.
\end{equation}
Then,  for any $\mu>0$, the condition number of the Hessian matrix $\mathbf{H}_{P}$ is strictly smaller than $\mathbf{H}_{G}$ :
\begin{equation}
    \kappa(\mathbf{H}_{P})<\kappa(\mathbf{H}_{G}).
\end{equation}	}

\begin{proof}[\textbf{Proof}]
From Proposition 3, the eigenvalues of the two Hessian matrices can be explicitly expressed in the form of Kronecker sums. For GMLL, the extreme eigenvalues are:
\begin{equation}
    \lambda_{\max} (\mathbf{H}_{G}) = \lambda_{\max}(\mathbf{X}^{\top}\mathbf{X}) + \lambda \lambda_{\max}(\boldsymbol{C}),\ \lambda_{\min} (\mathbf{H}_{G}) = \lambda_{\min}(\mathbf{X}^{\top}\mathbf{X}) + \lambda \lambda_{\min}(\boldsymbol{C}).
\end{equation}

For ProML, the extreme eigenvalues are:
\begin{equation}
    \lambda_{\max} (\mathbf{H}_{P}) = s\lambda_{\max}(\mathbf{X}^{\top}\mathbf{X}) + \lambda \lambda_{\max}(\boldsymbol{C}),\ \lambda_{\min} (\mathbf{H}_{P}) = s\lambda_{\min}(\mathbf{X}^{\top}\mathbf{X}) + \lambda \lambda_{\min}(\boldsymbol{C}).
\end{equation}

Let $A=\lambda_{\max}(\mathbf{X}^{\top}\mathbf{X}), B = \lambda\lambda_{\max}(\mathbf{X}^{\top}\mathbf{X}), a=\lambda_{\min}(\mathbf{X}^{\top}\mathbf{X}), b=\lambda\lambda_{\min}(\mathbf{X}^{\top}\mathbf{X})$. Consider the condition number as a continuous real-valued function of $s$:
\begin{equation}
    f(s) = \frac{s A + B}{s a + b},\ s\in(0,1].
\end{equation}

Evidently, $f(1) = \kappa(\mathcal{H}_{G})$ and $f(s) = \kappa(\mathcal{H}_{P})$. To determine the monotonicity of $f(s)$ on $(0,1]$, we compute its first-order derivative \textit{w.r.t.} $s$: 
\begin{equation}
    f'(s) =  \frac{Ab - aB}{(s a + b)^{2}}.
\end{equation}

Since the denominator is always positive, the sign of the derivative depends on that of the numerator :
\begin{equation}
    Ab - aB = \lambda\Big[\lambda_{\max}(\mathbf{X}^{\top}\mathbf{X})\,\lambda_{\min}(\boldsymbol{C}) - \lambda_{\min}(\mathbf{X}^{\top}\mathbf{X})\,\lambda_{\max}(\boldsymbol{C})\Big].
\end{equation}

The spectral ratio condition is precisely equivalent to:
\begin{equation}
    \frac{\lambda_{\max}(\mathbf{X}^{\top}\mathbf{X})}{\lambda_{\min}(\mathbf{X}^{\top}\mathbf{X})} > \frac{\lambda_{\max}(\boldsymbol{C})}{\lambda_{\min}(\boldsymbol{C})} \Longleftrightarrow  \lambda_{\max}(\mathbf{X}^{\top}\mathbf{X})\,\lambda_{\min}(\boldsymbol{C}) - \lambda_{\min}(\mathbf{X}^{\top}\mathbf{X})\,\lambda_{\max}(\boldsymbol{C}) > 0.
\end{equation}

Hence, $f'(s)>0$ holds for all $s\in (0,1]$, establishing that $f(s)$ is strictly monotonically increasing on $(0,1]$.

Since $s=\frac{\mu}{1+\mu}\in (0,1)$, strict monotonicity directly yields:
\begin{equation}
    f(s)<f(1) \Longrightarrow \kappa(\mathbf{H}_{P}) < \kappa(\mathbf{H}_{G}).
\end{equation}

This completes the proof.
\end{proof}

\section{Proof of Corollary 5: Step-size Broadening}\label{appx:coro_5}
\noindent\textbf{Corollary 5.} \textit{Let $\eta_{crit}=\frac{2}{\lambda_{\max}(\mathbf{H})}$ denote the maximum critical learning rate. For any $\mu> 0$, and under the same quadratic structural regularization setting as in Proposition 3 with $\boldsymbol{C}\succeq 0$, ProML achieves a  broadening of the secure learning rate interval for GMLL. Then, we have 
\begin{equation}
    \eta_{\text{crit}}^{G} \le \eta_{\text{crit}}^{P} \le \left(\frac{1+\mu}{\mu}\right) \eta_{\text{crit}}^{G}.
\end{equation}
}

\begin{proof}[\textbf{Proof}]
    The smoothness constant of a quadratic objective is given by the largest eigenvalue of its Hessian. Therefore, let $L_{G}:=\lambda_{\max}(\mathbf{H}_{G}), L_{P}:=\lambda_{\max}(\mathbf{H}_{P})$. From Proposition 3, we have $L_{P}\le L_{G}$. For gradient descent applied to a smooth convex quadratic objective, a standard stable step-size condition is $0<\eta < \frac{2}{L}$. Hence, the admissible stable step-size intervals for GMLL and ProML are respectively:
\begin{equation}
    0<\eta < \frac{2}{L_{G}}, \ 0<\eta < \frac{2}{L_{P}}.
\end{equation}

Since $L_{P}\le L_{G}$, it follows that $\frac{2}{L_{P}}\ge \frac{2}{L_{G}}$. 
Further, from the explicit Hessian expressions of GMLL and ProML in Proposition 3, and subtracting the two expressions yields:
\begin{equation}
    \mathbf{H}_{P}-s\mathbf{H}_{G}=(1-s)\lambda(\boldsymbol{C}\otimes I)\succeq 0,
\end{equation}
where the positive semi-definiteness follows from $\boldsymbol{C}\succeq 0$ and $1-s>0$. Thus, 
\begin{equation}
    \lambda_{\max}(\mathbf{H}_{P})\ge s\cdot \lambda_{\max} (\mathbf{H}_{G}).
\end{equation}

Combing $L_{P}\le L_{G}$ and $L_{P}\ge sL_{G}$, taking reciprocals and multiplying by 2 yields:
\begin{equation}
    \frac{2}{L_{G}}\le \frac{2}{L_{P}} \le \frac{2}{sL_{G}} = \left( \frac{1+\mu}{\mu} \right) \frac{2}{L_{G}}.
\end{equation}

According to the definition of the maximum critical learning rate, the above inequality is equivalent to:
\begin{equation}
    \eta_{\text{crit}}^{G} \le \eta_{\text{crit}}^{P} \le \left(\frac{1+\mu}{\mu}\right) \eta_{\text{crit}}^{G}.
\end{equation}

This completes the proof.
\end{proof}

\section{Proof of Proposition 6: Effective Hypothesis Class Contraction}\label{appx:prop_6} 
\noindent We first introduce the following lemma:
\begin{lemma}\label{rad:base}
    Let $\mathcal{F}_{B}=\{f_{\boldsymbol{W}}: \mathbb{R}^{d}\to \mathbb{R}^{c} \mid f_{\boldsymbol{W}}(\mathbf{X}) = \mathbf{X}\boldsymbol{W},\  \|\boldsymbol{W}\|_{F}\le B\}$ be a class of linear functions defined on the Frobenius norm ball. For any dataset $S$ of size $n$, its empirical Rademacher complexity is
    \begin{equation}
        \hat{\mathfrak{R}}_{S}(\mathcal{F}_{B}) = \mathbb{E}_{\mathbf{\sigma}}  \left[  \frac{1}{n}  \sup_{\|\boldsymbol{W}\|_{F}\le B}  \sum_{i=1}^n  \sum_{j=1}^c  \sigma_{ij} (\mathbf{X}\boldsymbol{W})_{ij} \right],
    \end{equation}
    where $\{\sigma_{ij}\}$ are independent uniform Rademacher random variables. Then, $\hat{\mathfrak{R}}_{S}(\mathcal{F}_{B})$ has the following exact expression:
    \begin{equation}
        \hat{\mathfrak{R}}_{S}(\mathcal{F}_{B}) \le \frac{B}{n} \mathbb{E}_{\mathbf{\sigma}}  \|\mathbf{X}^{\top} \mathbf{\Sigma}\|_F,
    \end{equation}
    where the $j$-th column of $\mathbf{\Sigma}\in \mathbb{R}^{n\times c}$ is defined as $\sum_{i=1}^n \sigma_{ij}\mathbf{x}_i$.
\end{lemma}

\noindent\textbf{Proposition 6.} \textit{Let $\mathcal{F}_{G}=\{\boldsymbol{W}:\|\boldsymbol{W}\|_{F}\le B_{1}\}$ and $\mathcal{F}_{P}=\{\boldsymbol{W}:\|\boldsymbol{W}\|_{F}\le B_{s}\}$ denote the empirical hypothesis space of GMLL and ProML, respectively, and let $S$ be a sample of size $n$. Then, for any $\mu>0$ and $\boldsymbol{C}\succ 0$, the empirical Rademacher complexity on $S$ satisfies
\begin{equation}
    \hat{\mathfrak{R}}_{S}(\mathcal{F}_{P}) \le \frac{\mu}{1+\mu} \hat{\mathfrak{R}}_{S}(\mathcal{F}_{G}).
\end{equation}
}

\begin{proof}[\textbf{Proof}]
From Proposition 1, the objective is 
\begin{equation}
    \mathcal{L}(\boldsymbol{W}) = \frac{s}{2} \|\mathbf{X}\boldsymbol{W} - \mathbf{Y}\|_F^2 + \frac{\lambda}{2} \text{tr}(\boldsymbol{WCW}^{\top}).
\end{equation}

Specifically, when $s = 1$, it corresponds to GMLL objective, and when $s = \frac{\mu}{1+\mu}$, it corresponds to ProML.
Taking the derivative \textit{w.r.t.} $\boldsymbol{W}$ and setting it to zero: 
\begin{equation}
    \nabla_{\boldsymbol{W}} \mathcal{L}(\boldsymbol{W}) = s \mathbf{X}^{\top}(\mathbf{X}\boldsymbol{W} - \mathbf{Y}) + \lambda \boldsymbol{WC}  = 0.
\end{equation}

Dividing both sides by $s$, we obtain:
\begin{equation}
    \mathbf{X}^{\top}\mathbf{X}\boldsymbol{W} + \frac{\lambda}{s} \boldsymbol{WC} = \mathbf{X}^{\top}\mathbf{Y}.
\end{equation}

Taking the Frobenius inner product of both sides of the above equation with $\boldsymbol{W}$:
\begin{equation}
    \langle \boldsymbol{W}, \mathbf{X}^{\top}\mathbf{X}\boldsymbol{W} \rangle_F + \frac{\lambda}{s} \langle \boldsymbol{W}, \boldsymbol{WC} \rangle_F = \langle \boldsymbol{W}, \mathbf{X}^{\top}\mathbf{Y} \rangle_F.
\end{equation}

According to the cyclic property of trace,  
\begin{equation}
\langle \boldsymbol{W}, \mathbf{X}^{\top}\mathbf{X}\boldsymbol{W} \rangle_F = \|\mathbf{X}\boldsymbol{W}\|_F^2, \  \text{and} \ \Tr(\boldsymbol{W}^{\top}\boldsymbol{WC})=\Tr(\boldsymbol{WC}\boldsymbol{W}^{\top}).
\end{equation}

It follows that
\begin{equation}
    \|\mathbf{X}\boldsymbol{W}\|_F^2 + \frac{\lambda}{s} \Tr(\boldsymbol{WCW}^{\top}) = \langle \boldsymbol{W}, \mathbf{X}^{\top}\mathbf{Y} \rangle_F. \label{eq:12}
\end{equation}

Since $\|\mathbf{X}\boldsymbol{W}\|_{F}^2\ge 0$, inequality (\ref{eq:13}) follows directly from (\ref{eq:12}):
\begin{equation}
    \frac{\lambda}{s} \Tr(\boldsymbol{WCW}^{\top}) \le \langle \boldsymbol{W}, \mathbf{X}^{\top}\mathbf{Y} \rangle_F. \label{eq:13}
\end{equation}

For each nonzero row vector $\mathbf{w}_i$, based on the property of minimizing the Rayleigh quotient, 
\begin{equation}
    R_{\boldsymbol{C}}(\mathbf{w}_i):= \frac{\mathbf{w}_i^{\top}\boldsymbol{C}\mathbf{w}_i}{\mathbf{w}_i^{\top}\mathbf{w}_i}\ge \lambda_{\min}(\boldsymbol{C}).
\end{equation} 

Multiplying both sides by $\mathbf{w}_i^{\top}\mathbf{w}_i=\|\mathbf{w}_i\|_2^2$, this implies
\begin{equation}
    \mathbf{w}_i^{\top}\boldsymbol{C}\mathbf{w}_i \ge \lambda_{\min}(\boldsymbol{C}) \|\mathbf{w}_i\|_2^2. 
\end{equation}

When $\mathbf{w}_i = 0$, the above equation clearly holds as well. Therefore, summing over all rows:
\begin{equation}
    \Tr(\boldsymbol{WCW}^{\top}) = \sum_{i=1}^{d} \mathbf{w}_i^{\top} \boldsymbol{C} \mathbf{w}_i \ge  \lambda_{\min}(\boldsymbol{C}) \sum_{i=1}^{d}\|\mathbf{w}_i\|_2^2 = \lambda_{\min}(\boldsymbol{C}) \|\boldsymbol{W}\|_F^2.
\end{equation}

Moreover, the Cauchy-Schwarz inequality for the Frobenius inner product gives
\begin{equation}
    \langle \boldsymbol{W}, \mathbf{X}^{\top}\mathbf{Y} \rangle_F \le \|\boldsymbol{W}\|_F \|\mathbf{X}^{\top}\mathbf{Y} \|_F
\end{equation}

Combining these with inequality (\ref{eq:13}), we have
\begin{equation}
    \frac{\lambda}{s}\lambda_{\min}(\boldsymbol{C}) \|\boldsymbol{W}\|_F^2 \le \frac{\lambda}{s}\Tr(\boldsymbol{WCW}^{\top}) \le  \langle \boldsymbol{W}, \mathbf{X}^{\top}\mathbf{Y} \rangle_F \le \|\boldsymbol{W}\|_F \|\mathbf{X}^{\top}\mathbf{Y} \|_F.
    \label{ap:lsff}
\end{equation}

If $\boldsymbol{W}\ne 0$, the desired norm bound holds trivially. Otherwise, dividing both sides of Eq.\eqref{ap:lsff} by $\|\boldsymbol{W}\|_F>0$ gives
\begin{equation}
    \|\boldsymbol{W}\|_F \le \frac{s \|\mathbf{X}^{\top}\mathbf{Y}\|_F}{\lambda \lambda_{\min}(\boldsymbol{C})} =:  B_s.
\end{equation}

When $s=1$, the objective reduces to the GMLL objective, and its minimizer satisfies
\begin{equation}
    \|\boldsymbol{W}\|_F \le  \frac{\|\mathbf{X}^{\top}\mathbf{Y}\|_F}{\lambda \lambda_{\min}(\boldsymbol{C})} =:  B_1.
\end{equation}

Thus, the ProML solution $\boldsymbol{W}$ is confined to a Frobenius norm ball of radius $B_{s}=s\cdot B_{1}$, where $B_{1}={\|\mathbf{X}^{\top}\mathbf{Y}\|_{F}}/{\lambda \lambda_{\min}(\boldsymbol{C})}$ corresponds to the GMLL norm bound. Therefore, $\mathcal{F}_{P}\subseteq \mathcal{F}_{G}$.

The Frobenius norm is self-dual, \textit{i.e.}, $\sup_{\|\boldsymbol{W}\|_F\le B}\langle \boldsymbol{W},\mathbf{X}^{\top} \mathbf{\Sigma} \rangle_F =  B \|\mathbf{X}^{\top} \mathbf{\Sigma}\|_F$. Thus, applying Lemma \ref{rad:base} to $\mathcal{F}_{G}$ and $\mathcal{F}_{P}$, and using $B_s = sB_1$, yields
\begin{equation}
    \hat{\mathfrak{R}}_{S}(\mathcal{F}_{P}) \le \frac{B_{s}}{n} \mathbb{E}_{\sigma} [ \|\mathbf{X}^{\top} \mathbf{\Sigma}\|_F ] = s\frac{B_{1}}{n} \mathbb{E}_{\sigma} [ \|\mathbf{X}^{\top} \mathbf{\Sigma}\|_F ] = s\hat{\mathfrak{R}}_{S}(\mathcal{F}_{G}),
\end{equation}
where $\mathbf{\Sigma}=\sum_{i=1}^n\sigma_{ij} x_{i}$. 

Finally, substituting $s=\frac{\mu}{1+\mu}\in (0, 1)$, it following that
\begin{equation}
    \hat{\mathfrak{R}}_{S}(\mathcal{F}_{P}) \le \hat{\mathfrak{R}}_{S}(\mathcal{F}_{G}) = \frac{\mu}{1+\mu} \hat{\mathfrak{R}}_{S}(\mathcal{F}_{G}) .
\end{equation}

This completes the proof.
\end{proof}

\section{Proof of Proposition 7: Convergence Property}\label{apx:conv}
\noindent We first proof the following lemmas:

\begin{lemma}\label{lem:conv_1}
Let $p\in (0,1]$ and define 
\begin{equation}
    \varphi(t) = \frac{2}{2-p}t - \frac{p}{2-p}t^{\frac{2}{p}} - 1. 
\end{equation}
Then, for any $t>0$, $\varphi(t) \leq 0$.
\end{lemma}

\begin{proof}[\textbf{Proof}]
Taking derivative of $\varphi(t)$ \textit{w.r.t.} $t$, we have
\begin{equation}
	\varphi ' (t) = \frac{2}{2-p} (1-t^{\frac{2}{p}-1}).
\end{equation}

Setting $\varphi'(t) = 0$ gives the unique stationary point $t = 1$ on $(0, +\infty)$. It is straightforward to verify that $\varphi(t)$ attains its global maximum at $t=1$. Therefore,
\begin{equation}
	\varphi(t) \leq \varphi(1) = 0,\quad t>0.
\end{equation}

The completes the proof.
\end{proof}

\begin{lemma}\label{the:con_lem_2}
    Let $\{\boldsymbol{W}_{t}\}_{t \ge 0}$ be the sequence of optimization processes generated by Algorithm 2, and let $\mathbf{w}_t^i$ denote the $i$-th row of $\boldsymbol{W}_t$. For a fixed iteration $t$, $p=1/2$, assume each nonzero row vector $\mathbf{w}_t^i$ for all $i = 1, \dots, d$. Then, for each $i$, the following row-wise inequality holds:
\begin{equation}
\begin{aligned}
\sum_{i=1}^{d}\left(\|w_{t+1}^{i}\|_{2}^{1/2} - \frac{1}{4} \frac{\| w_{t+1}^{i}\|_2^2}{\|w_{t}^{i}\|_2^{3/2}}\right) \leq 
	\sum_{i=1}^{d} \left( \|w_{t}^{i} \|_{2}^{1/2} - \frac{1}{4} \frac{\| w_{t}^{i}\|_2^2}{\|w_{t}^{i}\|_2^{3/2}} \right).
\end{aligned}
\end{equation}
\end{lemma}

\begin{proof}[\textbf{Proof}]
For each fixed $i$ and $p=1/2$, we apply Lemma \ref{lem:conv_1} with $t=\frac{\left\| w_{t+1}^{i} \right\|_2^{1/2}}{\left\|w_{t}^{i} \right \|_2^{1/2}}> 0$. This gives $\varphi(t) \le 0$, which is equivalent to 
\begin{equation}
	4 \frac{\left\| w_{t+1}^{i} \right\|_2^{1/2}}{\left\|w_{t}^{i} \right \|_2^{1/2}} - \frac{\left\| w_{t+1}^{i} \right\|_2^2}{\left\|w_{t}^{i} \right \|_2^{2}} \leq 3.
	\label{equ:var}
\end{equation}

Multiplying the two sides of inequality (\ref{equ:var}) by $\frac{1}{4} {\left\|w_{t}^{i} \right \|_2^{1/2}}$, we have
\begin{equation}
    {\left\| w_{t+1}^{i} \right\|_2^{1/2}} - \frac{1}{4} \frac{\left\| w_{t+1}^{i} \right\|_2^2}{\left\|w_{t}^{i} \right \|_2^{3/2}} \leq \frac{3}{4}{\|w_{t}^{i} \|_2^{1/2}}.
    \label{VIII.47a}
\end{equation}

Noting that
\begin{equation}
     \frac{\left\| w_{t}^{i} \right\|_2^2}{\left\|w_{t}^{i} \right \|_2^{3/2}}  = \| w_{t}^{i} \|_2^{2-{3/2}} = \left\| w_{t}^{i} \right\|_2^{1/2}, 
\end{equation}
and hence the right-hand side of Eq.\eqref{VIII.47a} can be rewritten as
\begin{equation}
    \begin{aligned}
        \frac{3}{4}\|w_{t}^{i}\|_2^{1/2} &= \|w_{t}^{i}\|_2^{1/2} - \frac{1}{4}\|w_{t}^{i}\|_2^{1/2}\\
        &= \|w_{t}^{i}\|_2^{1/2} - \frac{1}{4}\frac{\left\| w_{t}^{i} \right\|_2^2}{\left\|w_{t}^{i} \right \|_2^{3/2}}.
    \end{aligned}
\end{equation}

Then, substituting this identity into Eq.\ref{VIII.47a}, we obtain
\begin{equation}
	 {\left\| w_{t+1}^{i} \right\|_2^{1/2}} - \frac{1}{4} \frac{\left\| w_{t+1}^{i} \right\|_2^2}{\left\|w_{t}^{i} \right \|_2^{3/2}} \leq 
	 {\left\| w_{t}^{i} \right\|_2^{1/2}} - \frac{1}{4} \frac{\left\| w_{t}^{i} \right\|_2^2}{\left\|w_{t}^{i} \right \|_2^{3/2}}.
\end{equation}

Finally, summing all rows ($0 \leq i\leq d$) of the above inequality, we obtain the following conclusion
\begin{flalign}
	\sum_{i=1}^{d}\left(\|w_{t+1}^{i} \|_{2}^{1/2} - \frac{1}{4} \frac{\| w_{t+1}^{i} \|_2^2}{\|w_{t}^{i}  \|_2^{3/2}}\right)
	\leq 
	\sum_{i=1}^{d}\left(\|w_{t}^{i} \|_{2}^{1/2} - \frac{1}{4} \frac{\| w_{t}^{i} \|_2^2}{\|w_{t}^{i}  \|_2^{3/2}} \right ).
\end{flalign}   

The completes the proof.
\end{proof}

\noindent\textbf{Proposition 7.} \textit{Let $\{\boldsymbol{W}_{t}\}_{t \ge 0}$ be the sequence of optimization processes generated by the gradient update rule in Eq.(25). Suppose that for each row $i$ and each iteration $t$, $\boldsymbol{w}_t^i \neq \mathbf{0}$, and that the step size $\eta_t > 0$ is determined by the Lipschitz constant of $\nabla f_{\boldsymbol{D}}$ with respect to the Frobenius norm. Then, the objective function $f_{\boldsymbol{D}}(\boldsymbol{W})$ defined in Eq.(23) is monotonically non-increasing along the iterations, \textit{i.e.},
\begin{equation}
    f_{\boldsymbol{D}}(\boldsymbol{W}_{t+1}) \le f_{\boldsymbol{D}}(\boldsymbol{W}_{t}), \quad \forall t \ge 0.
    \label{eq:cor_mono}
\end{equation}
Consequently, the sequence $\{f_{\boldsymbol{D}}(\boldsymbol{W}_{t})\}$ converges to a finite value, and any limit point of $\{\boldsymbol{W}_{t}\}$ is a stationary point of Eq.(23) satisfying the first-order optimality condition.}

\begin{proof}[\textbf{Proof}]
According to the iterative scheme in Algorithm 2, the variable $\boldsymbol{W}_{t+1}$ is obtained by solving the following convex quadratic subproblem with fixed $\boldsymbol{D}_t$:
\begin{equation}
    \begin{aligned}
	\boldsymbol{W}_{t+1} = \arg \min_{\boldsymbol{W}} \frac{1}{2}\Tr(\boldsymbol{W}\!^{\top}\! \mathbf{A}\! \boldsymbol{W}) \!-\! s\Tr(\boldsymbol{W}\!^{\top}\mathbf{X}\!^{\top} \mathbf{Y})\!+ \frac{\alpha}{2}\Tr(\boldsymbol{W R W\!^{\top}})  + 2{\lambda}\Tr(\boldsymbol{W\!^{\top}\! D_t W}).
	\label{equ:omega}
\end{aligned}
\end{equation}

For simplicity, we assume
\begin{equation}
	\Upsilon  (\boldsymbol{W}, \mathbf{Y}, \boldsymbol{R}) \triangleq  \frac{1}{2}\Tr(\boldsymbol{W}\!^{\top}\! \mathbf{A}\! \boldsymbol{W}) \!-\! s\Tr(\boldsymbol{W}\!^{\top}\mathbf{X}\!^{\top} \mathbf{Y})\!+ \frac{\alpha}{2}\Tr(\boldsymbol{W R W\!^{\top}}).
\end{equation}

Therefore, Eq.(\ref{equ:omega}) can be rewritten as
\begin{equation}
	\boldsymbol{W}_{t+1} = \arg \min_{\boldsymbol{W}} \Upsilon  (\boldsymbol{W}, \mathbf{Y},\boldsymbol{R}) + \Tr(\boldsymbol{W}^{\!\top}\boldsymbol{D}_t\boldsymbol{W}).
    \label{eq:subprob}
\end{equation}

Since $\boldsymbol{W}_{t+1}$ minimizes the subproblem, for $\boldsymbol{W}_t$ we have
\begin{equation}
	\Upsilon(\boldsymbol{W}_{t+1}, \mathbf{Y}, \boldsymbol{R}) + \lambda \Tr(\boldsymbol{W}_{t+1}^{\!\top}\boldsymbol{D}_t\boldsymbol{W}_{t+1})
    \le
    \Upsilon(\boldsymbol{W}_t, \mathbf{Y}, \boldsymbol{R}) + \lambda \Tr(\boldsymbol{W}_t^{\!\top}\boldsymbol{D}_t\boldsymbol{W}_t).
    \label{eq:subopt}
\end{equation}

By the definition of \(\boldsymbol{D}_t\), its diagonal entries are given by
\begin{equation}
    (\boldsymbol{D}_t)_{ii} = \frac{1}{4} \|\boldsymbol{w}_t^i\|_2^{-3/2},
\end{equation}
where \(\boldsymbol{w}_t^i\) denotes the \(i\)-th row of \(\boldsymbol{W}_t\). Consequently,
\begin{align}
    \Tr(\boldsymbol{W}_{t+1}^{\!\top}\boldsymbol{D}_t\boldsymbol{W}_{t+1})
    = \sum_{i=1}^d \frac{1}{4} \frac{\|\boldsymbol{w}_{t+1}^i\|_2^2}{\|\boldsymbol{w}_t^i\|_2^{3/2}}, \qquad\Tr(\boldsymbol{W}_t^{\!\top}\boldsymbol{D}_t\boldsymbol{W}_t)
    = \sum_{i=1}^d \frac{1}{4} \frac{\|\boldsymbol{w}_t^i\|_2^2}{\|\boldsymbol{w}_t^i\|_2^{3/2}} . \label{eq:tr2}
\end{align}

Substituting  Eq.(\ref{eq:tr2}) into Eq.(\ref{eq:subopt}) yields
\begin{equation}
    \Upsilon(\boldsymbol{W}_{t+1}, \mathbf{Y}, \boldsymbol{R})
    + \lambda \sum_{i=1}^d \frac{1}{4} \frac{\|\boldsymbol{w}_{t+1}^i\|_2^2}{\|\boldsymbol{w}_t^i\|_2^{3/2}}
    \le
    \Upsilon(\boldsymbol{W}_t, \mathbf{Y}, \boldsymbol{R})
    + \lambda\sum_{i=1}^d \frac{1}{4}\frac{\|\boldsymbol{w}_t^i\|_2^2}{\|\boldsymbol{w}_t^i\|_2^{3/2}}.
    \label{eq:inter}
\end{equation}

Now, we add and subtract \(\lambda\sum_{i=1}^d \|\boldsymbol{w}_{t+1}^i\|_2^{1/2}\) on the left-hand side and $\lambda\sum_{i=1}^d \|\boldsymbol{w}_t^i\|_2^{1/2}$ on the right-hand side of (\ref{eq:inter}). Rearranging terms, we obtain
\begin{equation}
    \begin{aligned}
    \Upsilon(\boldsymbol{W}_{t+1}, \mathbf{Y}, \boldsymbol{R})
    &+ \lambda\sum_{i=1}^d \|\boldsymbol{w}_{t+1}^i\|_2^{1/2}
    - \lambda\sum_{i=1}^d \left( \|\boldsymbol{w}_{t+1}^i\|_2^{1/2} - \frac{1}{4}\frac{\|\boldsymbol{w}_{t+1}^i\|_2^2}{\|\boldsymbol{w}_t^i\|_2^{3/2}} \right) \\
    &\qquad\qquad  \le
    \Upsilon(\boldsymbol{W}_t, \mathbf{Y}, \boldsymbol{R})
    + \lambda\sum_{i=1}^d \|\boldsymbol{w}_t^i\|_2^{1/2}
    - \lambda\sum_{i=1}^d \left( \|\boldsymbol{w}_t^i\|_2^{1/2} - \frac{1}{4}\frac{\|\boldsymbol{w}_t^i\|_2^2}{\|\boldsymbol{w}_t^i\|_2^{3/2}} \right). 
\end{aligned}
\label{eq:addsub}
\end{equation}

According to the inequality established in Lemma \ref{the:con_lem_2}, which guarantees
\begin{equation}
    \lambda \sum_{i=1}^{d}\left(\|w_{t+1}^{i} \|_{2}^{1/2} - \frac{1}{4} \frac{\| w_{t+1}^{i} \|_2^2}{\|w_{t}^{i}  \|_2^{3/2}}\right)
	\leq 
	\lambda \sum_{i=1}^{d}\left(\|w_{t}^{i} \|_{2}^{1/2} - \frac{1}{4} \frac{\| w_{t}^{i} \|_2^2}{\|w_{t}^{i}  \|_2^{3/2}} \right ),
\end{equation}

Thus, rearranging Eq.\eqref{eq:addsub} and applying this inequality directly yields
\begin{equation}
    \Upsilon(\boldsymbol{W}_{t+1}, \mathbf{Y}, \boldsymbol{R})
    + \lambda\sum_{i=1}^d \|\boldsymbol{w}_{t+1}^i\|_2^{1/2}
    \le
    \Upsilon(\boldsymbol{W}_t, \mathbf{Y}, \boldsymbol{R})
    + \lambda\sum_{i=1}^d \|\boldsymbol{w}_t^i\|_2^{1/2}.
    \label{eq:final}
\end{equation}

Finally, recalling the definition of the $\ell_{2,1/2}$-matrix norm $\|\boldsymbol{W}\|_{2,1/2}^{1/2} = \sum_{i=1}^d \|\boldsymbol{w}^i\|_2^{1/2}$, inequality (\ref{eq:final}) is precisely
\begin{equation}
    \Upsilon(\boldsymbol{W}_{t+1}, \mathbf{Y}, \boldsymbol{R})
    + \lambda \|\boldsymbol{W}_{t+1}\|_{2,1/2}^{1/2}
    \le
    \Upsilon(\boldsymbol{W}_t, \mathbf{Y}, \boldsymbol{R})
    + \lambda \|\boldsymbol{W}_t\|_{2,1/2}^{1/2}.
    \label{eq:pro7_21}
\end{equation}

Eq.(\ref{eq:pro7_21}) can be equivalently written as
\begin{equation}
f_{\boldsymbol{D}}(\boldsymbol{W}_{t+1})\leq f_{\boldsymbol{D}}(\boldsymbol{W}_t),
\end{equation}
which establishes the monotonic non-increase of the objective function along the iterations.

This completes the proof.
\end{proof}

\section{Proof of Proposition 8: Generalization Bound with Rademacher Complexity}\label{apx:gen_f}

\noindent We first proof the following proposition and lemmas:

\begin{proposition}\label{prop:loss_lipschitz}
Let $\mathcal{W} \!=\! \{ W\!\in\mathbb{R}^{d\times c}: \|W\|_F\le Q \}$ be the hypothesis class in Eq.(23), and let $z=(x,y)$ denote an arbitrary sample satisfying $\|x\|_2\le 1$ and $y\in[0,1]^c$. For the per-sample profiled square loss 
\begin{equation}
    \ell(W;z) = \frac{s}{2}\|x^\top W-y\|_2^2, \qquad s=\frac{\mu}{1+\mu}\in(0,1),
\end{equation}
the function $\ell(\cdot;z)$ is Lipschitz continuous \textit{w.r.t.} the Frobenius norm over $\mathcal{W}$. Specifically,
for any $W_1,W_2\in\mathcal{W}$,
\begin{equation}
    | \ell(W_1;z)-\ell(W_2;z) | \le L_\ell\|W_1-W_2\|_F, 
\end{equation}
where $L_\ell = s\sqrt{c}$.
\end{proposition}

\begin{proof}[\textbf{Proof}]
    For any $W_1,W_2\in\mathcal{W}$, and let $h_1=x^{\top} W_1, h_2 = x^{\top} W_1$, using the identity $\|a\|_2^2-\|b\|_2^2=\langle a-b,a+b\rangle$, we have
\begin{equation}
    \begin{aligned}
    | \ell(W_1;z)-\ell(W_2;z) | 
    &= \frac{s}{2} | \|h_1-y\|_2^2 - \|h_2-y\|_2^2 | \\
    &= \frac{s}{2} | \langle h_1-h_2, h_1+h_2-2y \rangle | \\
    &\le \frac{s}{2} \|h_1-h_2\|_2 \|h_1+h_2-2y\|_2 . 
\end{aligned}
\label{eq:loss_difference}
\end{equation}

The first factor in Eq.\eqref{eq:loss_difference} satisfies
\begin{equation}
    \|h_1-h_2\|_2=\|x^\top(W_1-W_2)\|_2 \le \|x\|_2\|W_1-W_2\|_F \le \|W_1-W_2\|_F.
    \label{eq:f_for}
\end{equation}

Moreover, since $x^{\top} W\in [0,1]^c$, $y\in[0,1]^c$ and the absolute value of each component of the second factor doesn't exceed $2$, that is, let $v=h_1+h_2-2y$, the $j$-th component is
\begin{equation}
    v_j = h_{1j}+h_{2j}-2y_j
\end{equation}

Since $h_{1j}, h_{2j}, y_j\in [0,1]$, then we have
\begin{equation}
    -2 \le h_{1j}+h_{2j}-2y_j \le 2 \Longrightarrow |v_j| \le 2.
\end{equation}

And hence, we have 
\begin{equation}
\|v\|_2=\|h_1+h_2-2y\|_2 = \sqrt{\sum_{j=1}^c h_{1j}+h_{2j}-2y_j}
\le \sqrt{\sum_{j=1}^c 2^2} = 2\sqrt{c}. 
\label{eq:s_for}
\end{equation}

Substituting Eq.(\ref{eq:f_for}) and Eq.(\ref{eq:s_for}) into \eqref{eq:loss_difference} gives
\begin{equation}
    \left| \ell(W_1;z)-\ell(W_2;z) \right| \le s\sqrt{c} \|W_1-W_2\|_F. 
\end{equation}

Therefore, $\ell(\cdot;z)$ is $s\sqrt{c}$-Lipschitz continuous over $\mathcal{W}$.

The completes the proof.
\end{proof}

\begin{lemma}\label{the:genone}
Let $\mathcal{H}\!=\!\{\mathbf{x}\mapsto \mathbf{x} \boldsymbol{W}\!: \boldsymbol{W} \in \mathbb{R}^{d \times c} \}$ be the class of linear predictors used in ProFS, where $\mathbf{x}\!\in\!\mathbb{R}^d$ is an input instance. For an $L$-Lipschitz loss function $\ell: \mathbb{R}^c \to \mathbb{R}$, the Rademacher complexity of the composition $\ell \circ \mathcal{H}$ satisfies
\begin{equation}\label{eq:rad_bound}
\hat{\mathcal{R}}_S(\ell \circ \mathcal{H}) 
\;\le\; \frac{\sqrt{2} L}{n} \; \mathbb{E}_{\mathbf{\sigma}} \!\left[ \sup_{\boldsymbol{W} \in \mathcal{H}} \big\langle \boldsymbol{W}, \widetilde{\mathbf{X}} \big\rangle \right],
\end{equation}
where $\widetilde{\mathbf{X}}\!\in\!\mathbb{R}\!^{d \times c}$ is a matrix built from the data and Rademacher variables as defined in the proof.
\end{lemma}

\begin{proof}[\textbf{Proof}]
Recall the Rademacher complexity for a real-valued function class $\mathcal{G}$ in Definition \ref{preli_rademac}:
\begin{equation}
    \hat{\mathcal{R}}_S(\mathcal{G}) = \mathbb{E}_{\mathbf{\sigma}} \!\left[ \sup_{g \in \mathcal{G}} \frac{1}{n} \sum_{i=1}^n \sigma_i g(\mathbf{x}_i) \right],
\end{equation}
where $\sigma_i$ are independent Rademacher random variables.  

For a vector-valued function class $\mathcal{H}$ and an $L$-Lipschitz loss $\ell$, Lemma \ref{lemma:contro}  gives
\begin{equation}
    \hat{\mathcal{R}}_S(\ell \circ \mathcal{H}) 
\;\le\; \frac{\sqrt{2} L}{n} \; \mathbb{E}_{\mathbf{\sigma}} \!\left[ \sup_{\boldsymbol{W} \in \mathcal{H}} \sum_{i=1}^n \sum_{j=1}^c \sigma_{ij} h_j(\mathbf{x}_i) \right],
\label{supp_hatlh}
\end{equation}
where $\sigma_{ij}$ are independent Rademacher variables (one for each component $j$ and each sample $i$), and $h_j(\mathbf{x}_i)$ denotes the $j$-th output of the predictor.  

In our linear setting, $h_j(\mathbf{x}_i) = \mathbf{x}_i \mathbf{w}_j$, where $\mathbf{w}_j$ is the $j$-th column of $\boldsymbol{W}$. Hence
\begin{equation}
    \sum_{i=1}^n \sum_{j=1}^c \sigma_{ij} h_j(\mathbf{x}_i) = \sum_{i=1}^n \sum_{j=1}^c \sigma_{ij} (\mathbf{x}_i \mathbf{w}_j).
\end{equation}

For each $j$, define $\widetilde{\mathbf{x}}_j = \sum_{i=1}^n \sigma_{ij} \mathbf{x}_i \in \mathbb{R}^d$. Arranging these vectors as columns gives the matrix $\widetilde{\mathbf{X}} = [\widetilde{\mathbf{x}}_1, \dots, \widetilde{\mathbf{x}}_c] \in \mathbb{R}^{d \times c}$. Then, we have
\begin{equation}
    \sum_{i=1}^n \sum_{j=1}^c \sigma_{ij} (\mathbf{x}_i \mathbf{w}_j) = \sum_{j=1}^c \mathbf{w}_j^\top \widetilde{\mathbf{x}}_j = \langle \boldsymbol{W}, \widetilde{\mathbf{X}} \rangle,
    \label{supp_wtilex}
\end{equation}
where $\langle \boldsymbol{W}, \widetilde{\mathbf{X}} \rangle = \Tr(\boldsymbol{W}^\top \widetilde{\mathbf{X}})$ denotes the Frobenius inner product. Substituting Eq.\eqref{supp_wtilex} into Eq.\eqref{supp_hatlh} yields
\begin{equation}
    \hat{\mathcal{R}}_S(\ell \circ \mathcal{H}) \;\le\; \frac{\sqrt{2} L}{n} \; \mathbb{E}_{\mathbf{\sigma}} \!\left[ \sup_{\boldsymbol{W} \in \mathcal{H}} \big\langle \boldsymbol{W}, \widetilde{\mathbf{X}} \big\rangle \right].
\end{equation}

The completes the proof.
\end{proof}

\noindent With the Rademacher complexity above, we can derive the following bounds for ProFS: 

\begin{lemma}\label{the:4}
Let $\mathcal{H}\!=\!\{\mathbf{x}\!\mapsto\!\mathbf{x}\boldsymbol{W}\!: \boldsymbol{W}\!\in\!\mathbb{R}\!^{d\times c}, \|\boldsymbol{W}\|_F\!\le\!Q \}$ be the class of linear predictors used in ProFS, where $\|\mathbf{x}\|_2\!\le\!1$ for all instances. Assume the loss function $\ell$ is $L$-Lipschitz continuous function. Then, the Rademacher complexity of $\ell \circ \mathcal{H}$ is bounded by
\begin{equation}\label{eq:final_rad}
\hat{\mathcal{R}}_S(\ell \circ \mathcal{H}) \;\le\; \frac{\sqrt{2c}\; L Q}{\sqrt{n}}.
\end{equation}
\end{lemma}

\begin{proof}[\textbf{Proof}]
According to Lemma \ref{the:genone}, we have
\begin{equation}
    \hat{\mathcal{R}}_S(\ell \circ \mathcal{H}) \;\le\; \frac{\sqrt{2} L}{n}\; \mathbb{E}_{\mathbf{\sigma}}\Bigl[ \sup_{\boldsymbol{W}\in\mathcal{H}} \bigl\langle \boldsymbol{W}, \widetilde{\mathbf{X}} \bigr\rangle \Bigr].
\end{equation}

By the Cauchy–Schwarz inequality,
\begin{equation}
    \bigl\langle \boldsymbol{W}, \widetilde{\mathbf{X}} \bigr\rangle \le \|\boldsymbol{W}\|_F \; \|\widetilde{\mathbf{X}}\|_F.
\end{equation}

Since all $\boldsymbol{W}\in\mathcal{H}$ satisfy $\|\boldsymbol{W}\|_F \le Q$, we obtain
\begin{equation}
    \hat{\mathcal{R}}_S(\ell \circ \mathcal{H}) \;\le\; \frac{\sqrt{2} L Q}{n}\; \mathbb{E}_{\mathbf{\sigma}}\bigl[ \|\widetilde{\mathbf{X}}\|_F \bigr].
    \label{supp_rlh}
\end{equation}

Using the definition of $\widetilde{\mathbf{X}}$, it is easy to prove the bound
\begin{equation}
    \begin{aligned}
	\mathbb{E}_{\mathbf{\sigma}}\bigl[ \|\widetilde{\mathbf{X}}\|_F^2 \bigr]
	= \mathbb{E}_{\mathbf{\sigma}} \left [ \sum_{j=1}^{c} \| \tilde{\boldsymbol{X_j}}  \|_2^2 \right ]  
    =\mathbb{E}_{\mathbf{\sigma}} \left [ \sum_{j=1}^{c}   \| \sum_{i=1}^{n} \sigma_{ij}\mathbf{x}_i  \|_2^2 \right ] \leq cn.
    \label{supp_tildex}
\end{aligned}
\end{equation}

By Jensen's inequality, $\mathbb{E}[\|\widetilde{\mathbf{X}}\|_F] \le \sqrt{\mathbb{E}[\|\widetilde{\mathbf{X}}\|_F^2]} \le \sqrt{c n}$. Substituting Eq.(\ref{supp_tildex}) into Eq.(\ref{supp_rlh}) yields
\begin{equation}
    \hat{\mathcal{R}}_S(\ell \circ \mathcal{H}) \;\le\; \frac{\sqrt{2c}\; L Q}{\sqrt{n}}.
\end{equation}

The completes the proof.
\end{proof} 

\noindent\textbf{Proposition 8}
\textit{Let $\mathcal{W}$ be the class of linear predictors used in Eq.(23), where $\|\mathbf{x}\|_2\!\le\!1$ for all instances. Suppose that the loss function $\ell$ is $s\sqrt{c }$-Lipschitz and bounded in $[0,M]$. Then for any $\delta \in (0,1)$ with probability at least $1-\delta$ over a sample $S$ of size $n$, the  following inequality holds for all $\boldsymbol{W}\in \mathcal{W}$,
\begin{equation}
	R(\boldsymbol{W}) \le \hat{R}_S(\boldsymbol{W})+\frac{\mu2\sqrt{2}Qc}{(1+\mu)\sqrt{n}}+ 3M\sqrt{\frac{\ln(2/\delta)}{2n}},
	\label{euq:ld}
\end{equation}
where $R(\boldsymbol{W}) = \mathbb{E}_{(\mathbf{x},\mathbf{y})}[\ell(\boldsymbol{W};(\mathbf{x},\mathbf{y}))]$ is the expected risk,  $\hat{R}_S(\boldsymbol{W}) = \frac{1}{n}\sum_{i=1}^n \ell(\boldsymbol{W};(\mathbf{x}_i,\mathbf{y}_i))$ is the empirical risk, and $s=\frac{\mu}{1+\mu}$, $Q\!=\!\max\{\|\boldsymbol{W}\|_F: \boldsymbol{W}\in \mathcal{W}\}$.}

\begin{proof}[\textbf{Proof}]
    It is possible to directly prove this proposition by applying Lemma \ref{the:4}, Lemma \ref{lemma:gengen} with $s=\frac{\mu}{1+\mu}$.
\end{proof}

\section{Proof of Proposition 9: Smoothness of the Surrogate Objective}\label{apx:grad_lipschitz}
\noindent\textbf{Proposition 9} \textit{Let $S=\{z_i=(x_i,y_i)\}_{i=1}^{n}$ be a training sample, where $x_i\in\mathbb{R}^{d}$ and $y_i\in\mathbb{R}^{c}$. For a fixed reweighting matrix $\boldsymbol{D}\succeq 0$, define the regularized empirical objective of Eq.(23)
\begin{equation}
    F_S(\boldsymbol{W}) = \sum_{i=1}^{n}\ell(\boldsymbol{W};z_i) + \frac{\alpha}{2}{\Tr}(\boldsymbol{WRW}^\top) + 2\lambda{\Tr}(\boldsymbol{W}^\top \boldsymbol{DW}),
\end{equation}
where $\ell(\boldsymbol{W};z_i) = \frac{s}{2}\|x_i^\top \boldsymbol{W}-y_i\|_2^2$, $s=\mu/(1+\mu)$. 
If the structural regularization is convex, then $F_S$ is convex and $\beta$-smooth \textit{w.r.t.} the Frobenius norm. Consequently, for any $\boldsymbol{W}_1,\boldsymbol{W}_2 \in \mathbb{R}^{d\times c}$:
\begin{equation}
    \|\nabla F_S(\boldsymbol{W}_1)-\nabla F_S(\boldsymbol{W}_2)\|_F \le \beta\|\boldsymbol{W}_1-\boldsymbol{W}_2\|_F,
\end{equation}
where
\begin{equation}
    \beta = {s}\|\mathbf{X}^\top \mathbf{X}\|_2 + \frac{\alpha}{2}\|\widetilde{\boldsymbol{R}}\|_2 + 4\lambda\|\boldsymbol{D}\|_2, \  \widetilde{\boldsymbol{R}}=(\boldsymbol{R}+\boldsymbol{R}^{\top}).
\end{equation}}

\begin{proof}[\textbf{Proof}]

From the gradient expression of $\nabla \ell$ in Eq.(24), we have
\begin{equation}
\begin{aligned}
    \nabla F_S(\boldsymbol{W}) &= {s}\mathbf{X}^\top(\mathbf{X}\boldsymbol{W}-Y) + \frac{\alpha}{2}\boldsymbol{W(\boldsymbol{R}+\boldsymbol{R}^{\top})} + 4\lambda \boldsymbol{DW}.\\
\overset{\widetilde{\boldsymbol{R}}\!=\!(\boldsymbol{R}\!+\!\boldsymbol{R}\!^{\top})}{\Longleftrightarrow}  \nabla F_S(\boldsymbol{W}) &= {s}\mathbf{X}^\top(\mathbf{X}\boldsymbol{W}-Y) + \frac{\alpha}{2}\boldsymbol{W\widetilde{\boldsymbol{R}}} + 4\lambda \boldsymbol{DW}.
\end{aligned}
\label{eq:objective_gradient}
\end{equation}

For any $\boldsymbol{W}_1,\boldsymbol{W}_2 \in \mathbb{R}^{d\times c}$, let $\Delta \boldsymbol{W}=\boldsymbol{W}_1-\boldsymbol{W}_2$. It follows from Eq.\eqref{eq:objective_gradient} that
\begin{align}
\nabla F_S(\boldsymbol{W}_1)-\nabla F_S(\boldsymbol{W}_2) = & {s}\mathbf{X}^\top \mathbf{X}\Delta \boldsymbol{W} + \frac{\alpha}{2}\Delta \boldsymbol{W\widetilde R} + 4\lambda \boldsymbol{D}\Delta \boldsymbol{W}.
\end{align}

Taking the Frobenius norm and applying the triangle inequality gives
\begin{equation}
\|\nabla F_S(\boldsymbol{W}_1)-\nabla F_S(\boldsymbol{W}_2)\|_F \le {s}\|\mathbf{X}^\top \mathbf{X}\Delta \boldsymbol{W}\|_F + \frac{\alpha}{2}\|\Delta \boldsymbol{W\widetilde R}\|_F + 4\lambda\|\boldsymbol{D}\Delta \boldsymbol{W}\|_F.
\label{eq:gradient_difference}
\end{equation}

Using $\|AB\|_F\le\|A\|_2\|B\|_F$, or $\|AB\|_F\le\|A\|_F\|B\|_2$, we obtain
\begin{equation}
\begin{aligned}
\|\mathbf{X}^\top \mathbf{X}\Delta \boldsymbol{W}\|_F &\le \|\mathbf{X}^\top \mathbf{X}\|_2\|\Delta \boldsymbol{W}\|_F,\\
\|\Delta \boldsymbol{W\widetilde R}\|_F &\le \|\widetilde {\boldsymbol{R}}\|_2\|\Delta \boldsymbol{W}\|_F,\\
\|\boldsymbol{D}\Delta \boldsymbol{W}\|_F &\le \|\boldsymbol{D}\|_2\|\Delta \boldsymbol{W}\|_F.
\end{aligned}
\end{equation}

Substituting these inequalities into Eq.\eqref{eq:gradient_difference} yields
\begin{equation}
    \|\nabla F_S(\boldsymbol{W}_1)-\nabla F_S(\boldsymbol{W}_2)\|_F \le \left({s}\|\mathbf{X}^\top \mathbf{X}\|_2 + \frac{\alpha}{2}\|\widetilde {\boldsymbol{R}}\|_2 + 4\lambda\|\boldsymbol{D}\|_2 \right)  \|\boldsymbol{W}_1-\boldsymbol{W}_2\|_F.
\end{equation}

Therefore, $F_S$ is $\beta$-smooth.

This completes the proof.
\end{proof}

\section{Proof of Proposition 10: Generalization Bound with Uniform Stability}\label{apx:genstab}
\noindent We first proof the following lemma:

\begin{lemma}\label{the:stable}
Let the loss function $\ell$ is convex, $L$-Lipschitz and $\beta$-smooth. Suppose that we run gradient descent (GD) with step sizes $\eta_t\!\le\! 2/\beta$ for $T$ steps. Then, the algorithm satisfies uniform stability with 
\begin{equation}
    \epsilon_{\mathrm{stab}}\le \frac{2L^2}{n} \sum_{t=1}^{T} \eta_t.
\end{equation}
\end{lemma}

\begin{proof}[\textbf{Proof}]
In the following, we use $w$ to denote $\boldsymbol{W}$. 
Let $S$ and $S'$ be two samples of size $n$ differing in only one example. Let $w_t$ and $w'_t$ denote the iterates produced by GD on sample $S$ and $S'$ respectively, starting from $w_0 = w'_0$. Define $\delta_t = \|w_t - w'_t\|$.

At iteration $t+1$, the update rule is deterministic and uses all samples:
\begin{equation}
    w_{t+1} = w_t - \eta_t \nabla \ell_S(w_t), \quad
w'_{t+1} = w'_t - \eta_t \nabla \ell_{S'}(w'_t).
\end{equation}

Based on the triangle inequality and the update rule:
\begin{equation}
    \begin{aligned}
    \delta_{t+1} &= \|(w_t - \eta_t \nabla \ell_S(w_t)) - (w'_t - \eta_t \nabla \ell_{S'}(w'_t))\|\\
    &\le \|w_t - w'_t\| + \eta_t\|\nabla \ell_S(w_t) - \nabla \ell_{S'}(w'_t)\|.
\end{aligned}
\label{supp_delt}
\end{equation}

Since $S$ and $S'$ differ only in the $i$-th example, we have:
\begin{equation}
    \nabla \ell_S(w_t) - \nabla \ell_{S'}(w'_t) = \frac{1}{n}[(\nabla \ell(w_t;z_i) \!-\! \nabla \ell(w'_t;z'_i)) + \sum_{j \ne i} (\nabla \ell(w_t;z_j)\!-\!\nabla \ell(w'_t;z_j))].
\end{equation}

By the Lipschitz property and the smoothness property of $\ell$,
\begin{equation}
    \begin{aligned}
&\|\nabla \ell(w_t;z_i) - \nabla \ell(w'_t;z'_i)\| \le 2L,\\
\|\nabla \ell(w_t;z_j) & \!-\! \nabla \ell(w'_t;z_j)\| \!\le\! \beta \|w_t - w'_t\| \!=\! \beta\delta_t\ (j\ne i).
\end{aligned}
\label{supp_smo2}
\end{equation}

Substituting  Eq.(\ref{supp_smo2}) into Eq.(\ref{supp_delt}) yields
\begin{equation}
    \|\nabla \ell_S(w_t) - \nabla \ell_{S'}(w'_t)\| \le \frac{1}{n}\left[2L + (n-1)L\delta_t\right].
\end{equation}

By Lemma \ref{lem:bounded}, each gradient update is $\eta_tL$-bounded. Moreover, due to convexity and the step‑size condition $\eta_t\le \frac{2}{\beta}$, the gradient update is 1‑expansive (Lemma \ref{lem:smooth_expansive}), we obtain the following inequality by Lemma \ref{lem:growth}:
\begin{equation}
    \delta_{t+1} \le \delta_t + \eta_t \cdot \frac{2L}{n}.
\end{equation}

Unrolling this recurrence from $\delta_0 = 0$ yields:
\begin{equation}
    \delta_T \le \frac{2L}{n} \sum_{t=1}^T \eta_t.
\end{equation}

Finally, by the $L$-Lipschitz continuity of $\ell(\cdot; z)$, we have 
\begin{equation}
\mathbb{E}|\ell(w_T; z) - \ell(w'_T; z)| \le L \delta_T \le \frac{2L^2}{n} \sum_{t=1}^T \eta_t,
\end{equation}
which is the desired uniform stability bound in Definition \ref{preli_uniform}.

The completes the proof.
\end{proof}

With this stability bound in hand, we can directly translate it into a generalization guarantee. The relationship between generalization and uniform stability is established in Proposition 10.

\noindent\textbf{Proposition 10} 
\textit{Let the loss function $\ell$ is convex, $s\sqrt{c}$-Lipschitz, and $\beta$-smooth, $0\le \ell(\boldsymbol{W}\!_{S};z)\le M$. Suppose that $\mathcal{W}$ be the class of linear predictors used in Eq.(23), where $\|\mathbf{x}\|_2\!\le\!1$ for all instances, and  $\boldsymbol{W}\!_{S}$ be the output obtained after $T$ iterations of the update rule in Eq.(25) on a training sample $S$ of size $n$, where the learning rate at the t-th iteration $\eta_t$ satisfies $0<\eta_t\!\le\! 2/\beta$. Then, for any $\delta\! \in\! (0, 1)$, the following bound holds with probability at least $1\!-\!\delta$ over the random draw of $S$,
\begin{equation}
    R(\boldsymbol{W}_{S})\le \hat{R}(\boldsymbol{W}_{S}) + \frac{2\mu^2 c}{(1+\mu)^2n} \sum_{t=1}^T \eta_t + \left(4\frac{\mu^2 c}{(1+\mu)^2}\sum_{t=1}^T \eta_t + M \right)\sqrt{\frac{\log\frac{1}{\delta}}{2n}}.
\end{equation}}

\begin{proof}[\textbf{Proof}]
We use $w$ to denote $\boldsymbol{W}$. 
Consider the perturbed training dataset $S'$ obtained by replacing a sample in the original training set, where the original sample $z_i$ is replaced with a new sample $z'_i$. Let $\Phi(S)=R(w_S)-\hat{R}(w_S)$, we have    
\begin{equation}
    \begin{aligned}
    |&\Phi(S)-\Phi(S')| \\
    &= |(R(w_S)-\hat{R}(w_S)) - (R(w'_S)-\hat{R}(w'_S))|\\
    &\le |\hat{R}(w_S) - \hat{R}(w’_S)| + |{R}(w_S) - {R}(w'_S)|\\
    &\le \frac{|\hat{R}(w_S,z_i) - \hat{R}(w’_S,z'_i)|}{n} + \sum_{j\ne i}\frac{|\hat{R}(w_S,z_j) - \hat{R}(w’_S,z_j)|}{n} +\ |{R}(w_S) - {R}(w'_S)|\\
    &= \frac{|\hat{R}(w_S,z_i) - \hat{R}(w’_S,z'_i)|}{n} + \sum_{j\ne i}\frac{|\hat{R}(w_S,z_j) - \hat{R}(w’_S,z_j)|}{n} +\ |\mathbb{E}_z[{R}(w_S,z) - {R}(w'_S,z)]|\\
    &\le \frac{M}{n} + \frac{(n-1)\epsilon_{\text{stab}}}{n} + \epsilon_{\text{stab}}\\
    &\le \frac{M}{n} + 2\epsilon_{\text{stab}}.
\end{aligned}
\end{equation}

Thus, the function $\Phi(\cdot)$ satisfies the conditions of McDiarmid's inequality (Lemma \ref{lemma:McDia}) with $c_i=\frac{M}{n} + 2\epsilon_{\text{stab}}$. Next, we derive an upper bound estimate for the expected value of $\Phi$:
\begin{equation}
    \begin{aligned}
    \mathbb{E}_S[\Phi(S)] &= \mathbb{E}_S[R(w_S)-\hat{R}(w_S)]\\
    &= \mathbb{E}_{S,z'}[R(w_S,z')-\hat{R}(w_S,z')]\\
    &\le \epsilon_{\text{stab}}.
\end{aligned}
\end{equation}

Which yields 
\begin{equation}
    \begin{aligned}
    \mathbb{P}(\Phi(S)\ge \epsilon_{\text{stab}} + \epsilon) &\le \mathbb{P}(\Phi(S)\ge \mathbb{E}_S[\Phi(S)] + \epsilon)\\
    &\le \exp{\left ( -\frac{2\epsilon^2}{n(\frac{M}{n}+2\epsilon_{\text{stab}})^2} \right)}\\
    &= \exp{\left ( -\frac{2n\epsilon^2}{(M+2n\epsilon_{\text{stab}})^2} \right)}.
\end{aligned}
\end{equation}

Thus, setting the right hand side to $\delta$, we obtain that with probability at least $1-\delta$,
\begin{equation}
    R(w_S)\le \hat{R}(w_S) + \epsilon_{\text{stab}} + \left(2n\epsilon_{\text{stab}}+M\right)\sqrt{\frac{\log\frac{1}{\delta}}{2n}}.
\end{equation}

It is straightforward to get this proposition by combing above result with Lemma \ref{the:stable} and Proposition \ref{prop:loss_lipschitz}.

The completes the proof.
\end{proof}

\vspace{2mm}
\section{Additional Experimental Results and Analysis}\label{app:resu}
\subsection{Additional Comparative Results on One Error and Coverage}
Tables \ref{tab:oe} and \ref{tab:oe_asy} report the complete results of ProFS in terms of \textit{One-Error} (\textit{OE}) and \textit{Coverage} (\textit{CV}) on the
12 benchmark datasets. Table~\ref{tab:oe} compares ProFS with conventional approaches that are primarily based on symmetric label relationships, whereas Table~\ref{tab:oe_asy} presents the comparisons with methods specifically designed to model asymmetric label
dependencies, together with their relevant symmetric variants. For both metrics, a smaller value indicates better performance.

\begin{table*}[!h]
	\centering
    \vspace{2mm}
\caption{{Prediction performance of each conventional comparing algorithm (mean$\pm$std) in terms of \textit{One Error} and \textit{Coverage}, where the best results are highlighted in \textbf{boldface}}.}
\label{tab:oe}
    \footnotesize
\setlength{\tabcolsep}{9pt}
\begin{tabular}{clccccccc}
    \toprule
    \multirow{2}{*}{Id}&\multirow{2}{*}{Datasets} & \multicolumn{7}{c}{\textit{One-Error} ($\downarrow$)} \\ \cline{3-9}
    & & ML-KNN 		 & JFSC 			& MDFS 				& GLFS 				& GRROOR 		& MLACO 	& ProFS \\ \toprule
    1  & Emotions      & 0.292$\pm$0.030 & 0.310$\pm$0.008 	& 0.321$\pm$0.020 	& 0.320$\pm$0.037 	&0.323$\pm$0.057&0.368$\pm$0.020  	& \textbf{0.243$\pm$0.058}\\
    2  & Cal500        & 0.116$\pm$0.015 & 0.118$\pm$0.027 	& 0.137$\pm$0.019 	& 0.128$\pm$0.020 &\textbf{0.102$\pm$0.006} &0.130$\pm$0.001  & 0.109$\pm$0.030 \\
    3  & Birds         & 0.361$\pm$0.012 & 0.366$\pm$0.042 	& 0.380$\pm$0.016 	& 0.378$\pm$0.022 &0.408$\pm$0.036 &0.479$\pm$0.016  & \textbf{0.333$\pm$0.050} \\
    4  & Image         & 0.350$\pm$0.010 & 0.345$\pm$0.033 	& 0.362$\pm$0.025 	& 0.424$\pm$0.053 &0.415$\pm$0.097 &0.404$\pm$0.034  & \textbf{0.342$\pm$0.023} \\
    5  & Scene         & 0.312$\pm$0.048 & 0.278$\pm$0.012 	& 0.380$\pm$0.031 	& 0.357$\pm$0.084 &0.400$\pm$0.165 &0.303$\pm$0.037  & \textbf{0.273$\pm$0.009} \\
    6  & Yeast         & 0.238$\pm$0.012 & 0.236$\pm$0.012 	& 0.254$\pm$0.007 	& 0.249$\pm$0.007 &0.267$\pm$0.010 &0.259$\pm$0.005  & \textbf{0.207$\pm$0.019} \\
    7  & Business      & 0.120$\pm$0.008 & 0.117$\pm$0.006 & 0.122$\pm$0.003 	& 0.128$\pm$0.004 &0.128$\pm$0.005 &0.131$\pm$0.006  & \textbf{0.115$\pm$0.011} \\
    8 & Entertainment & 0.625$\pm$0.010 & 0.496$\pm$0.012 & 0.575$\pm$0.046 & 0.523$\pm$0.060 & 0.567$\pm$0.089 &0.595$\pm$0.038  & \textbf{0.408$\pm$0.002} \\
    9 & Education     & 0.500$\pm$0.015 & 0.507$\pm$0.008 & 0.500$\pm$0.023 & 0.553$\pm$0.034 & 0.593$\pm$0.033 &0.611$\pm$0.019  & \textbf{0.482$\pm$0.009} \\
    10 & Computers     & 0.428$\pm$0.016 & 0.422$\pm$0.013 & 0.404$\pm$0.021 & 0.447$\pm$0.021 & 0.428$\pm$0.044 &0.462$\pm$0.008  & \textbf{0.339$\pm$0.007} \\
    11 & Health        & 0.392$\pm$0.019 & 0.338$\pm$0.009 & 0.362$\pm$0.021 & 0.347$\pm$0.026 & 0.399$\pm$0.052 &0.420$\pm$0.028  & \textbf{0.294$\pm$0.006} \\
    12 & Social        & 0.355$\pm$0.020 & 0.327$\pm$0.017 & 0.326$\pm$0.011 & 0.367$\pm$0.018 & 0.394$\pm$0.062 &0.415$\pm$0.023  & \textbf{0.312$\pm$0.020} \\ 
    \hline
    & Average value 	& 0.341			& 0.322				& 0.344				& 0.352				& 0.369			& 0.381				& \textbf{0.288}\\
	\end{tabular}
\footnotesize
\setlength{\tabcolsep}{5.9pt}
\begin{tabular}{clccccccc}
    \toprule
    \multirow{2}{*}{Id}&\multirow{2}{*}{Datasets} & \multicolumn{7}{c}{\textit{Coverage} ($\downarrow$)} \\ \cline{3-9}
    & & ML-KNN 		 & JFSC 			& MDFS 				& GLFS 				& GRROOR 		& MLACO 	& ProFS \\ \toprule
    1  & Emotions      & 1.842$\pm$0.186 	& 2.028$\pm$0.131 	& 1.949$\pm$0.069 	& 1.965$\pm$0.134 	&1.970$\pm$0.276 	&2.069$\pm$0.044 & \textbf{1.807$\pm$0.150} \\
    2  & Cal500        & 131.104$\pm$1.769 	& 130.059$\pm$2.877 & 130.756$\pm$0.487 & 130.601$\pm$0.858 &131.786$\pm$0.655  &130.203$\pm$0.182  & \textbf{129.400$\pm$0.967} \\
    3  & Birds         & 3.107$\pm$0.321 	& \textbf{2.976$\pm$0.473} 	& 3.411$\pm$0.130 	& 3.563$\pm$0.207 	&3.785$\pm$0.359 	&4.275$\pm$0.113  & 3.271$\pm$0.272 \\
    4  & Image         & 1.026$\pm$0.069 	& 1.026$\pm$0.015 	& 1.047$\pm$0.063 	& 1.232$\pm$0.128 	&1.219$\pm$0.279  &1.128$\pm$0.066  & \textbf{0.940$\pm$0.079} \\
    5  & Scene         & 0.658$\pm$0.071 	& 0.619$\pm$0.036 	& 0.888$\pm$0.067 	& 0.798$\pm$0.233 	&1.001$\pm$0.590  &0.618$\pm$0.097  & \textbf{0.610$\pm$0.042} \\
    6  & Yeast         & 6.322$\pm$0.242 	& 6.456$\pm$0.202 	& 6.515$\pm$0.121 	& 6.508$\pm$0.148 	&7.104$\pm$0.123  &6.527$\pm$0.061  & \textbf{6.246$\pm$0.284} \\
    7  & Business      & 2.279$\pm$0.181 	& 3.164$\pm$0.213 	& 2.240$\pm$0.065 	& 2.409$\pm$0.053 	&2.401$\pm$0.097  &2.478$\pm$0.057  & \textbf{2.209$\pm$0.110} \\
    8 & Entertainment & 3.376$\pm$0.103 	&3.051$\pm$0.140 	& 3.233$\pm$0.083 	& 3.310$\pm$0.238 	&3.415$\pm$0.375  &3.491$\pm$0.143  & \textbf{2.847$\pm$0.112} \\
    9 & Education     & 3.612$\pm$0.153 	& 3.991$\pm$0.433 	& 3.461$\pm$0.115 	& 3.782$\pm$0.179 	&3.947$\pm$0.199  &4.123$\pm$0.096  & \textbf{3.460$\pm$0.138} \\
    10 & Computers     & 4.146$\pm$0.259 	& 4.455$\pm$0.351 	& 3.792$\pm$0.129 	& 4.189$\pm$0.208 	&3.997$\pm$0.332  &4.471$\pm$0.118  & \textbf{3.632$\pm$0.221} \\
    11 & Health        & 3.404$\pm$0.169 	& 3.708$\pm$0.404 	& 3.473$\pm$0.306 	& 3.265$\pm$0.281 	&3.314$\pm$0.360  &3.545$\pm$0.197  & \textbf{2.839$\pm$0.117} \\
    12 & Social        & 3.194$\pm$0.152 	& 4.260$\pm$0.230 	& 3.128$\pm$0.118 	& 3.671$\pm$0.153 	&3.524$\pm$0.250  &3.738$\pm$0.143  & \textbf{2.773$\pm$0.123} \\ 
    \hline
    & Average value 	& 13.673			& 13.816				& 13.658				& 13.774				& 13.955			& 13.889				& \textbf{13.336}\\
    \bottomrule
\end{tabular}
\end{table*}

\begin{table*}[!h]
	\centering
    \vspace{2mm}
\caption{{Prediction performance of each asymmetric-related comparing algorithm (mean$\pm$std) in terms of \textit{One Error} and \textit{Coverage}, where o/x indicates whether ProFS is superior/inferior to the other method.}}
\label{tab:oe_asy}
\footnotesize
\setlength{\tabcolsep}{4.5pt}
\begin{tabular}{clcccccccc}
    \toprule
    \multirow{2}{*}{Id}&\multirow{2}{*}{Datasets} & \multicolumn{8}{c}{\textit{One Error} ($\downarrow$)} \\ \cline{3-10}
    & & ML2G 		 & ML2G-S 			& MFSRD 				& ACML 				& MAHR 		& LSDF 	& ProF2S & ProFS \\ \toprule
    1  & Emotions      & 0.315$\pm$0.065$\bullet$ 	& 0.317$\pm$0.087$\bullet$ 	& 0.350$\pm$0.154$\bullet$ 	& 0.268$\pm$0.048$\bullet$ 	& 0.314$\pm$0.055$\bullet$ 	& 0.241$\pm$0.012$\bullet$ & 0.292$\pm$0.055$\bullet$ & 0.243$\pm$0.058 \\
    2  & Cal500        & 0.588$\pm$0.214$\bullet$ 	& 0.608$\pm$0.188$\bullet$  & 0.118$\pm$0.031$\bullet$   & 0.115$\pm$0.033$\bullet$   & 0.004$\pm$0.005$\circ$           & 0.126$\pm$0.020$\bullet$  & 0.295$\pm$0.032$\bullet$ & 0.109$\pm$0.030 \\
    3  & Birds         & 0.296$\pm$0.032$\circ$ 	& 0.295$\pm$0.021$\circ$ 	& 0.395$\pm$0.121$\bullet$ 	& 0.295$\pm$0.039$\circ$ 	& 0.389$\pm$0.057$\bullet$ 	       & 0.429$\pm$0.061$\bullet$  & 0.335$\pm$0.037$\bullet$ &0.333$\pm$0.050   \\
    4  & Image         & 0.369$\pm$0.038$\bullet$ 	& 0.362$\pm$0.028$\bullet$ 	& 0.435$\pm$0.101$\bullet$ 	& 0.348$\pm$0.008$\bullet$ 	& 0.373$\pm$0.007$\bullet$  & 0.277$\pm$0.019$\circ$  &0.368$\pm$0.041$\bullet$  & 0.342$\pm$0.023    \\
    5  & Scene         & 0.297$\pm$0.011$\bullet$ 	& 0.297$\pm$0.011$\bullet$ 	& 0.406$\pm$0.040$\bullet$ 	& 0.258$\pm$0.015$\circ$ 	& 0.352$\pm$0.012$\bullet$  & 0.208$\pm$0.017$\circ$  & 0.283$\pm$0.012$\bullet$ & 0.273$\pm$0.009 \\
    6  & Yeast         & 0.235$\pm$0.012$\bullet$ 	& 0.239$\pm$0.010$\bullet$ 	& 0.264$\pm$0.346$\bullet$ 	& 0.235$\pm$0.027$\bullet$ 	& 0.251$\pm$0.012$\bullet$  & 0.228$\pm$0.028$\bullet$  & 0.228$\pm$0.015$\bullet$ & 0.207$\pm$0.019\\
    7  & Business      & 0.128$\pm$0.006$\bullet$ 	& 0.125$\pm$0.007$\bullet$ 	& 0.134$\pm$0.083$\bullet$ 	& 0.119$\pm$0.009$\bullet$ 	& 0.135$\pm$0.007$\bullet$  & 0.108$\pm$0.011$\circ$  & 0.119$\pm$0.014$\bullet$ & 0.115$\pm$0.011\\
    8 & Entertainment & 0.394$\pm$0.016$\circ$ 	& 0.390$\pm$0.014$\circ$ 	& 0.633$\pm$0.035$\bullet$ 	& 0.391$\pm$0.004$\circ$ 	& 0.378$\pm$0.044$\circ$  & 0.422$\pm$0.005$\bullet$  & 0.414$\pm$0.002$\bullet$ & 0.408$\pm$0.002\\
    9 & Education     & 0.546$\pm$0.112$\bullet$ 	& 0.562$\pm$0.128$\bullet$ 	& 0.649$\pm$0.063$\bullet$ 	& 0.460$\pm$0.018$\circ$ 	& 0.472$\pm$0.031$\circ$  & 0.497$\pm$0.017$\bullet$  & 0.484$\pm$0.011$\bullet$ & 0.482$\pm$0.009\\
    10 & Computers     & 0.374$\pm$0.065$\bullet$ 	& 0.375$\pm$0.056$\bullet$ 	& 0.461$\pm$0.067$\bullet$ 	& 0.350$\pm$0.014$\bullet$ 	& 0.496$\pm$0.021$\bullet$  & 0.356$\pm$0.008$\bullet$  & 0.362$\pm$0.008$\bullet$ & 0.339$\pm$0.007\\
    11 & Health        & 0.331$\pm$0.009$\bullet$ 	& 0.323$\pm$0.011$\bullet$ 	& 0.475$\pm$0.084$\bullet$ 	& 0.254$\pm$0.015$\circ$ 	& 0.367$\pm$0.025$\bullet$  & 0.259$\pm$0.016$\circ$  & 0.284$\pm$0.007$\circ$ & 0.294$\pm$0.006\\
    12 & Social        & 0.384$\pm$0.010$\bullet$ 	& 0.387$\pm$0.009$\bullet$ 	& 0.404$\pm$0.092$\bullet$ 	& 0.273$\pm$0.012$\circ$ 	& 0.385$\pm$0.013$\bullet$   & 0.302$\pm$0.007$\circ$  & 0.307$\pm$0.023$\circ$ & 0.312$\pm$0.020\\ 
    \hline
    & Average value 	& 0.355$\bullet$			& 0.357$\bullet$				& 0.394$\bullet$				& 0.281$\circ$				& 0.326$\bullet$			& 0.288				& 0.314$\bullet$         & 0.288\\
    \bottomrule
\end{tabular}
\footnotesize
\setlength{\tabcolsep}{2.5pt}
\begin{tabular}{clcccccccc}
    \multirow{2}{*}{Id}&\multirow{2}{*}{Datasets} & \multicolumn{8}{c}{\textit{Coverage} ($\downarrow$)} \\ \cline{3-10}
    & & ML2G 		 & ML2G-S 			& MFSRD 				& ACML 				& MAHR 		& LSDF 	& ProF2S & ProFS \\ \toprule
    1  & Emotions      & 1.861$\pm$0.147$\bullet$ & 1.873$\pm$0.149$\bullet$	& 2.047$\pm$0.644$\bullet$ 	& 1.835$\pm$0.124$\bullet$ 	& 1.904$\pm$0.071$\bullet$ & 1.690$\pm$0.011$\circ$ & 2.059$\pm$0.160$\bullet$ & 1.807$\pm$0.150\\
    2  & Cal500        & 140.150$\pm$2.39$\bullet$ & 139.650$\pm$2.62$\bullet$ 	& 132.890$\pm$10.08$\bullet$ 	& 129.466$\pm$2.01$\bullet$ 	& 169.393$\pm$0.79$\bullet$ & 129.27$\pm$2.31$\bullet$  & 131.600$\pm$1.13$\bullet$ & 129.400$\pm$0.97\\
    3  & Birds         & 3.409$\pm$0.430$\bullet$ & 3.380$\pm$0.465$\bullet$ 	& 3.803$\pm$3.551$\bullet$ 	& 2.805$\pm$0.637$\circ$ 	& 4.313$\pm$0.701$\bullet$ & 2.488$\pm$0.270$\circ$  & 3.358$\pm$0.218$\bullet$ & 3.271$\pm$0.272\\
    4  & Image         & 1.055$\pm$0.065$\bullet$ & 1.041$\pm$0.063$\bullet$ 	& 1.315$\pm$0.693$\bullet$ 	& 1.010$\pm$0.060$\bullet$ 	& 1.077$\pm$0.066$\bullet$  & 0.851$\pm$0.069$\circ$  & 1.183$\pm$0.068$\bullet$ & 0.940$\pm$0.079\\
    5  & Scene         & 0.644$\pm$0.056$\bullet$ & 0.643$\pm$0.051$\bullet$ 	& 0.866$\pm$0.264$\bullet$ 	& 0.538$\pm$0.039$\circ$ 	& 0.745$\pm$0.043$\bullet$  & 0.410$\pm$0.032$\circ$  & 0.841$\pm$0.027$\bullet$ & 0.610$\pm$0.042\\
    6  & Yeast         & 7.213$\pm$0.214$\bullet$ & 7.250$\pm$0.204$\bullet$ 	& 6.887$\pm$1.567$\bullet$ 	& 6.357$\pm$0.142$\bullet$ 	& 7.321$\pm$0.305$\bullet$  & 6.238$\pm$0.139$\circ$  & 6.496$\pm$0.257$\bullet$ & 6.246$\pm$0.284\\
    7  & Business      & 3.042$\pm$0.290$\bullet$ & 3.078$\pm$0.369$\bullet$ 	& 2.802$\pm$3.087$\bullet$ 	& 2.681$\pm$0.148$\bullet$ 	& 10.390$\pm$0.529$\bullet$  & 1.917$\pm$0.080$\circ$  & 2.249$\pm$0.112$\bullet$ & 2.209$\pm$0.110\\
    8 & Entertainment & 3.138$\pm$0.149$\bullet$ & 3.132$\pm$0.160$\bullet$ 	& 3.759$\pm$2.294$\bullet$ 	& 3.447$\pm$0.226$\bullet$ 	& 10.296$\pm$0.119$\bullet$  & 2.730$\pm$0.088$\circ$  & 2.860$\pm$0.004$\bullet$ & 2.847$\pm$0.112\\
    9 & Education     & 4.523$\pm$0.371$\bullet$ & 4.466$\pm$0.337$\bullet$ 	& 4.394$\pm$2.003$\bullet$ 	& 5.297$\pm$0.186$\bullet$ 	& 19.233$\pm$0.345$\bullet$  & 3.392$\pm$0.089$\circ$  & 3.474$\pm$0.123$\bullet$ & 3.460$\pm$0.138\\
    10 & Computers     & 4.582$\pm$0.371$\bullet$ & 4.738$\pm$0.355$\bullet$ 	& 5.428$\pm$5.624$\bullet$ 	& 4.718$\pm$0.300$\bullet$ 	& 13.351$\pm$0.289$\bullet$  & 3.478$\pm$0.086$\circ$  & 3.555$\pm$0.233$\circ$ & 3.632$\pm$0.221\\
    11 & Health        & 4.594$\pm$0.327$\bullet$ & 4.643$\pm$0.295$\bullet$ 	& 4.015$\pm$2.971$\bullet$ 	& 4.108$\pm$0.168$\bullet$ 	& 11.410$\pm$0.384$\bullet$  & 2.595$\pm$0.029$\circ$  & 2.837$\pm$0.123$\bullet$ & 2.839$\pm$0.117\\
    12 & Social        & 4.524$\pm$0.288$\bullet$ & 4.375$\pm$0.267$\bullet$ 	& 4.592$\pm$5.846$\bullet$ 	& 4.132$\pm$0.326$\bullet$ 	& 10.073$\pm$0.579$\bullet$  & 2.999$\pm$0.144$\bullet$  & 2.818$\pm$0.123$\bullet$ & 2.773$\pm$0.123\\ 
    \hline
    & Average value 	& 14.895$\bullet$			& 14.856$\bullet$				& 14.400$\bullet$				& 13.866$\bullet$				& 21.626$\bullet$			& 13.172$\circ$				& 13.611$\bullet$           & 13.336\\
    \bottomrule
\end{tabular}
\end{table*}

\subsection{More Analysis of ProFS}
These comparison methods are all based on the assumption of symmetric label relationship, whereas ProFS's remarkable advantages across all experiments prove that asymmetric label correlations better align with true data distribution than symmetric ones. By preserving the directionality of label relationships, ProFS successfully reduces structural noise in the weight matrix $\boldsymbol{W}$ during the learning process, which implies that the performance improvement in MLFS fundamentally benefits from the asymmetric topological structure of labels.

In addition, the evaluation metrics of MLL often exhibit a trade-off phenomenon, this is, a model may perform well on \textit{HL}, which independently evaluates the binary classification accuracy of each label, but poorly on \textit{RL} or \textit{OE}, which evaluate the relative ranking ability of labels. However, ProFS breaks this bottleneck and achieves a clear advantage on both metrics. This directly confirms the success of introducing an auxiliary matrix $\mathbf{P}$ for decoupling design. By separating the weight matrix $\boldsymbol{W}$ from the dual tasks of fitting labels and capturing label correlations, the model maintains strong discriminative power while perfectly preserving the topological structure among asymmetric labels.

It should be noted that ProFS achieved sub-optimal results on the \textit{OE} and \textit{CV} metrics for \textit{CAL500} and \textit{Birds} datasets. This is attributed to the inherent structure of the data, \textit{CAL500} is a typical extreme multi-label dataset, while \textit{Birds} exhibits high-dimensional, small-sample and sparse data properties. This leads to it highly susceptible to introducing micro-topological noise during $k$-means clustering due to semantic gaps or the curse of dimensionality. Moreover, the nonlinear ordering and min-max extremum operators relied upon by the \textit{Sugeno integral} have a noise-amplifying effect, ultimately causing model overfitting. This suggests that the construction of asymmetric priors can't overlook the “sample-feature-label” ratio topology inherent in the data itself. 

For the non-convex $\ell_{2,1/2}$-matrix norm, our experimental results show no signs of performance collapse typically caused by trapping in local optima, instead, it yields robust outcomes with low variance across various datasets. This indicates that the potent sparsity benefits of the $\ell_{2,1/2}$ norm markedly overshadow its optimization challenges, substantiating that ProFS extracts an exceptionally refined features by successfully filtering out structural noise.

\vspace{2mm}
\section{Implementation Details of ProML evaluation}\label{apx:simu}
This section details the experimental setups for evaluating the GMLL and ProML learning paradigms. 
The experiments are conducted on the \textit{Emotions} dataset, which comprises 391 samples for training and 202 samples for testing. The random seed is fixed to 42, and the learning rate is set to $1e^{-4}$. During the training phase, the hyper-parameters are configured as follows: $\lambda=0.05, \alpha = 0.02$ and $\mu=0.8$, with the maximum number of iterations set to 600. The matrix $C$ represents the label correlation matrix computed via cosine similarity.

Fig. \ref{fig:simu_F}-\ref{fig:simu_21_2} illustrate the experimental results when employing Frobenius, $\ell_{1}$, $\ell_{2,1}$, and $\ell_{2,1/2}$ matrix norms as regularization terms. It can be observed that under these norm constraints, ProML consistently yields superior performance over GMLL across all five evaluation metrics. 
All optimization tasks are solved via gradient descent.

\begin{figure*}[!bh]
\vspace{2mm}
    \centering
    \includegraphics[width=\textwidth]{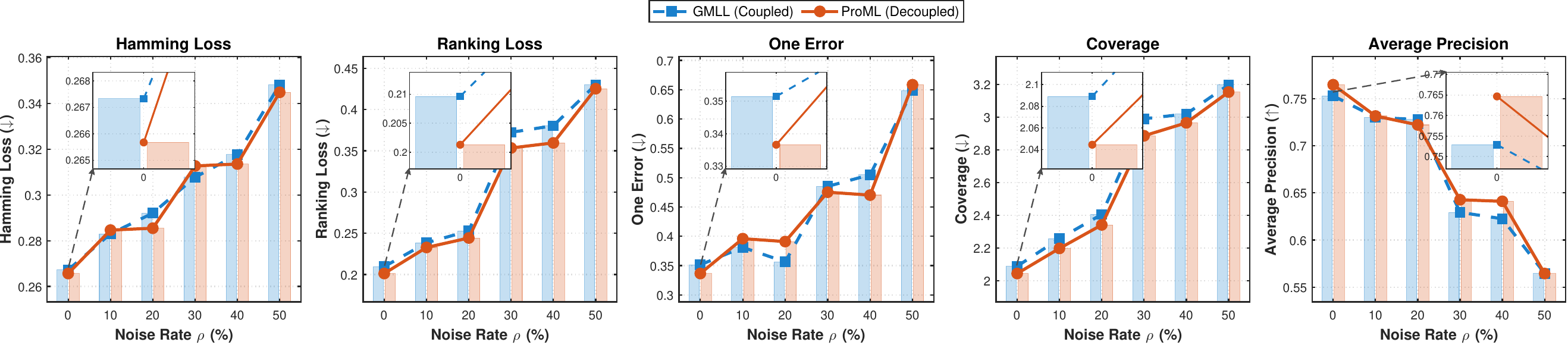}
    \caption{Comparison between GMLL and ProML across five evaluation metrics. Here, the standard regularization term $\mathcal{R}(\boldsymbol{W})=\|\boldsymbol{W}\|_F^2$ is employed.}\label{fig:simu_F}
\end{figure*}

\begin{figure*}[!th]
    \centering
    \includegraphics[width=\textwidth]{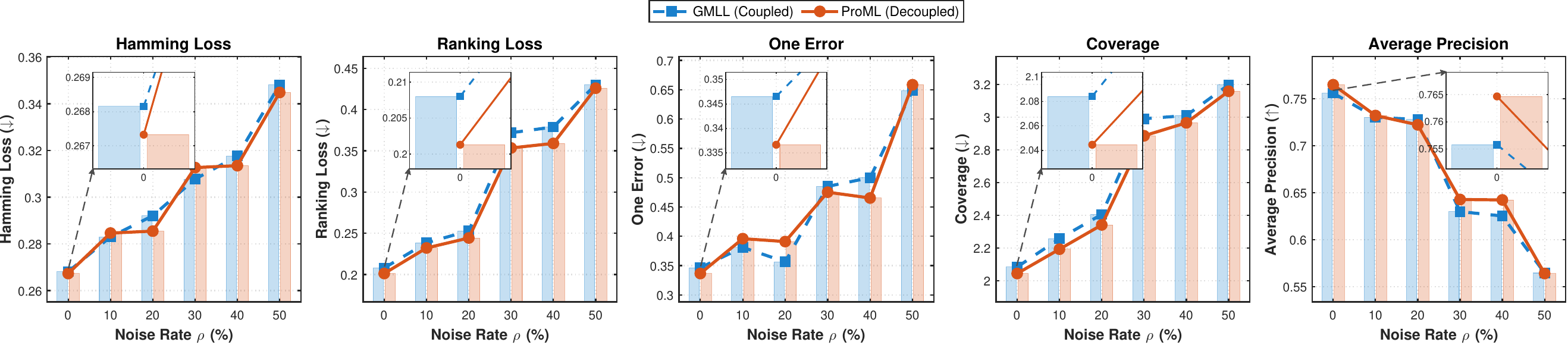}
    \caption{Comparison between GMLL and ProML across five evaluation metrics. Here, the standard regularization term $\mathcal{R}(\boldsymbol{W})=\|\boldsymbol{W}\|_1$ is employed.}    \label{fig:simu_1}
    \vspace{1.5mm}

    \includegraphics[width=\textwidth]{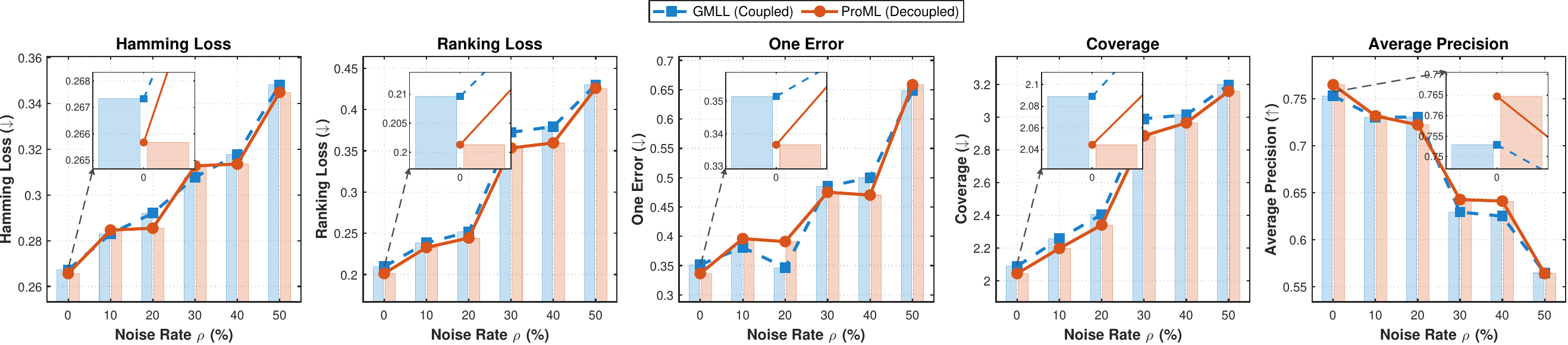}
    \caption{Comparison between GMLL and ProML across five evaluation metrics. Here, the standard regularization term $\mathcal{R}(\boldsymbol{W})=\|\boldsymbol{W}\|_{2,1}$ is employed.}    \label{fig:simu_21}
    \vspace{1.5mm}

    \includegraphics[width=\textwidth]{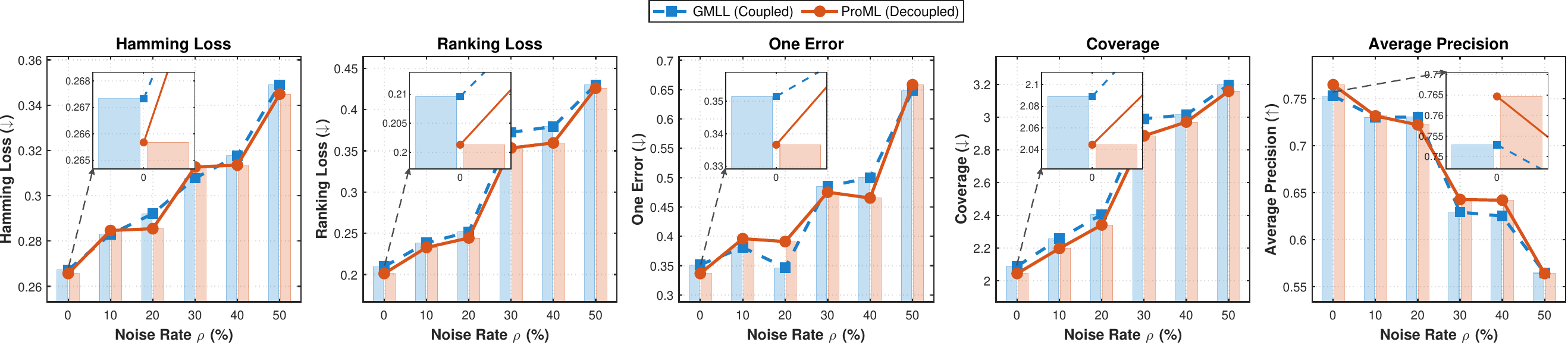}
    \caption{Comparison between GMLL and ProML across five evaluation metrics. Here, the standard regularization term $\mathcal{R}(\boldsymbol{W})=\|\boldsymbol{W}\|_{2,1/2}$ is employed.}
    \label{fig:simu_21_2} 
\end{figure*}

\vfill

\bibliographystyle{IEEEtran}
\bibliography{IEEEabrv,references}